\documentclass{article}
\usepackage[margin=1in]{geometry}
\usepackage{float}

\usepackage[utf8]{inputenc} 
\usepackage[T1]{fontenc}    
\usepackage[hidelinks]{hyperref}       
\usepackage{url}            
\usepackage{booktabs}       
\usepackage{amsfonts}       
\usepackage{nicefrac}       
\usepackage{microtype}      
\usepackage{xcolor}         

\usepackage{graphicx} 
\usepackage[parfill]{parskip}
\usepackage{enumitem}
\usepackage{amsmath,amssymb,amsthm}
\usepackage{tikz}
\usetikzlibrary{automata, positioning, arrows}
\usepackage{subcaption}
\usepackage{multirow}
\usepackage{relsize}

\allowdisplaybreaks

\usepackage[numbers]{natbib}

\definecolor{blanchedalmond}{rgb}{1.0, 0.92, 0.8}
\definecolor{beige}{rgb}{0.96, 0.96, 0.86}
\definecolor{blond}{rgb}{0.98, 0.94, 0.75}
\definecolor{brass}{rgb}{0.71, 0.65, 0.26}

\definecolor{buff}{rgb}{0.94, 0.86, 0.51}
\definecolor{burlywood}{rgb}{0.87, 0.72, 0.53}
\definecolor{camel}{rgb}{0.76, 0.6, 0.42}
\definecolor{bronze}{rgb}{0.8, 0.5, 0.2}
\definecolor{brown}{rgb}{0.59, 0.29, 0.0}

\definecolor{copper}{rgb}{0.72, 0.45, 0.2}
\definecolor{cinnamon}{rgb}{0.82, 0.41, 0.12}
\definecolor{darkchestnut}{rgb}{0.6, 0.41, 0.38}

\definecolor{brew1}{HTML}{fef0d9}
\definecolor{brew2}{HTML}{fdcc8a}
\definecolor{brew3}{HTML}{fc8d59}
\definecolor{brew4}{HTML}{d7301f}

\newcommand{\newtext}[1]{\textcolor{black}{#1}}

\def\R{\mathbb{R}}
\def\E{\mathbb{E}}

\def\argmin{\text{argmin}}
\def\Id{\mathbf{I}_d}
\def\In{\mathbf{I}_n}
\newcommand\independent{\protect\mathpalette{\protect\independenT}{\perp}}
\def\independenT#1#2{\mathrel{\rlap{$#1#2$}\mkern2mu{#1#2}}}

\def\ExcessRisk{{\rm ExcessRisk}} 

\def\stheta{\theta^{\star}} 
\def\y{Y} 
\def\x{X} 

\def\Y{\mathbf{Y}} 
\def\X{\mathbf{X}} 
\def\Noise{\boldsymbol{\eta}} 
\def\Xt{\mathbf{X}^\top}

\def\u{\mathbf{u}} 
\def\v{\mathbf{v}} 
\def\s{s} 

\def\Ud{\mathbf{U}_d}
\def\Ur{\mathbf{U}_r}
\def\Vd{\mathbf{V}_d}
\def\Vr{\mathbf{V}_r}
\def\Sd{S_d}
\def\Sr{S_r}

\def\ts{\tilde{s}} 
\def\tS{\tilde{S}} 

\def\htheta{\hat{\theta}}
\def\hSigma{\hat{\Sigma}} 
\def\ttheta{\tilde{\theta}}
\def\btheta{\bar{\theta}}

\def\slambda{\lambda^\star}
\def\bxik{\xi^{(k)}}
\def\bxir{\xi^{(r)}}
\def\bxi{\xi}
\def\xik{\bar{\xi}^{(k)}}
\def\xir{\bar{\xi}^{(r)}}
\def\barxi{\bar{\xi}}
\def\txi{\tilde{\xi}}

\def\Pre{\mathbf{P}} 
\def\Om{\mathbf{\Omega}}
\def\bone{\mathbf{1}}

\def\Mk{M^{(k)}}
\def\mk{m^{(k)}}
\def\ck{c}

\def\Bk{B^{(k)}}
\def\bk{b^{(k)}}

\def\Vk{V^{(k)}}
\def\vk{v^{(k)}}

\def\Ak{A^{(k)}(\lambda)}
\def\Ar{A^{(r)}(\lambda)}

\newtheorem{assumption}{Assumption}
\newtheorem{lemma}{Lemma}[section]
\newtheorem{theorem}{Theorem}
\newtheorem{remark}{Remark}[section]
\newtheorem{proposition}{Proposition}[section]

\title{Understanding the Gains from Repeated Self-Distillation}

\author{%
Divyansh Pareek \quad \quad  Simon S. Du  \quad \quad Sewoong Oh \\
\textit{Paul G. Allen School of Computer Science and Engineering} \\
\textit{University of Washington, Seattle, WA} \\
\texttt{\{dpareek,ssdu,sewoong\}@cs.washington.edu}\\
}

\date{} 

\begin{document}
\maketitle

\begin{abstract} 
Self-Distillation is a special type of knowledge distillation where the student model has the same architecture as the teacher model. Despite using the same architecture and the same training data, self-distillation has been empirically observed to improve performance, especially when applied repeatedly. For such a process, there is a fundamental question of interest: How much gain is possible by applying multiple steps of self-distillation? To investigate this relative gain, we propose studying the simple but canonical task of linear regression. Our analysis shows that the excess risk achieved by multi-step self-distillation can significantly improve upon a single step of self-distillation, reducing the excess risk by a factor as large as $d$, where $d$ is the input dimension. Empirical results on regression tasks from the UCI repository show a reduction in the learnt model's risk (MSE) by up to $47$\%.

\end{abstract}


\section{Introduction} \label{sec-intro}
Knowledge distillation \cite{distill2015} was initially proposed as a way to transfer the knowledge learnt by a larger teacher model to a smaller student model, which can then be deployed in limited resource settings. 
The process is as follows: Train a teacher ($T$) model using ground-truth labels, then use its predictions to supervise the training of a student ($S$) model via a combined per-sample loss, 
\begin{equation} \label{onestep-dist-obj}
    \xi \cdot \boldsymbol{\ell} \bigl( \hat{y}_T, y_S(\theta) \bigr) + (1 - \xi) \cdot \boldsymbol{\ell} \bigl( y, y_S(\theta) \bigr)\;, 
\end{equation} 
where $\boldsymbol{\ell}$ denotes the loss function, $y$ is the ground-truth label, $\hat{y}_T$ denotes the teacher's prediction, and $y_S(\theta)$ denotes the student's prediction, parameterized by the learnable $\theta$.
The extra hyperparameter $\xi$ is called the imitation parameter \cite{lopez2015unifying}, generally restricted to $\xi \in [0, 1]$.
It gives additional freedom to the student to balance importance between labels and teacher's predictions. The student trained via this distillation objective (i.e., utilizing the teacher's predictions through $\xi \neq 0$) has been widely observed to generalize better than when trained only on the labels (i.e., $\xi=0$). This gain has been attributed to `dark knowledge' that is $(i)$ impossible to be directly extracted from the training data by the small model, but $(ii)$ easily learnt by the large model and transferred to the small model. 

Challenging this interpretation, \citet{li2017learning} and \citet{furlanello2018born} empirically observed performance gains through distillation even when the teacher and student are same-sized models. One can set $T$ and $S$ to have the \emph{same architecture}, and $S$ trained with the objective in Eq.~\eqref{onestep-dist-obj} outperforms $T$. This is referred to as Born-Again Networks (BANs) or {\em Self-Distillation} (SD). Furthermore, repeatedly applying self-distillation on the same training data with a student model having the same architecture provides additional gains on benchmark datasets and architectures \cite{furlanello2018born,yang2019training,zhang2020self}. At each step, the student from the previous step acts as the teacher used to train a new student model under the self-distillation loss of Eq.~\eqref{onestep-dist-obj}. For such {\em multi-step self-distillation}, there is a fundamental question of interest: \emph{How much more gain can we get by repeatedly applying self-distillation?}

Recently, \citet{pmlr-v202-das23d} provided theoretical understanding of the original one-step self-distillation. For the canonical task of fixed design linear regression, considering the standard ridge estimator as both the teacher and student model, \cite{pmlr-v202-das23d} showed that there is indeed a regime of problem instances in which the \emph{optimal} student (i.e., with optimally tuned ridge parameter $\lambda$ and  imitation parameter $\xi$) can provably achieve a strictly lower test error than the \emph{optimal} teacher (i.e. with optimally tuned $\lambda$). However, the amount of this gain has not been characterized in closed form, and can only be numerically evaluated for a given problem instance. Inspired by this work, we aim to study the performance gains from multi-step self-distillation under linear regression. 

{\bf Contributions.} We summarize our contributions below.
\begin{itemize}
[leftmargin=2em,itemsep=0em,topsep=0em]
    \item Under the fixed design linear regression defined in Section~\ref{sec-setup-LR}, we show that the \emph{optimal} multi-step self-distilled model \newtext{(i.e., each $\xi$ value at each step is optimized for the validation accuracy of the final multi-step self-distilled model) }can achieve a test error that is a factor of $d$ smaller than the \emph{optimal} one-step self-distillation (Theorem~\ref{thm-FD-separation}), under certain assumptions on the problem parameters (Assumption~\ref{assump-FD-sj}). Here, $d$ is the dimension of the input. 
    \newtext{Our analysis in Theorem~\ref{thm-FD-separation} suggests that the sequence of $\xi$ parameters provides additional freedom that can control the spectrum of eigenvalues of the linear estimator. Optimally choosing these $\xi$ parameters can significantly reduce the variance of the estimator, leading to a factor of (up to) $d$ difference in the overall test errors of multi-step SD compared to $1$-step SD.}
    We note that \citet{pmlr-v202-das23d} also observed a bias-variance tradeoff associated with the $\xi$ parameter for $1$-step SD compared to the ridge\newtext{, which was the reason behind $1$-step SD strictly outperforming the ridge}.

    \item  We demonstrate the necessity of Assumption~\ref{assump-FD-sj} theoretically (Theorems~\ref{thm-FD-nonsep-assump} and \ref{thm-FD-nonsep-nonpeaky}) and numerically (Figure~\ref{fig-synth-all}). Further, we provide a lower bound for the test error that any repeated SD can achieve, and show that only $r$ steps of SD (with optimally chosen $\xi$ at each step) are sufficient to achieve this lower bound, when the input data matrix has rank $r$ (Theorem~\ref{thm-sdkPre-optPre}).
    \item By capturing the functional form of the test error in $\xi$ (Theorem~\ref{thm-quadraticRisk}), we also show a method to practically select the $\xi$ parameters for real-world regression tasks.
In Section~\ref{sec-expts}, we empirically show that this theoretical insight leads to selecting effective $\xi$ values, which can indeed achieve a lower test error on real-world regression tasks.
\end{itemize}

\section{Related Work} \label{sec-related}

\textbf{Knowledge distillation and self-distillation.} 
\citet{distill2015}, \citet{ba2014deep} proposed knowledge distillation to transfer knowledge learnt by large teacher models into smaller student models without any substantial performance drop (e.g.,  \cite{romero2014fitnets,rusu2015policy,huang2017like,chen2017learning,urban2017do,liu2019structured,sun2019patient} and surveys in \cite{gou2021knowledge,hu2022teacher}). Distillation also provides interpretability \cite{liu2018improving}, robustness to adversarial examples \cite{papernot2016distillation}, and defense against backdoor attacks \cite{yoshida,li2021neural,xia2022eliminating}, although stronger backdoor attacks have been proposed that bypass distillation defense \cite{jha2023label}. Perhaps surprisingly, empirical observations show that performance {\em improves} when a teacher model is distilled into a student model with the same architecture on the same training data ({\em self-distillation}). 
Performance gains with one-step self-distillation of the form Eq.~\eqref{onestep-dist-obj} were first demonstrated by \citet{li2017learning} for AlexNet on YFCC100M. Further gains can be achieved by repeating self-distillation, as shown for the DenseNet architecture family on CIFAR10 and CIFAR100 \cite[Table 2]{furlanello2018born}. To empirically explain such gains, \citet{zhang2020self} measured prediction uncertainty on the same multi-step experiments and offered an interpretation that soft labels capture sample-level uncertainties.  \citet{yang2019training} also reproduced the same experiments and explained the gains as knowledge refinement on the class similarities. We will analytically study the gains that can be achieved with such repeated self-distillation. 

Many variations of self-distillation have also been proposed. Snapshot Distillation \cite{yang2019snapshot} tries to treat previous checkpoints (snapshots) of the same model as the teacher. \citet{zhang2018deep} employ a group of collaborative students with no teacher. \citet{zhang2021self,zhang2019your} use it for model self-improvement, and 
DINO \cite{caron2021emerging} adopts self-distillation for self-supervised learning. 
Knowledge distillation is also popular for transfer learning, where the student model is trained on a different dataset than the teacher model \cite{yim2017gift,zagoruyko2016paying,ahn2019variational}, which is not a setting we address. 
With the recent scaling of data, \cite{zhu2024iterative}, \cite{lin2024rho} are relevant works using a teacher model for either label editing or data reweighing.

\newpage

\textbf{Theory of distillation and self-distillation.} Theoretical understanding of distillation started with \citet{phuong2019towards} studying linear student networks.
\citet{mobahi2020self} studied self-distillation theoretically in the restricted setting of $\xi = 1$ (i.e. only teacher supervision, no ground-truth labels), showing that in this setting, the SD process acts as a regularizer, with a few steps of SD helping, but further steps hurting model performance. We study the setting where $\xi$ is not restricted to $1$, and show a different conclusion. In particular, we observe that more steps of SD always provide an improvement, if the $\xi$ parameters are chosen optimally.
\citet{allen-zhu2023towards} analyzed a stylized setting, where a different view of the data is learned by different models, and show how ensemble methods can combine the views, achieving improved test accuracy. This framework is used to show how self-distillation can also improve model accuracy by implicitly performing ensembling.
\citet{menon2021statistical} theoretically studied distillation in the classification setting, and also observed a bias-variance tradeoff underlying teacher supervision.
\citet{pmlr-v202-das23d} theoretically studied one-step self-distillation for fixed design linear regression and binary classification, and showed that the student can provably achieve a lower test error than the teacher. We take inspiration from them and study the multi-step SD to characterize this performance gain, showing that the multi-step SD can outperform one-step SD by a large factor.
\citet{jeong2024understanding} also take inspiration from \cite{pmlr-v202-das23d} and provide understanding of multi-step self-distillation in the multi-class classification setting.


\section{Problem formulation and background on self-distillation} \label{sec-setup}

Focusing on the simple but canonical task of linear regression, we investigate the performance gain from applying repeated self-distillation.

\subsection{Linear regression} \label{sec-setup-LR}

For the observed response $\y \in \R$ and the covariate $\x \in \R^d$, the following assumption is standard in linear regression, e.g., \cite{pmlr-v202-das23d}.  
\begin{assumption} \label{assump-LR}
There exist $\stheta \in \R^d$ and $\gamma>0$ such that
    ($i$) $\E[\y | \x] = \langle \stheta, \x \rangle$;
    ($ii$) ${\rm Var}[\y | \x] = \gamma^2$ for all $\x \in \R^d$; and 
    ($iii$) $\left( \y - \E[\y | \x] \right) \independent \x$, i.e. the label noise is independent of $\x$. 
\end{assumption}
The training set of size $n$ is denoted by $\X \in \R^{d \times n}$, the collection of covariates, and $\Y \in \R^n$, the responses; $\Y = \Xt \stheta + \Noise$, with $\Noise$ satisfying $\E[\Noise] = 0$, $\E[\Noise \Noise^\top] = \gamma^2 \In$. 
The problem instance is defined by its parameters $(\X, \stheta, \gamma^2)$, treating $\X=[\x_1, \x_2, \cdots \x_n]$ as fixed but $\Y$ as random. The training set $(\X, \Y)$ is one occurrence of the random noise $\Noise \in \R^n$.
In this {\em fixed design} setup, 
the excess risk of an estimator $\htheta$ is defined using the standard $\hSigma_n$-norm, $\| v\|_{\hSigma_n}=\|\hSigma_n^{1/2} v\|_2$, as
\begin{align} \label{eq-ER-fixed}
    \ExcessRisk(\htheta) \;\; &:= \;\; \E_{\Noise} \left[ \| \htheta - \stheta \|_{\hSigma_n}^2 \right]\;, 
\end{align}
where $\hSigma_n := (\nicefrac{1}{n})\X \Xt$ is the covariance matrix, and the expectation is over the randomness in $\Noise$.
Measuring the error in the $\hSigma_n$-norm ensures that the signal-to-noise ratio is uniform in all directions. 
The popular ridge estimator serves as a baseline, using a single hyperparameter $\lambda > 0$: 
\begin{align} 
    \htheta(\lambda) \;&:=\; \arg\min_{\theta \in \R^d} \left( \| \Y - \Xt \theta \|^2 + \lambda \| \theta \|^2 \right) 
    \;=\; \left( \X \Xt + \lambda \Id \right)^{-1} \X \Y \label{est-ridge-2}\;.
\end{align} 
We use  $\Om_\lambda := \X \Xt + \lambda \Id$ throughout. We consider only $\lambda > 0$, but surprisingly, \citet{kobak2020optimal} showed that the optimal penalty $\slambda$ (one with the lowest risk) can indeed be negative. However we will largely work in the non-overparameterized case ($n > d$), where $\slambda > 0$ holds.

\subsection{Self-distillation} \label{sec-setup-Estimators}

Applying the self-distillation loss in Eq.~\eqref{onestep-dist-obj} to linear regression with hyperparameters $\lambda > 0$ and $\xi \in \R$,
\begin{align} \label{est-sd1}
    \htheta(\lambda, \xi) \;&:=\; \arg\min_{\theta \in \R^d} \left(  \xi \,\| \Xt \htheta(\lambda) - \Xt \theta \|^2 +  (1 - \xi) \,\| \Y - \Xt \theta \|^2 + \lambda  \,\| \theta \|^2 \right) \\
    &=\; \left( \X \Xt + \lambda \Id \right)^{-1} \X \underbrace{\left( \xi \cdot \X^\top \htheta (\lambda) + (1 - \xi) \cdot \Y \right)}_{\text{New label}} \label{est-sd1-2} \\
    &= \;\underbrace{\left\{ (1 - \xi) \cdot \Id + \xi \cdot \Om_\lambda^{-1} \X \Xt \right\}}_{\text{Pre-conditioner: function of }(\lambda, \xi)} \cdot \underbrace{\Om_\lambda^{-1} \X \Y}_{\text{Ridge }  \htheta(\lambda)} \;, \label{est-sd1-3}
\end{align}
where $\xi\in\R$ is not restricted to the conventional $[0, 1]$ interval. This additional freedom is meaningful since it can result in a strictly better solution, as noted by \citet{pmlr-v202-das23d} both theoretically (Remark 3.6) and empirically (Table 1). \newtext{It is worth noting that the optimization problem in eq.~\eqref{est-sd1} remains convex for any $\xi \in \R$, since its hessian evaluates to $\xi \cdot 2 \X\Xt + (1 - \xi) \cdot 2 \X \Xt = 2 \X \Xt \succeq 0$.}
On the other hand, the teacher and student use the same ridge penalty $\lambda$ for simplicity.

We call this estimator {\em $1$-step self-distillation}.  This can be interpreted as 
($i$) assigning new labels that combine the ground-truth labels with the teacher's predictions, or ($ii$) pre-multiplying the usual ridge estimator with a pre-conditioner. Note that $\xi = 0$ recovers ridge. 
\citet[Theorem 3.8]{pmlr-v202-das23d} show that under a certain condition, 1-step self-distillation strictly dominates ridge, i.e.,
\begin{equation} \label{eq-thm-sdgain}
    \min_{\lambda \geq 0, \xi \in \R} \E_{\Noise} \left[ \| \htheta(\lambda, \xi) - \stheta \|_2^2  \right] \;<\; \min_{\lambda \geq 0} \E_{\Noise} \left[ \| \htheta(\lambda) - \stheta \|_2^2  \right]\;,
\end{equation}
where the risk is measured in the non-standard Euclidean norm. The same strict inequality can be shown under the standard $\hSigma_n$-norm under a slightly modified condition stated in Proposition~\ref{prop-condition}.
This naturally leads to a fundamental question: \emph{How much more gain can we get by repeatedly applying self-distillation?}

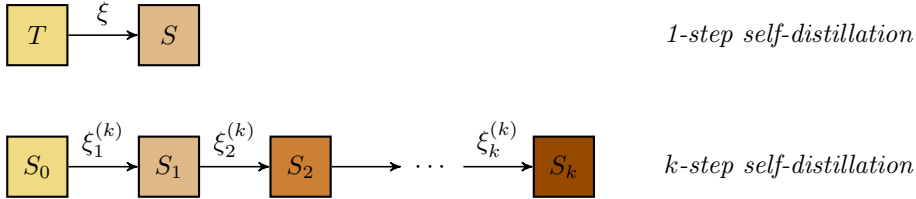
\begin{figure}[h]
\centering
\begin{tikzpicture}[->, >=stealth', auto, semithick, node distance=1.75cm]
\tikzstyle{state}=[draw, thick, fill=buff, minimum size=8mm]

\node[state] (T) {$T$};
\node[state, fill=burlywood] (S) [right of=T] {$S$};
\node[state, draw=none, fill=white] (lbl) at (10, 0) {\textit{1-step self-distillation}};
\path (T) edge node {$\xi$} (S);

\node[state] (S0) [below of=T] {$S_0$};
\node[state, fill=burlywood] (S1) [right of=S0] {$S_1$};
\node[state, fill=bronze] (S2) [right of=S1] {$S_2$};
\node[state, draw=none, fill=white] (Sdots) [right of =S2] {$\cdots$};
\node[state, fill=brown] (Sk) [right of=Sdots] {$S_k$};
\node[state, draw=none, fill=white] (lbla) at (10, -1.75) {\textit{$k$-step self-distillation}};
\path (S0) edge node {$\bxik_1$} (S1)
      (S1) edge node {$\bxik_2$} (S2)
      (S2) edge (Sdots)
      (Sdots) edge node {$\bxik_k$} (Sk);
\end{tikzpicture}
\caption{The standard $1$-step self-distillation defined in Eq.~\eqref{onestep-dist-obj} with parameter $\xi$ and $k$-step self-distillation that repeatedly applies Eq.~\eqref{onestep-dist-obj} with parameter $\bxik=[\bxik_1,\bxik_2,\ldots,\bxik_k] \in \R^k$.
}
\label{fig-sdk}
\end{figure}

\subsection{Repeated self-distillation} \label{sec-setup-Estimators-sdk}
The standard repeated application of self-distillation starts with the teacher model, $T$ (which we also refer to as the zeroth model,  $S_0$), and applies self-distillation sequentially for $k$ steps. At each step $i$, Eq.~\eqref{onestep-dist-obj} is applied with the $(i-1)^{th}$ model, $S_{i-1}$ as the teacher, and the $i^{th}$ model, $S_i$ as the student, with an imitation parameter $\bxik_i$, i.e., $\hat\theta \in \arg\min_\theta \big\{ \bxik_i \boldsymbol{\ell} \bigl( \hat{y}_{S_{i-1}}, y_{S_i}(\theta) \bigr) + (1 - \bxik_i )  \boldsymbol{\ell} \bigl( y, y_{S_i}(\theta) \bigr)\,\big\}$ for $i\in[k]$. The collection of parameters is denoted by $\bxik=[\bxik_1,\bxik_2,\ldots,\bxik_k] \in\R^k$.

This repeated self-distillation has been studied, for example, theoretically in \cite{mobahi2020self} and empirically in \cite{furlanello2018born}. We aim to understand its gain under linear regression, where we prove that 
\begin{align} \label{est-sdk}
    \htheta(\lambda, \underbrace{\bxik}_{\in \R^k}) \;&=\; 
    \underbrace{\left\{ \left( 1 - \sum_{i=1}^k \xik_i \right) \Id + \sum_{i=1}^k \xik_i \left( \Om_\lambda^{-1} \X \Xt \right)^i \right\}}_{\text{Pre-conditioner: }\Pre (\lambda, \bxik)} \cdot \underbrace{ \Om_\lambda^{-1} \X \Y}_{\text{Ridge }  \htheta(\lambda)} \;, 
\end{align}
with $\xik_i := (1 - \bxik_{k-i}) \prod_{l=k-i+1}^{k} \bxik_{l}$ for each $i \in [k]$, and we let $\bxik_0=0$. The proof that repeated SD with $\bxik \in \R^k$ results in Eq.~\eqref{est-sdk} is provided in Appendix~\ref{sec-app-xis-objective}. 
Here 
$\bxik \in \R^k$ denote the imitation parameters, $\lambda$ denotes the ridge coefficient for all the models, and 
$\xik \in \R^k$ is a \emph{reparametrization} of $\bxik \in \R^k$ (details in Appendix \ref{sec-app-xis-objective}).
We call this {\em $k$-step self-distillation}. 
Note the increasing flexibility in the pre-conditioner matrix. 
The increasing powers of $\Om_{\lambda}^{-1} \X \Xt$ in the above expression are still numerically stable, since, for $\lambda > 0$, $\Om_{\lambda}^{-1} \X \Xt$ is PSD with all eigenvalues in $[0, 1]$. 
As an aside, one can also consider a version of SD where the $i^{th}$ model receives supervision from all $S_{<i}$ instead of just $S_{i-1}$. Appendix \ref{sec-app-xis-fullVinc} shows that this version provides no extra representational capacity over the repeated version presented above, when all $k$ entries of $\bxik$ are optimized as free parameters. Hence, the procedure in Figure \ref{fig-sdk} suffices for analysis.



\section{Main results for linear regression} \label{sec-results-FD}

The main object of our study is to theoretically demonstrate the gains from repeated self-distillation. Concretely, we aim to show that there can be a significant multiplicative separation between the excess risk achieved by $r$-step SD (Self-Distillation), where $r$ is the rank of the input $\X$; compared to the ridge estimator, as well as the $1$-step SD (Section~\ref{sec-main1}). The necessity of the two main assumptions is shown in Section~\ref{sec-main2}. The sufficiency of $r$ steps of SD is shown in Section~\ref{sec-results-FD-optimalPre}. In Section~\ref{sec-results-FD-quadratic}, we provide an exact characterization of the excess risk achieved by $k$-step SD (for any $k$).


\subsection{The $r$-step self-distillation significantly improves upon the $1$-step self-distillation} \label{sec-main1}


We show the desired separation under 
the following assumption (and two more mild technical assumptions specified fully in Appendix~\ref{sec-app-proof-thm-FD-separation}). 
\begin{assumption}  \label{assump-FD-sj}
Assume the following two conditions hold on the problem instance $\left( \X, \stheta, \gamma^2 \right):$
\begin{enumerate}
\item No two non-zero singular values of $\X$ collide, i.e. $\s_1 > \s_2 > \cdots > \s_r > 0$, where $\{\s_j\}_{j=1}^r$ denote the non-zero singular values of the input data matrix $\X$ whose rank is denoted by $r$. \label{assump-FD-sj1}
\item $\measuredangle(\stheta, \u_1)=0,$ where $\{ \u_j \}_{j=1}^d$ denote the eigenvectors of $\X \Xt$, $\u_1$ being the leading one. \label{assump-FD-sj2}
\end{enumerate}
\end{assumption}
Assumption~\ref{assump-FD-sj} is needed to show that $r$-step SD achieves a small excess risk in Eq.~\eqref{eq-thm-FD-sep-SD}. 
In general, both these conditions are somewhat necessary for the separation, as we show in Theorems \ref{thm-FD-nonsep-assump} and \ref{thm-FD-nonsep-nonpeaky}. 
We now state our main result. We show that under the above assumption, there exists a family of problem instances, $(\X, \stheta, \gamma^2)$, such that the excess risk achieved by $r$-step SD is a factor of $r:={\rm rank}(\X)$ smaller than that of the ridge estimator \emph{and} the 1-step SD. 




\begin{theorem} \label{thm-FD-separation}
Under the fixed design linear regression in Assumption~\ref{assump-LR}, there exists a family of problem instances satisfying Assumption \ref{assump-FD-sj} such that for any instance $(\X, \stheta, \gamma^2)$ in the family, it holds that
\begin{align}
    \exists \lambda > 0, \exists \bxir \in \R^r, \hspace{10pt} &\ExcessRisk\left( \htheta (\lambda, \bxir) \right) \leq \frac{\gamma^2}{n} \label{eq-thm-FD-sep-SD} \;, \\
    \forall \lambda > 0, \forall \xi \in \R, \hspace{10pt} &\ExcessRisk\left( \htheta (\lambda, \xi) \right) \geq  \left( \nicefrac{0.99}{2^{9}} \right) \, \frac{r\gamma^2}{n} \label{eq-thm-FD-sep-oneSD} \;, \text{ and }\\
    \forall \lambda > 0, \hspace{10pt} &\ExcessRisk\left(\htheta(\lambda) \right) \geq \left( 0.98 \right) \, \frac{r \gamma^2}{n} \label{eq-thm-FD-sep-Ridge}\;,
\end{align}
where $r:={\rm rank}(\X)$, $n$ is the number of samples, $\htheta(\lambda,\bxir)$ and $\htheta(\lambda,\xi)$ are the $r$-step and $1$-step SD estimators defined in Eqs.~\eqref{est-sdk} and \eqref{est-sd1} respectively, and $\htheta(\lambda)$ is the ridge estimator defined in Eq.~\eqref{est-ridge-2}.
\end{theorem}

We provide precise conditions and a proof in Appendix~\ref{sec-app-proof-thm-FD-separation}. Since each $k$-step SD includes $(k-1)$-step SD as a special case with the proper choice of $\bxik$, the hyperparameter-tuned excess risk of repeated SD is monotonically non-increasing. However, it is perhaps unexpected that the multiplicative separation between $r$-step SD and $1$-step SD can be as large as $\Omega(r)$, demonstrating the gains of repeated SD. 
\newtext{Figure~\ref{fig-synth-ridge-vs-rstepSD} illustrates this $\Omega(r)$ multiplicative separation on a synthetic family of problems.}
Note that $\Omega(d)$ separation can be achieved by choosing the problem instance to have rank $r=d$, at the cost of requiring many more steps of SD. This $\Omega(d)$ factor is the largest multiplicative separation possible with self-distillation, as shown by the fundamental lower bound in Theorem~\ref{thm-sdkPre-optPre} for any pre-conditioning based approach. 

\begin{remark}
    SD significantly outperforms ridge by primarily reducing the variance. For the lower bound on ridge's excess risk, i.e., Eq.~\eqref{eq-thm-FD-sep-Ridge}, we ignored the bias term and only used the variance term. The repeated SD (Eq.~\eqref{eq-thm-FD-sep-SD}) primarily reduces the variance to improve the excess risk over Eq.~\eqref{eq-thm-FD-sep-Ridge}.
\end{remark} 

\begin{figure}[h]
\includegraphics[width=\textwidth]{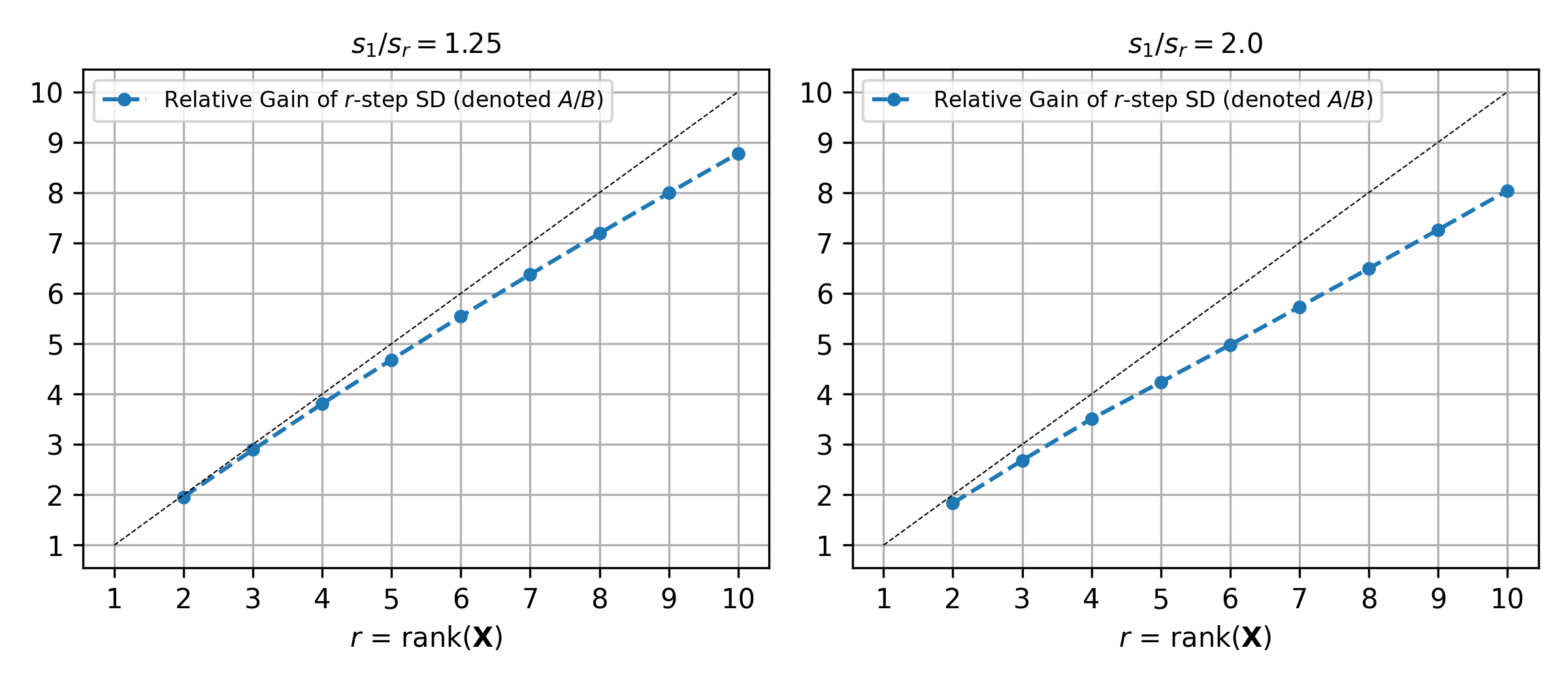}
\caption{
\newtext{On a synthetic problem family with dimension $d=100$, noise variance $\gamma=0.1$, and $\stheta=\u_1$ (agreement with Asmp.~\ref{assump-FD-sj}.\ref{assump-FD-sj2}); we set the singular values of $\X$ with a power law from $s_1=1$ to $s_r=\{0.8, 0.5\}$ (left and right panels) and vary $r={\rm rank}(\X)$. Both plots show a linear increase of the relative gain of $r$-step self-distillation in excess risk, i.e. the ratio $\nicefrac{A}{B}$ where $A:= \min_{\lambda > 0} \ExcessRisk \bigl( \htheta(\lambda) \bigr)$ and $B:= \min_{\lambda > 0, \bxir \in \R^r} \ExcessRisk \bigl( \htheta(\lambda, \bxir) \bigr)$; demonstrating that $r$-step SD outperforms ridge by a factor of $\Omega(r)$, with the constant inside the $\Omega$ (i.e. slope of the line) changing with the effective condition number, $\nicefrac{s_1}{s_r}$.}
}
\label{fig-synth-ridge-vs-rstepSD}
\end{figure}


\subsection{Necessity of Assumption~\ref{assump-FD-sj}} \label{sec-main2}

In Figure~\ref{fig-synth-all}, we empirically show on synthetic tasks how violating Assumption~\ref{assump-FD-sj}.\ref{assump-FD-sj1} or \ref{assump-FD-sj}.\ref{assump-FD-sj2} leads to higher excess risks, even for the $r$-step SD ($r=4$ in the example). This supports the necessity of both assumptions, which we analytically investigate in the following. 

\textbf{Necessity of Assumption~\ref{assump-FD-sj}.\ref{assump-FD-sj1} on $\X$}. 
We assume that the non-zero singular values of $\X$ are unique. This allows us to tightly upper bound the excess risk achieved by $r$-step SD in Eq.~\eqref{eq-thm-FD-sep-SD} via Theorem~\ref{thm-sdkPre-optPre}. We show in the following that some version of Assumption \ref{assump-FD-sj}.\ref{assump-FD-sj1} is also \emph{necessary}. For a more detailed explanation of why  we need Assumption \ref{assump-FD-sj}.\ref{assump-FD-sj1}, we refer the reader to Remark \ref{remark-necessity-FD-sj1}.


\begin{theorem} \label{thm-FD-nonsep-assump}
Under the hypotheses of Theorem \ref{thm-FD-separation} except for Assumption~\ref{assump-FD-sj}.\ref{assump-FD-sj1}, if the singular values of $\X$ satisfy  $\s_1 = \ldots=s_r=1$, where $r={\rm rank}(\X)$,  for all $k \geq 1$, $\lambda>0$, and $\bxik\in\R^k$, we have 
\begin{align} 
    &\ExcessRisk\left( \htheta \left( \lambda, \bxik \right) \right) \; \geq \; \frac{r \gamma^2}{n} \left( {1 + \frac{r \gamma^2}{\sum_{j=1}^r \langle \stheta, \u_j \rangle^2}} \right)^{-1} \label{eq-thm-FD-nonsep-assump-1} \;.
\end{align}
Furthermore, there exists $\lambda > 0$ such that the ridge, $\htheta(\lambda)$, achieves this lower bound with equality. 
\end{theorem}
We provide a proof in Appendix~\ref{sec-app-proof-thm-FD-nonsep-assump}. This implies that when there is no gap in the singular values of the input $\X$, there is no separation between ridge and SD estimators (repeated or not).
Intuitively, if the $\s_j$ are all equal, 
the pre-conditioner for ridge (i.e., $\Om_{\lambda}^{-1}$) and the pre-conditioner for the repeated SD, both are restricted to have all eigenvalues to be equal.
(Repeated) SD has no degrees of freedom to deviate from this.
However, $\s_j$'s being unequal provides the freedom for the $\bxik$ to control the SD's pre-conditioner matrix, and reduce the excess risk. 
This is also why in Remark \ref{remark-necessity-FD-sj1}, we hypothesize that numerically, the optimal $\left(\bxik\right)^\star$ depends inversely on the min-gap of the singular values.
\newtext{Figure~\ref{fig-synth-zistar-vs-singgap} demonstrates this increasing relationship of the magnitude of the optimal $\xi$ parameters w.r.to the decreasing singular gap.}

\begin{figure}[h]
\begin{subfigure}{.33\textwidth}
  \includegraphics[width=\textwidth]{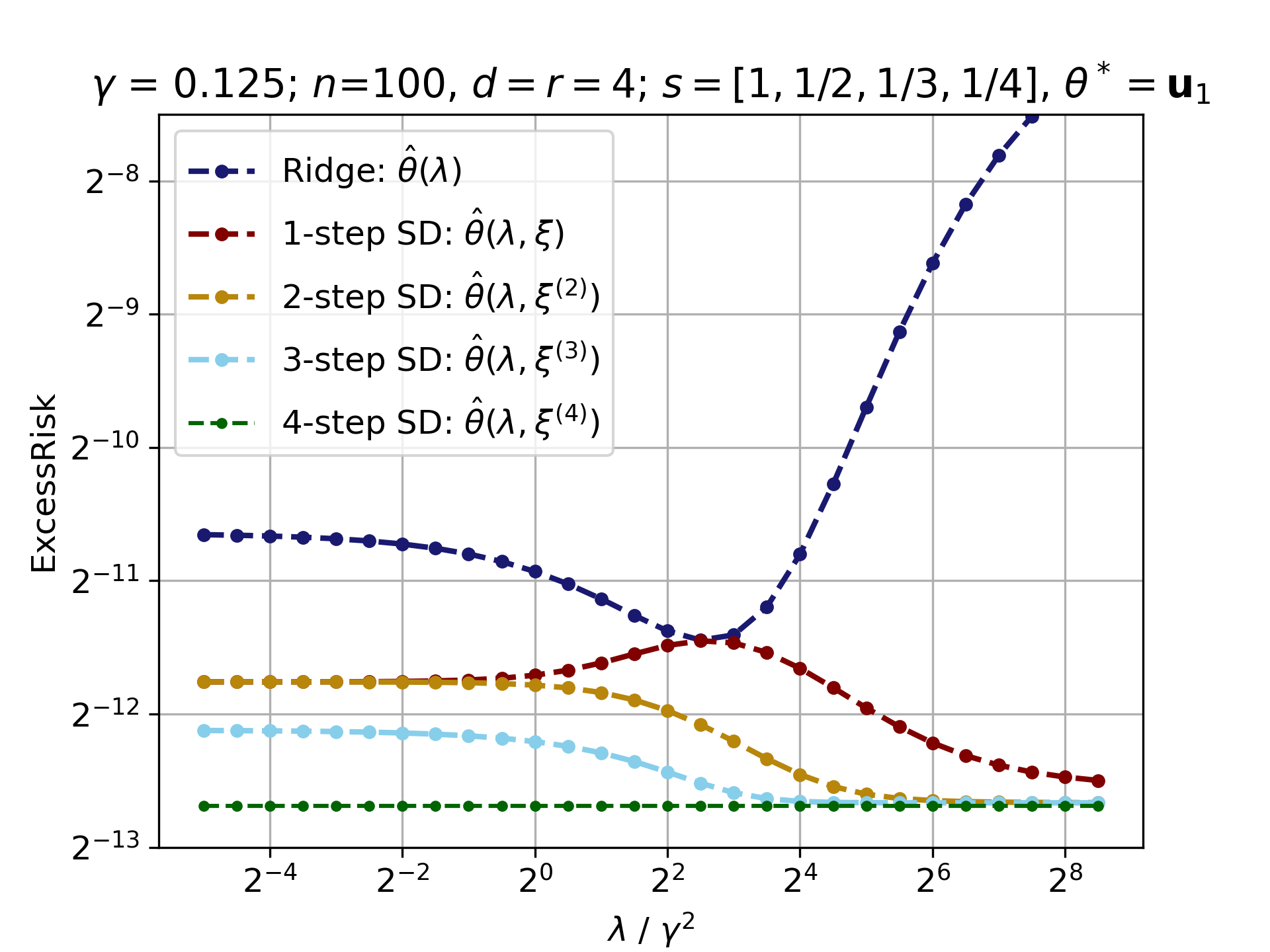}
  \caption{$\s := [1,\nicefrac{1}{2},\nicefrac{1}{3},\nicefrac{1}{4}], \stheta = \u_1$}
  \label{fig-synth_A}
\end{subfigure}
\begin{subfigure}{.33\textwidth}
  \includegraphics[width=\linewidth]{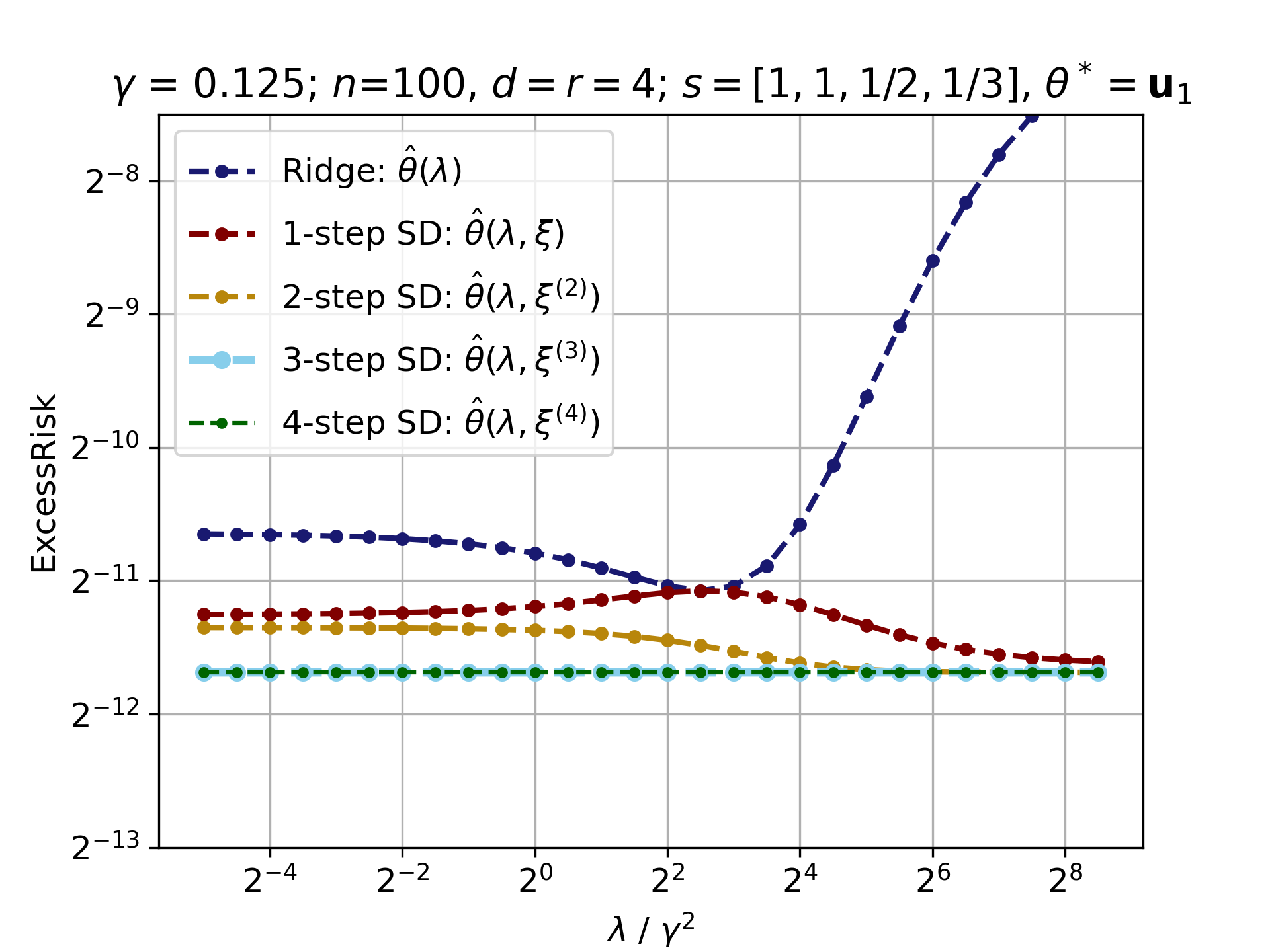}
  \caption{$\s := [1,1,1/2,1/3]$}
  \label{fig-synth_B}
\end{subfigure}
\begin{subfigure}{.33\textwidth}
  \includegraphics[width=\linewidth]{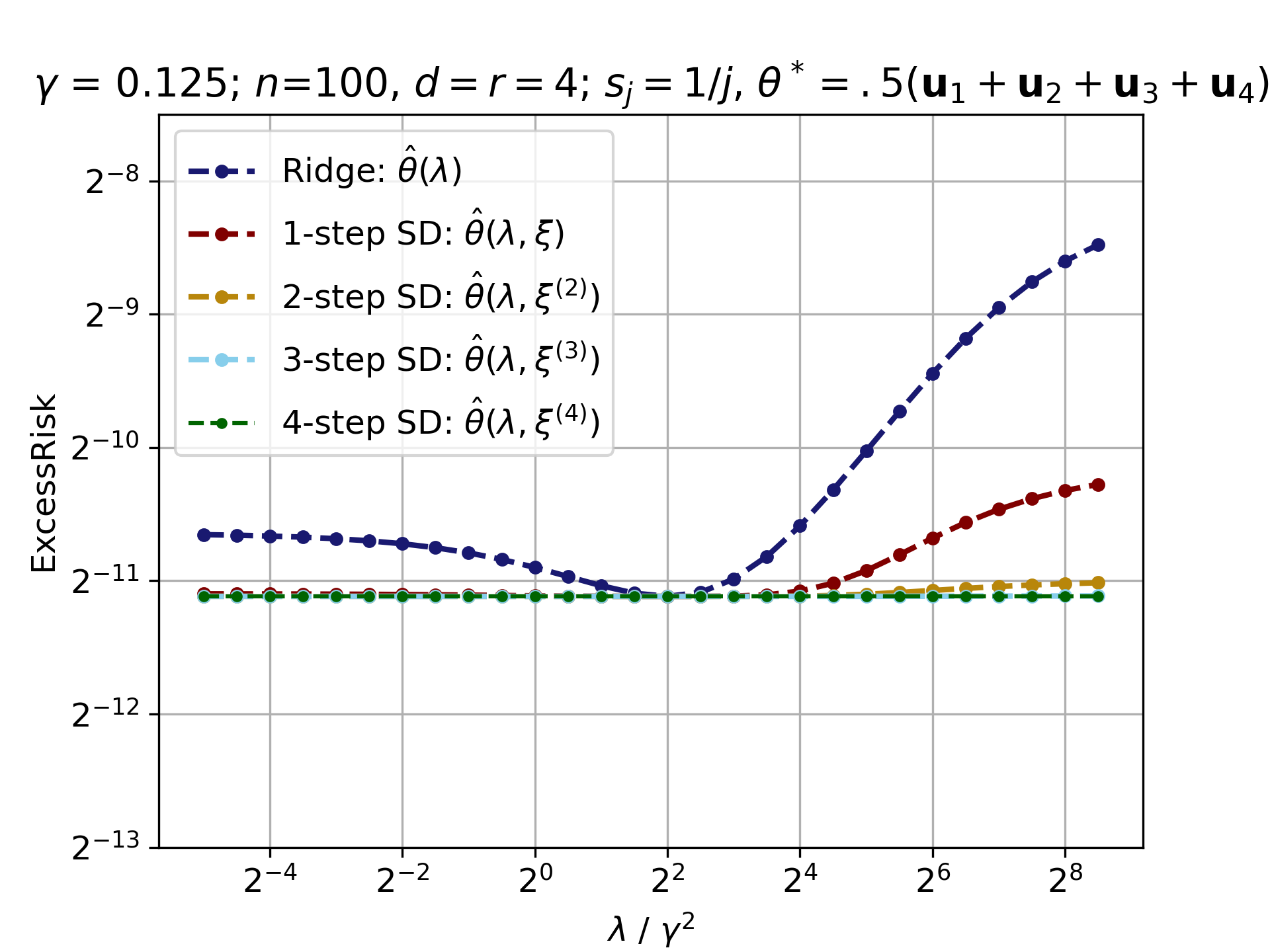}
  \caption{$\stheta := 0.5(\u_1+\u_2+\u_3+\u_4)$}
  \label{fig-synth_C}
\end{subfigure}
\caption{On a synthetic task (explained in Section~\ref{sec-expts-synth}), $\X$ has rank $4$ with (a) $\stheta=\u_1$ and distinct $s_j$'s; (b) $s=[1,1,\nicefrac{1}{2},\nicefrac{1}{3}]$; (c) $\stheta=0.5(\u_1+\u_2+\u_3+\u_4)$. Each additional step of SD with optimal choice of $\bxik$ reduces $\ExcessRisk( \htheta(\lambda, (\bxik)^\star) )$ for any choice of $\lambda$ on the $x$-axis. Panel (a) satisfies Asmp.~\ref{assump-FD-sj} and hence $4$-step SD is necessary to achieve the optimal excess risk. This is no longer true when Asmp.~\ref{assump-FD-sj}.\ref{assump-FD-sj1} is violated (b) or Asmp.~\ref{assump-FD-sj}.\ref{assump-FD-sj2} is violated (c).
\newtext{
Excess risk achieved by $4$-step SD (i.e. the green lines) in panels (a) and (c) exactly match the numerical value given by RHS of eq.~\eqref{eq-thm-sdk-opt-ER-LowerBound}, i.e. the fundamental lower bound for any SD estimator. But this is not the case in panel (b) [which has the same lower bound from eq.~\eqref{eq-thm-sdk-opt-ER-LowerBound} as panel (a)], because Asmp.~\ref{assump-FD-sj}.\ref{assump-FD-sj1} is violated.}
}
\label{fig-synth-all}
\end{figure}

\begin{figure}[H]
\includegraphics[width=\textwidth]{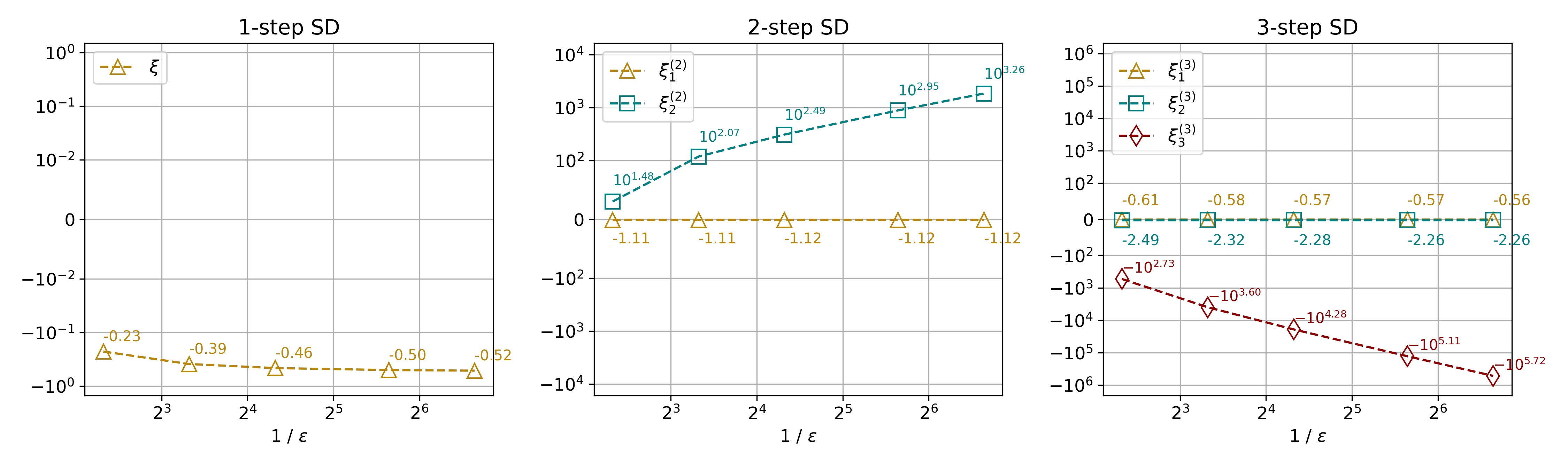}
\caption{
\newtext{On the synthetic problem from Figure~\ref{fig-synth_A}, we fix $\lambda = 0.125$ and set the singular values of $\X$ as $\s_j = \{1 - (j-1) \epsilon\}$, i.e. consecutive values are separated by $\epsilon$. For $k$-step SD with $k=\{1, 2, 3\}$, we plot $(\bxik)^\star (\lambda)$ (i.e. optimal values of the $\xi$ parameters) by varying $\epsilon \in \{0.2, 0.1, 0.05, 0.02, 0.01\}$. The magnitude of $\bxik_k$ values increases as the singular gap $\epsilon$ decreases, verifying Remark~\ref{rem-ill-cond}.}
}
\label{fig-synth-zistar-vs-singgap}
\end{figure}

\textbf{Necessity of Assumption~\ref{assump-FD-sj}.\ref{assump-FD-sj2} on $\stheta$}. 
For simplicity, we assumed that $\stheta$ is aligned solely with $\u_1$ (i.e., projection onto any other $\u_j$ is zero for $j \geq 2$). 
In general, it is sufficient that $\stheta$ is completely aligned with \emph{any one} of the eigenvectors of $\X\Xt$ (not necessarily the leading eigenvector) to prove a large separation between ridge and repeated SD.
We make this precise in Appendix \ref{sec-app-peaky-discussion}. 
We show next that if 
$\stheta$ is equally (mis)aligned with all the eigenvectors $\{ \u_j \}_{j=1}^r$ of $\X\Xt$, then again there is no separation between ridge and repeated SD. 

\begin{theorem} \label{thm-FD-nonsep-nonpeaky}
Under the hypotheses of Theorem~\ref{thm-FD-separation} except for Assumption~\ref{assump-FD-sj}.\ref{assump-FD-sj2}, if the true parameter $\stheta$ satisfies $\langle \stheta, \u_j \rangle^2 = z$ for all $j \in [r]$, it holds that for all $z>0$, $k \geq 1,  \lambda > 0$, and $\bxik \in \R^{k}$, 
\begin{align} 
    &\ExcessRisk\left( \htheta ( \lambda, \bxik ) \right) \; \geq \; \frac{\gamma^2}{n} \sum_{j=1}^r \left( {1 + \frac{\gamma^2}{z \s_j^2}} \right)^{-1} \;. \label{eq-thm-FD-nonsep-nonpeaky-1} 
\end{align}
Furthermore, there exists $\lambda > 0$ such that the ridge, $\htheta(\lambda)$, achieves this lower bound with equality. 
\end{theorem}

We provide a proof in Appendix~\ref{sec-app-proof-thm-FD-nonsep-nonpeaky}. Similar conditions are needed when analyzing 1-step SD in \cite{pmlr-v202-das23d} as well; \cite[Eq.~(9) in Theorem 3.8]{pmlr-v202-das23d} is required for the $1$-step SD to strictly outperform ridge. Observe that Eq.~(9) is violated when either ($i$) $\s_j$'s are all equal or ($ii$) $\langle \stheta, \u_j \rangle^2 $'s are all equal.


\subsection{$r$ steps of self-distillation are sufficient} \label{sec-results-FD-optimalPre}
For a given problem instance $(\X,\stheta,\gamma^2)$, the excess risk achieved by the $k$-step SD with parameter $\bxik$ can be exactly characterized (Theorem~\ref{thm-quadraticRisk}), but it is complicated and can only be evaluated numerically in general. On the other hand, we show that there exists a fundamental lower bound that holds for a linear family of  estimators including all repeated SD, and this lower bound has a simple characterization (Lemma~\ref{lem-optimalPre}). Furthermore, we show that under a mild assumption on the eigenvalues of $\X\Xt$ in Assumption~\ref{assump-FD-sj}.\ref{assump-FD-sj1}, the $r$-step SD achieves the lower bound (Theorem~\ref{thm-sdkPre-optPre}). This allows a precise characterization of the performance of $r$-step SD.  

\begin{theorem} \label{thm-sdkPre-optPre}
    Under the fixed design linear regression in Assumption~\ref{assump-LR}, the excess risk of any $k$-step SD estimator on an instance $(\X, \stheta, \gamma^2)$, is lower bounded for all $k \geq 1$, $\lambda > 0$, and $\bxik \in \R^{k}$ by 
    \begin{align} 
        \ExcessRisk\left( \htheta ( \lambda, \bxik ) \right) &\; \geq \;  \frac{\gamma^2}{n} \sum_{j=1}^r \frac{ \langle \stheta, \u_j \rangle^2 }{\left( \langle \stheta, \u_j \rangle^2  + \frac{\gamma^2}{\s_j^2} \right)} \label{eq-thm-sdk-opt-ER-LowerBound}\;, 
    \end{align}
    where $(s_j,\u_j)$ is the $j^{\text{th}}$ eigenvalue and eigenvector of $\X$ and $r:={\rm rank}(\X)$. Furthermore, 
    if Assumption \ref{assump-FD-sj}.\ref{assump-FD-sj1} holds then there exists $\lambda > 0$ and $\bxir \in \R^{r}$ such that the equality is achieved by the $r$-step SD estimator $\htheta(\lambda,\bxir)$.
\end{theorem}
A proof is provided in Appendix~\ref{sec-app-proof-thm-sdkPre-optPre} and we provide a sketch below.  \\
\textbf{Proof sketch}. The lower bound in Eq.~\eqref{eq-thm-sdk-opt-ER-LowerBound} is an instantiation of Lemma~\ref{lem-optimalPre}, since $\htheta( \lambda, \bxik)$ is a specific linear family estimator with  
$\Pre = \Pre \left(\lambda, \bxik \right) \, \Om_{\lambda}^{-1}$ defined in Eq.~\eqref{est-sdk}. 
To show achievability, we need to show that $\Pre \left(\lambda, \bxik \right) \, \Om_{\lambda}^{-1}=\Pre^\star$ for some value of $k$, $\lambda$, and $\bxik$. 
This holds when the below system of $r$ linear equations admits a solution for the $k$ parameters (i.e. $\bxik$), with an \emph{extra} free parameter $\lambda > 0$.  
We show that with $k=r$ and under Assumption \ref{assump-FD-sj}.\ref{assump-FD-sj1}, there exists $\lambda > 0$ that will ensure the existence of a solution for this system of equations. 
\begin{align} \label{eq-thm-sdk-opt-condition}
    \left( 1 - \sum_{i=1}^k \xik_i \left\{ 1 - \left( \frac{\s_j^2}{\lambda + \s_j^2} \right)^i \right\} \right) \, \frac{\s_j^2}{\lambda + \s_j^2} \; &= \; \frac{\langle \stheta, \u_j \rangle^2}{\langle \stheta, \u_j \rangle^2 + \frac{\gamma^2}{\s_j^2}}\; \quad \quad  \forall j \in [r] 
\end{align}
\begin{remark} [Necessity of Assumption~\ref{assump-FD-sj}.\ref{assump-FD-sj1}] \label{remark-necessity-FD-sj1} \label{rem-ill-cond}
This assumption is required for (\ref{eq-thm-sdk-opt-condition}). Otherwise, the LHS would be the same for indices $j$ and $j+1$ if $\s_j = \s_{j+1}$, but the RHS could still be different as $\langle \stheta, \u_j \rangle \neq \langle \stheta, \u_{j+1} \rangle$ generally. If Assumption~\ref{assump-FD-sj}.\ref{assump-FD-sj1} does not hold, there might not be any $\bxik$ satisfying the set of equations for a general $\stheta \in \R^d$.
Further, the system of linear equations in Eq.~\eqref{eq-thm-sdk-opt-condition} becomes more ill-conditioned as the singular values $\s_j, j \in [r]$ get closer to each other. Capturing this dependence explicitly is outside the scope of this paper. 
\end{remark}


{\bf Lower bound for a linear family.}
Consider a linear family of estimators of the form $\htheta(\Pre) := \Pre \cdot \X\Y$, for $\Pre := \Ud \tS \Ud^\top$, whose eigenspace coincides with that of $\X\Xt$ (i.e., $\Ud=[\u_1,\ldots,\u_d]$) and has 
$d$ degrees of freedom represented by the eigenvalues $\tS = {\rm diag}[ \ts_1, \cdots , \ts_d ]$. This is a generic form of any linear estimator, albeit with the restriction of the eigenvectors matching the underlying $\Ud$. In particular, $k$-step SD is an instantiation of this with $\Pre = \Pre(\lambda, \bxik) \, \Om_{\lambda}^{-1}$ (refer to Eq.(~\ref{est-sdk})). 
\begin{lemma} \label{lem-optimalPre}
    The Excess Risk for $\htheta(\Pre) = \Pre \cdot \X \Y$ where $\Pre := \Ud  \tS \Ud^\top $, satisfies
    \begin{align} \label{eq-lem-optExcessRisk}
        \ExcessRisk \left( \htheta(\Pre) \right) \;&\geq\; \frac{\gamma^2}{n} \sum_{j=1}^r \frac{ \langle \stheta, \u_j \rangle^2 }{\left( \langle \stheta, \u_j \rangle^2  + \frac{\gamma^2}{\s_j^2} \right)}  \; ,
    \end{align}
    with equality achieved at $\Pre = \Pre^\star = \Ud \tS^\star \Ud^\top$, given by
    \begin{equation} \label{eq-lem-tsigma*}
        \ts_j^\star \;=\; \begin{cases}
            \frac{\langle \stheta, \u_j \rangle^2 }{\left( \langle \stheta, \u_j \rangle^2  \s_j^2 + \gamma^2 \right)} \text{ }, &\text{ } j \leq r \text{\rm { } (i.e., } \s_j > 0 \text{\rm )} \\
            \text{any real value }, &\text{ } j \geq r+1 \text{\rm  { }(i.e., } \s_j = 0 \text{\rm )} \\
        \end{cases} \;.
    \end{equation}  
\end{lemma}
\textbf{Proof sketch}. This is straightforward.  One can expand the excess risk  for $\htheta(\Pre)$, which is  a quadratic expression in $\ts_j, j \in [d]$. Completing the squares shows the lower bound of Eq.~(\ref{eq-lem-optExcessRisk}) and the optimal values $\ts^\star_j$ of Eq.~(\ref{eq-lem-tsigma*}). A full proof is given in Appendix \ref{sec-app-proof-lem-optimalPre}. 



\subsection{The excess risk for the $k$-step SD estimator is quadratic in $\xik \in \R^k$} \label{sec-results-FD-quadratic}

We give an explicit formula for the excess risk achieved by for the $k$-step SD estimator from Eq.~\eqref{est-sdk}. 
Since $\htheta(\lambda, \bxik)$ is linear in each of $\xik_i, i \in [k]$ (recall that $\xik$ is a reparametrization of $\bxik$), the overall excess risk is \emph{quadratic} in $\xik$ as shown below. Appendix \ref{sec-app-proof-thm-quadraticRisk} provides a proof and the expressions for $\Mk$, $\mk$, and $\ck$.
This quadratic form will be especially useful in experiments.
\begin{theorem}[Informal version of Theorem~\ref{thm-quadraticRisk-formal} in Appendix~\ref{sec-app-proof-thm-quadraticRisk}] \label{thm-quadraticRisk}
     Under the fixed design linear regression in Assumption~\ref{assump-LR}, the excess risk achieved by the $k$-step SD is quadratic in $\xik \in \R^k$: 
    \begin{equation} \label{quadratic-ER}
        \ExcessRisk \left( \htheta(\lambda, \bxik)  \right) \;=\; \left(\xik\right)^\top \underbrace{\Mk}_{\in \R^{k \times k}} \left(\xik \right) + 2  \left( \xik \right)^\top \underbrace{\mk}_{\in \R^k} + \text{ } \ck\;. 
    \end{equation}
\end{theorem}
From the detailed expressions given in Appendix~\ref{sec-app-proof-thm-quadraticRisk}, we note that $\Mk$ is a sum of $r$ symmetric rank-$1$ matrices, which means it can have a maximum rank of $r$. This implies that $\Mk\in\R^{k\times k}$ for $k > r$ is rank-deficient (causing no additional decrease in the excess risk if the $\xir \in \R^r$ were chosen optimally to minimize the excess risk). This indicates that $r$ steps of SD might be sufficient to achieve the optimal excess risk, which is indeed what we observe in Theorem~\ref{thm-sdkPre-optPre}.




\section{Experiments} \label{sec-expts}

In this section, we empirically show that multi-step SD can outperform the ridge and single-step SD.
We first present a synthetic setting (section~\ref{sec-expts-synth}) to validate our theory. In section~\ref{sec-expts-hparams}, we discuss a strategy to select $\xi$ parameters based on the theoretical insight from section~\ref{sec-results-FD-quadratic}. In section~\ref{sec-expts-regr}, we implement that strategy on real-world regression tasks and show that it can indeed select performant $\xi$ values that provide multi-step SD estimators that achieve a smaller test risk. 

\subsection{Synthetic Experiments} \label{sec-expts-synth}
We validate our theoretical results with a fixed design synthetic experiment. We consider a problem with $d=r=4$, and set problem parameters $(\X, \stheta, \gamma^2)$. Namely, $\X$'s singular values are set as $\s_j := \nicefrac{1}{j}$ for $j \in [4]$, and $\stheta := \u_1$ as in Theorem \ref{thm-FD-separation}. Figure~\ref{fig-synth-all} shows the result for $\gamma = 0.125$, along with two more settings that validate the necessity of our assumptions (validating Theorems \ref{thm-FD-nonsep-assump} and \ref{thm-FD-nonsep-nonpeaky}).
Figure~\ref{fig-synth_A} confirms that repeated steps of SD do provide a reduction in the excess risk, since the \emph{lowest point} of the curve for each $k$ reduces as $k$ increases.
Also note that the optimal $\lambda$ for each $k$ (one that produces lowest excess risk estimator) is different.
Appendix~\ref{sec-app-synth-details} presents some more settings, including $\stheta := \nicefrac{1}{\sqrt{2}} (\u_1 + \u_2)$ for comparison with \cite{pmlr-v202-das23d}.
An interesting phenomenon in Figure~\ref{fig-synth-all} is that local maxima in $k$-step SD's curve coincide with local minima in $(k-1)$-step SD's curve, which was proven for $k=1$ in \cite{pmlr-v202-das23d}, and we observe empirically for all $k$.

{\bf Explanation of Figures~\ref{fig-synth-all},~\ref{fig-regr-all}}. \newtext{For the fixed design synthetic experiment in Figure~\ref{fig-synth-all} and the random design real-world experiment in Figure~\ref{fig-regr-all} (section~\ref{sec-expts-regr}),} the curves plotted are with the optimal $(\bxik)^\star$ for each $\lambda$. Hence, the curve of $k$-step SD will point-wise be lower/equal to the curve of $(k-1)$-step SD, since more steps of SD only provide more freedom. We say $k$-step SD {\emph strictly} dominates $(k-1)$-step SD when the minimum value of $k$-step SD's excess risk is strictly lower than that of $(k-1)$-step SD. For Figure~\ref{fig-synth-all}, the optimal $(\bxik)^\star$ is found analytically from the problem parameters. For real-world datasets in Figure~\ref{fig-regr-all}, we describe how to find $(\bxik)^\star$ for any given $\lambda$ in the next section.

\subsection{Choosing the hyperparameters $\xi$ for real-world datasets} \label{sec-expts-hparams}
We have shown that at the cost of introducing additional hyperparameters \emph{and} setting them to their optimal values, one can extract a large (upto $\Omega(d)$) performance gain.
However, how does one select these $\xi$'s for real-world datasets? The standard practice is to use a validation set, and perform a grid search. But this becomes infeasible for $k$-step SD for larger values of $k$, since performing a search over parameters in $\R^{k+1}$ (i.e $k$ values of $\bxik$ and $1$ value of $\lambda$) quickly becomes impractical.
But our theoretical analysis provides an insight that can be used to directly compute the optimal $\bxik \in \R^k$ (for a chosen $\lambda$) given a few evaluations on the validation set with certain chosen $\bxik$ values.
Namely, Theorem \ref{thm-quadraticRisk} tells us that the $\ExcessRisk$ is quadratic in $\xik$ (the reparameterized version). Now the coefficients of the quadratic depend on unknown quantities (like $\stheta, \gamma^2$), however we can use the validation set to estimate these coefficients. 

For example, for $1$-step SD, we know that the $\ExcessRisk (\htheta(\lambda, \xi)) = A\xi^2 + 2 B \xi + C$ for unknown $A, B, C$ (that depend on $\lambda$). Training $3$ specific $1$-step SD estimators, with $\xi = \{-1, 0, 1\}$, and measuring each of those $3$ estimators' Risk on the validation set, lets us estimate $A, B, C$. We then know that $\xi^\star = \nicefrac{-B}{A}$, for the chosen value of $\lambda$. Hence, we only need to perform a grid search over \emph{one} parameter, $\lambda$. Appendix \ref{sec-app-choosing-xi} provides more discussion, and provides a detailed illustration of the above process for $k=2$.
Note that this is feasible when the cost/time needed for a single training run of a $k$-step SD is small (since we need to do it multiple times for various $\bxik$ values). 

\subsection{Real-world regression experiments} \label{sec-expts-regr}

We implement multi-step SD for real-world regression tasks from the UCI repository \cite{ucirepo}, and demonstrate that $2$-step SD can outperform ridge and $1$-step SD. Note that for this section, the test set will contain fresh samples of $\x \in \R^d$, i.e. random design linear regression instead of fixed design.
Our metric for an estimator's performance will be 
mean squared error (MSE) on a test set of unseen examples, which is the empirical version of total risk (i.e. excess risk + unknown noise variance $\gamma^2$).
Given $n_{test}$ such examples denoted by $\{ (x_i, y_i)\}, i \in [n_{test}]$, the MSE of an estimator $\htheta$ is given by the mean of per-sample squared errors, i.e., $MSE(\htheta) = \nicefrac{\sum_{i}( \langle \htheta, x_i \rangle - y_i )^2}{n_{test}}$.


Using the training and validation splits, we compute ($i$) Optimal ridge: $\htheta (\slambda_0)$, ($ii$) Optimal $1$-step SD: $\htheta (\slambda_1, \xi^\star)$, and ($iii$) Optimal $2$-step SD: $\htheta \bigl(\slambda_2, (\bxi_1^\star, \bxi_2^\star)\bigr)$. 
This is done by plotting the MSE on the validation set for a grid of $\lambda$ values for all three estimators, and choosing the $\lambda$ that achieves the lowest error for each one. For any given $\lambda$, the optimal $\bxi^\star (\lambda)$ is chosen by the strategy described in Section~\ref{sec-expts-hparams}.
Finally, we evaluate the MSE of all three chosen estimators on the test set, which serves as our performance metric (refer to Table~\ref{table}).
Appendix~\ref{sec-app-regr} explains the overall methodology in greater detail. We apply this methodology on three datasets (descriptions in Appendix~\ref{sec-app-datasets}).

Table~\ref{table} describes the results we observe. For two of the three datasets, $2$-step SD can outperform both ridge and $1$-step SD. For the Air Quality dataset, $2$-step SD significantly outperforms both ridge and $1$-step SD, reducing the MSE by $47.2$\% compared to the optimal ridge.
In contrast, for the AEP dataset, we observe that the SD process cannot improve upon the ridge at all.
The MSE curves in Figure~\ref{fig-regr-all} shed more light on these observations. Notice how Figures~\ref{fig-regr-AQ},~\ref{fig-regr-AF} show a gap in the ridge and $2$-step SD (similar to Figure~\ref{fig-synth_A}), whereas Figure~\ref{fig-regr-AEP} shows no such gap (similar to Figure~\ref{fig-synth_C}).

We also verify that the strategy described in section~\ref{sec-expts-hparams} indeed selects performant $\xi$s.
Appendix~\ref{sec-app-observed-quadratic} shows that the optimal $\xi$ values given in Table~\ref{table}, selected via the strategy in section~\ref{sec-expts-hparams}, are indeed the ones that minimize the validation MSE. Further, we empirically observe the quadratic nature of risk (MSE) vs $\xi$ described in Theorem~\ref{thm-quadraticRisk} (Figure~\ref{fig-regr-appendix} in Appendix~\ref{sec-app-observed-quadratic}).
\begin{table}[h!]
    \centering
    \caption{Chosen hyperparameter values and the achieved test set MSE for ridge and $1,2$-step SD.} \label{table}
    \resizebox{\linewidth}{!}{
    \begin{tabular}{c c c c c}
        \toprule
        \textbf{Dataset} & & {\bf Optimal ridge} & {\bf Optimal $1$-step SD} & {\bf Optimal $2$-step SD} \\
        \midrule
        \multirow{2}{*}{\texttt{Air Quality}} & \textsmaller{Optimality hyperparameters} & $\slambda_0 = 10^2$ & $\slambda_1, \xi^\star = 10^3, -4.1$ & $\slambda_2, (\bxi_1^\star, \bxi_2^\star) = 10^3, (-0.9, -16.2)$ \\
                                   & \textsmaller{Test set MSE} & 2.01 & 1.99 & {\bf 1.06} \\
        \midrule
        \multirow{2}{*}{\texttt{Airfoil}} & \textsmaller{Optimality hyperparameters} & $\slambda_0 = 10^2$ & $\slambda_1, \xi^\star = 10^0,66.5$ & $\slambda_2, (\bxi_1^\star, \bxi_2^\star) = 10^3,(-1.8, -7.8)$ \\
                                   & \textsmaller{Test set MSE} & 1.34 & 1.22 & {\bf 1.19} \\
        \midrule
        \multirow{2}{*}{\texttt{AEP}} & \textsmaller{Optimality hyperparameters} & $\slambda_0 = 10^{2.5}$ & $\slambda_1, \xi^\star = 10^{2.5}, 0.1$ & $\slambda_2, (\bxi_1^\star, \bxi_2^\star) = 10^{2.5}, (-2.4, -2.3)$ \\
                                   & \textsmaller{Test set MSE} & {\bf 0.62} & 0.62 & 0.63 \\
        \bottomrule
    \end{tabular}
    }
\end{table}
\begin{figure}[h]
\begin{subfigure}{.33\textwidth}
  \includegraphics[width=\textwidth]{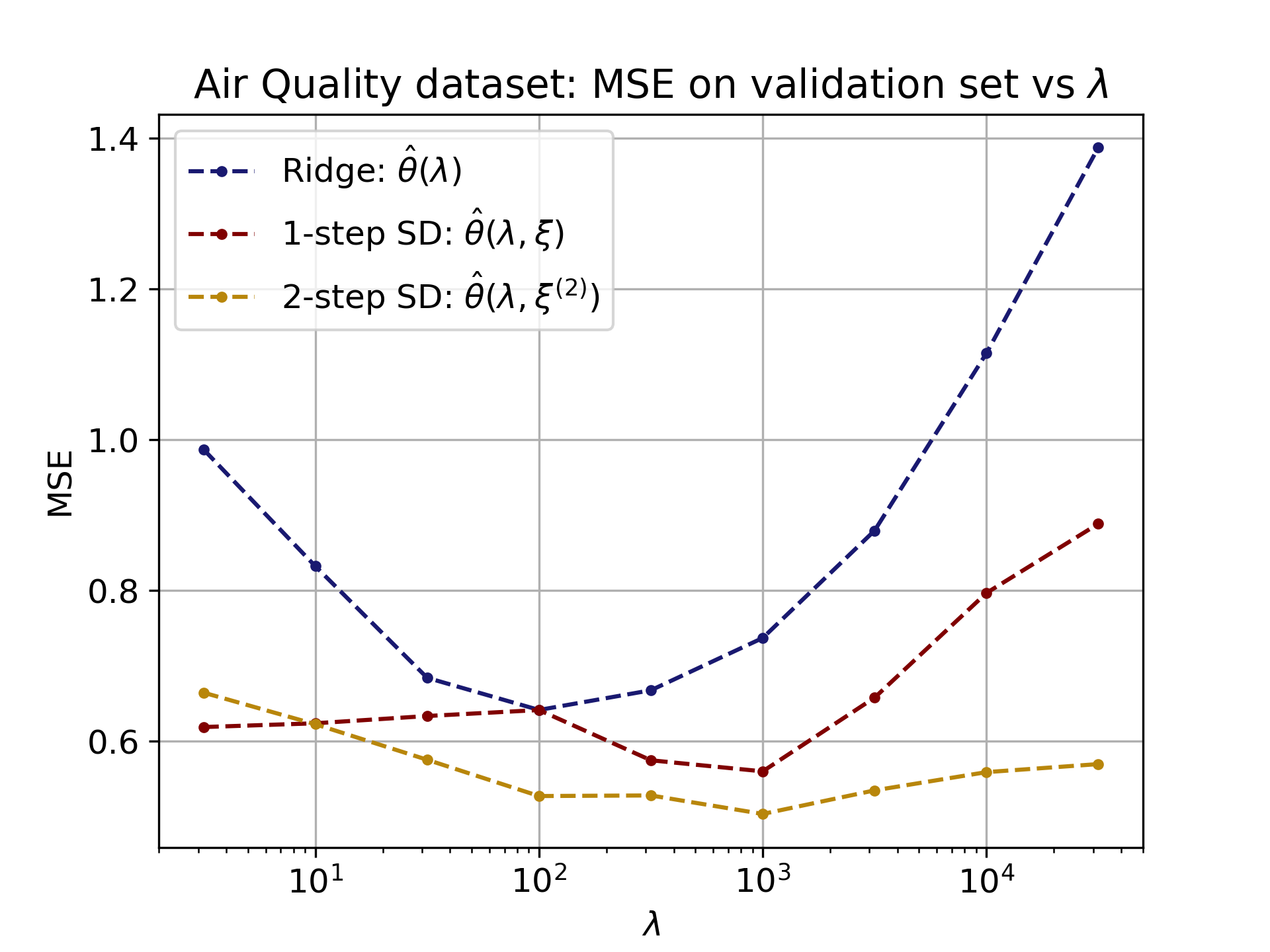}
  \caption{Air Quality dataset}
  \label{fig-regr-AQ}
\end{subfigure}
\begin{subfigure}{.33\textwidth}
  \includegraphics[width=\linewidth]{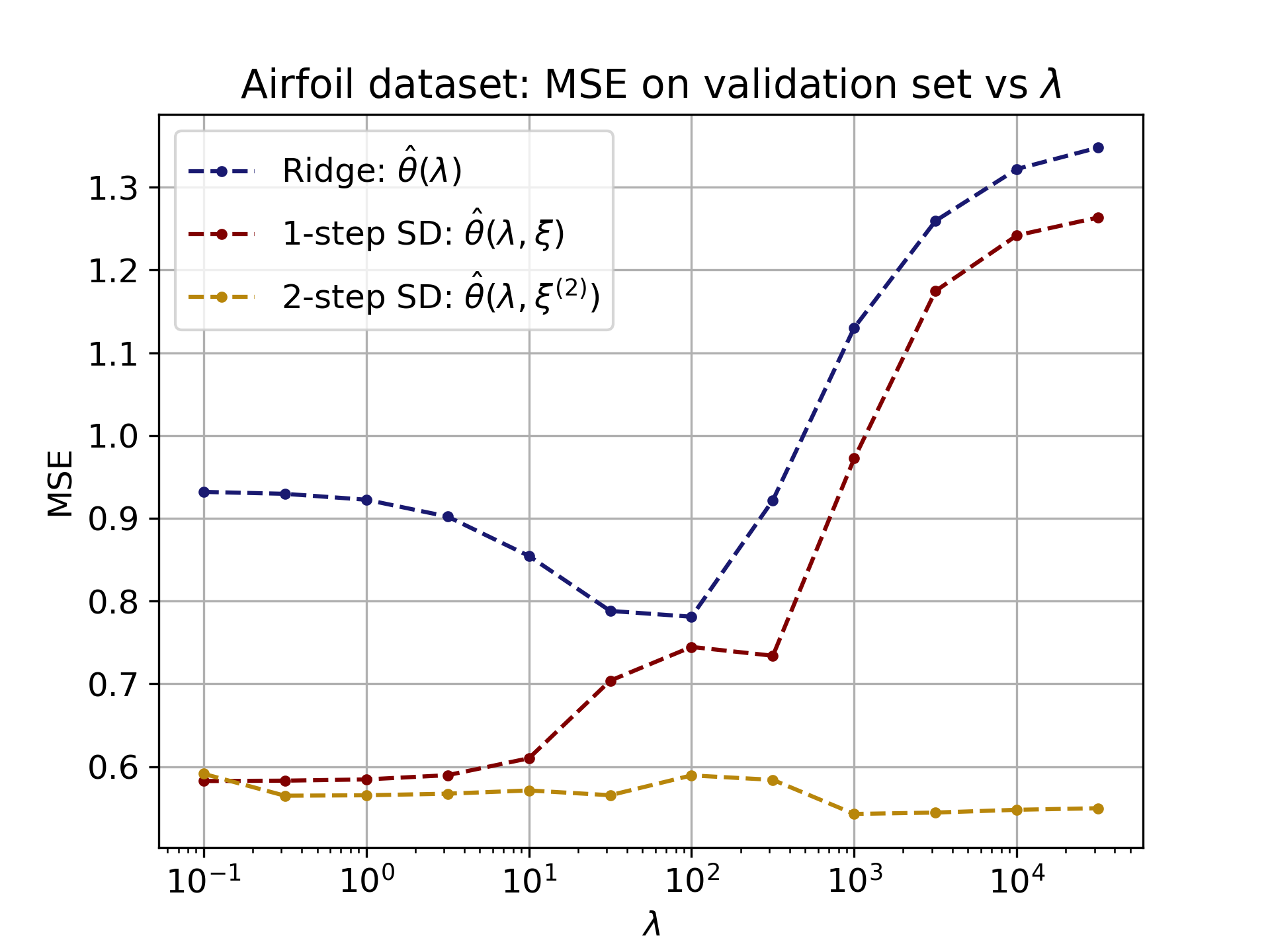}
  \caption{Airfoil dataset}
  \label{fig-regr-AF}
\end{subfigure}
\begin{subfigure}{.33\textwidth}
  \includegraphics[width=\linewidth]{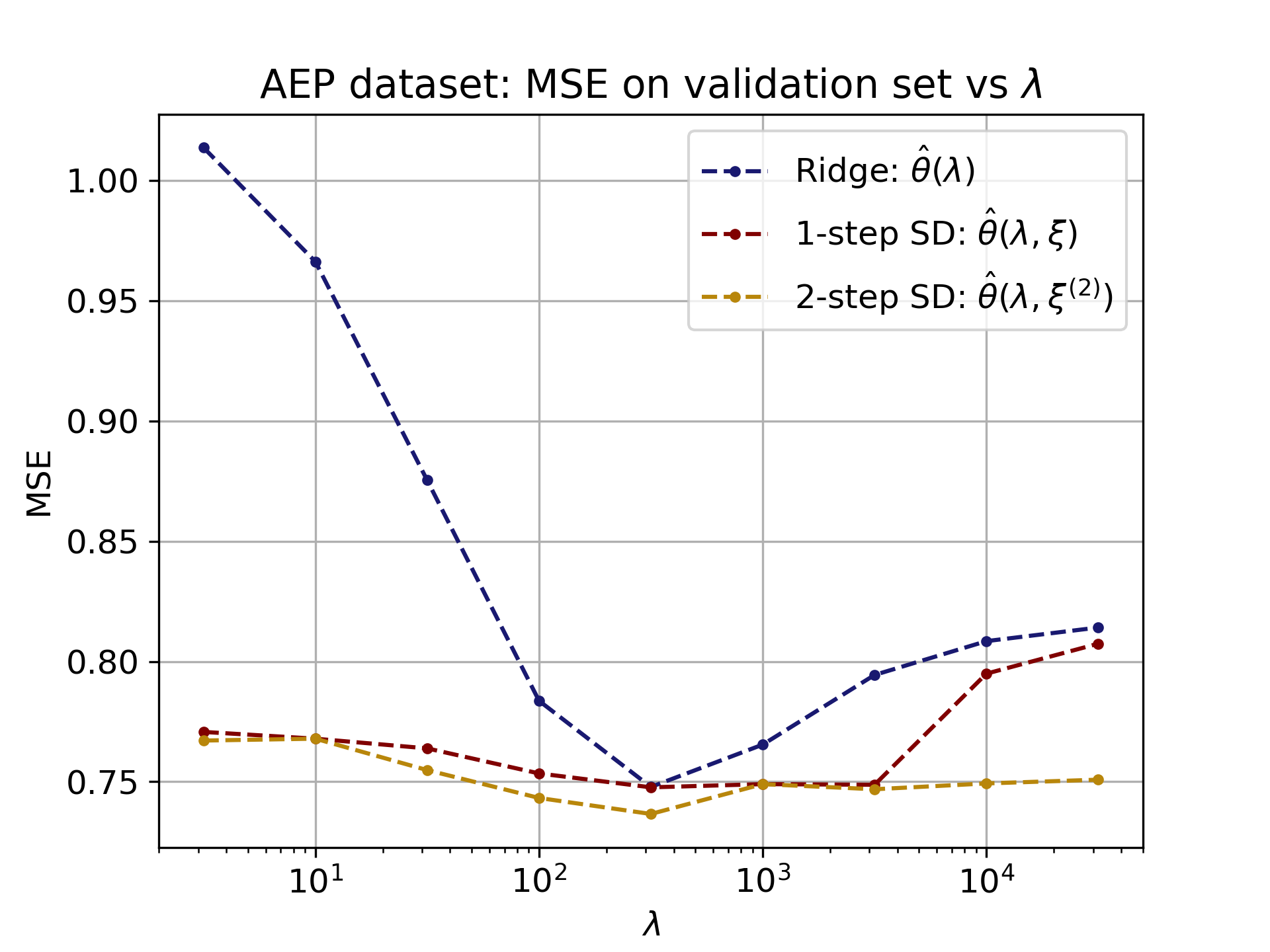}
  \caption{AEP dataset}
  \label{fig-regr-AEP}
\end{subfigure}
\caption{Validation set MSE vs $\lambda$ for three estimators: Ridge, $1$-step SD and $2$-step SD.}
\label{fig-regr-all}
\end{figure}

\section{Conclusion and Broader Impacts} \label{sec-conclusion}
In this paper, we theoretically studied the multi-step self-distillation for fixed design linear regression, with the goal of characterizing its performance compared to the single-step SD. Perhaps surprisingly, we demonstrated that the optimal multi-step SD can outperform the optimal single-step SD by a factor as large as $d$ in the estimator's excess risk, where $d$ is the input dimension of the regression.
Our analysis is limited by the fixed design assumption, and it would be useful to study the case of random design linear regression as well.
We empirically demonstrated the gains from using $2$-step SD on simple linear regression tasks. 
Larger scale empirical studies of multi-step SD, especially leveraging the insights of Section~\ref{sec-expts-hparams} on hyperparameter search, remain as a direction of future work.

Our contributions are largely on the theoretical understanding of multi-step self-distillation, and its potential performance gains. At a high-level, self-distillation can use data more effectively, since it allows us to extract more knowledge from the same training dataset. In today's age with data being one of the most important resources, this has positive potential impacts through more judicious use of data.
On the other hand, we propose to use multiple steps of self-distillation, requiring more compute and potentially contributing to higher environmental costs.

\section*{\newtext{Acknowledgements}}
\newtext{
This work is supported by NSF grants no.~2019844, 2112471, 2229876,  2134106, 2143493, and 2229881.
}

\bibliographystyle{abbrvnat}
\bibliography{citations}

\begin{thebibliography}{44}
\providecommand{\natexlab}[1]{#1}
\providecommand{\url}[1]{\texttt{#1}}
\expandafter\ifx\csname urlstyle\endcsname\relax
  \providecommand{\doi}[1]{doi: #1}\else
  \providecommand{\doi}{doi: \begingroup \urlstyle{rm}\Url}\fi

\bibitem[Ahn et~al.(2019)Ahn, Hu, Damianou, Lawrence, and
  Dai]{ahn2019variational}
S.~Ahn, S.~X. Hu, A.~Damianou, N.~D. Lawrence, and Z.~Dai.
\newblock Variational information distillation for knowledge transfer.
\newblock In \emph{Proceedings of the IEEE/CVF conference on computer vision
  and pattern recognition}, pages 9163--9171, 2019.

\bibitem[Allen-Zhu and Li(2023)]{allen-zhu2023towards}
Z.~Allen-Zhu and Y.~Li.
\newblock Towards understanding ensemble, knowledge distillation and
  self-distillation in deep learning.
\newblock In \emph{The Eleventh International Conference on Learning
  Representations}, 2023.

\bibitem[Ba and Caruana(2014)]{ba2014deep}
J.~Ba and R.~Caruana.
\newblock Do deep nets really need to be deep?
\newblock \emph{Advances in neural information processing systems}, 27, 2014.

\bibitem[Brooks et~al.(2014)Brooks, Pope, and
  Marcolini]{misc_airfoil_self-noise_291}
T.~Brooks, D.~Pope, and M.~Marcolini.
\newblock {Airfoil Self-Noise}.
\newblock UCI Machine Learning Repository, 2014.
\newblock {DOI}: https://doi.org/10.24432/C5VW2C.

\bibitem[Candanedo(2017)]{misc_appliances_energy_prediction_374}
L.~Candanedo.
\newblock {Appliances Energy Prediction}.
\newblock UCI Machine Learning Repository, 2017.
\newblock {DOI}: https://doi.org/10.24432/C5VC8G.

\bibitem[Caron et~al.(2021)Caron, Touvron, Misra, J{\'e}gou, Mairal,
  Bojanowski, and Joulin]{caron2021emerging}
M.~Caron, H.~Touvron, I.~Misra, H.~J{\'e}gou, J.~Mairal, P.~Bojanowski, and
  A.~Joulin.
\newblock Emerging properties in self-supervised vision transformers.
\newblock In \emph{Proceedings of the IEEE/CVF international conference on
  computer vision}, pages 9650--9660, 2021.

\bibitem[Chen et~al.(2017)Chen, Choi, Yu, Han, and
  Chandraker]{chen2017learning}
G.~Chen, W.~Choi, X.~Yu, T.~Han, and M.~Chandraker.
\newblock Learning efficient object detection models with knowledge
  distillation.
\newblock \emph{Advances in neural information processing systems}, 30, 2017.

\bibitem[Das and Sanghavi(2023)]{pmlr-v202-das23d}
R.~Das and S.~Sanghavi.
\newblock Understanding self-distillation in the presence of label noise.
\newblock In \emph{Proceedings of the 40th International Conference on Machine
  Learning}, pages 7102--7140. PMLR, 2023.

\bibitem[Furlanello et~al.(2018)Furlanello, Lipton, Tschannen, Itti, and
  Anandkumar]{furlanello2018born}
T.~Furlanello, Z.~Lipton, M.~Tschannen, L.~Itti, and A.~Anandkumar.
\newblock Born again neural networks.
\newblock In \emph{International Conference on Machine Learning}, pages
  1607--1616. PMLR, 2018.

\bibitem[Gou et~al.(2021)Gou, Yu, Maybank, and Tao]{gou2021knowledge}
J.~Gou, B.~Yu, S.~J. Maybank, and D.~Tao.
\newblock Knowledge distillation: A survey.
\newblock \emph{International Journal of Computer Vision}, 129\penalty0
  (6):\penalty0 1789--1819, 2021.

\bibitem[Hinton et~al.(2015)Hinton, Vinyals, and Dean]{distill2015}
G.~Hinton, O.~Vinyals, and J.~Dean.
\newblock Distilling the knowledge in a neural network.
\newblock In \emph{NIPS Deep Learning and Representation Learning Workshop},
  2015.
\newblock URL \url{http://arxiv.org/abs/1503.02531}.

\bibitem[Hu et~al.(2022)Hu, Li, Liu, Chen, Wang, and Liu]{hu2022teacher}
C.~Hu, X.~Li, D.~Liu, X.~Chen, J.~Wang, and X.~Liu.
\newblock Teacher-student architecture for knowledge learning: A survey.
\newblock \emph{arXiv preprint arXiv:2210.17332}, 2022.

\bibitem[Huang and Wang(2017)]{huang2017like}
Z.~Huang and N.~Wang.
\newblock Like what you like: Knowledge distill via neuron selectivity
  transfer.
\newblock \emph{arXiv preprint arXiv:1707.01219}, 2017.

\bibitem[Jankovic()]{jankovicQUAD}
V.~Jankovic.
\newblock Quadratic functions in several variables.
\newblock \url{http://elib.mi.sanu.ac.rs/files/journals/tm/15/tm821.pdf}.

\bibitem[Jeong and Chung(2024)]{jeong2024understanding}
H.~Jeong and H.~W. Chung.
\newblock Understanding self-distillation and partial label learning in
  multi-class classification with label noise.
\newblock \emph{arXiv preprint arXiv:2402.10482}, 2024.

\bibitem[Jha et~al.(2023)Jha, Hayase, and Oh]{jha2023label}
R.~Jha, J.~Hayase, and S.~Oh.
\newblock Label poisoning is all you need.
\newblock \emph{Advances in Neural Information Processing Systems},
  36:\penalty0 71029--71052, 2023.

\bibitem[Kelly et~al.()Kelly, Longjohn, and Nottingham]{ucirepo}
M.~Kelly, R.~Longjohn, and K.~Nottingham.
\newblock The uci machine learning repository.
\newblock \url{https://archive.ics.uci.edu}.

\bibitem[Kobak et~al.(2020)Kobak, Lomond, and Sanchez]{kobak2020optimal}
D.~Kobak, J.~Lomond, and B.~Sanchez.
\newblock The optimal ridge penalty for real-world high-dimensional data can be
  zero or negative due to the implicit ridge regularization.
\newblock \emph{Journal of Machine Learning Research}, 21\penalty0
  (169):\penalty0 1--16, 2020.

\bibitem[Li et~al.(2017)Li, Yang, Song, Cao, Luo, and Li]{li2017learning}
Y.~Li, J.~Yang, Y.~Song, L.~Cao, J.~Luo, and L.-J. Li.
\newblock Learning from noisy labels with distillation.
\newblock In \emph{Proceedings of the IEEE international conference on computer
  vision}, pages 1910--1918, 2017.

\bibitem[Li et~al.(2021)Li, Lyu, Koren, Lyu, Li, and Ma]{li2021neural}
Y.~Li, X.~Lyu, N.~Koren, L.~Lyu, B.~Li, and X.~Ma.
\newblock Neural attention distillation: Erasing backdoor triggers from deep
  neural networks, 2021.

\bibitem[Lin et~al.(2024)Lin, Gou, Gong, Liu, Shen, Xu, Lin, Yang, Jiao, Duan,
  et~al.]{lin2024rho}
Z.~Lin, Z.~Gou, Y.~Gong, X.~Liu, Y.~Shen, R.~Xu, C.~Lin, Y.~Yang, J.~Jiao,
  N.~Duan, et~al.
\newblock Rho-1: Not all tokens are what you need.
\newblock \emph{arXiv preprint arXiv:2404.07965}, 2024.

\bibitem[Liu et~al.(2018)Liu, Wang, and Matwin]{liu2018improving}
X.~Liu, X.~Wang, and S.~Matwin.
\newblock Improving the interpretability of deep neural networks with knowledge
  distillation.
\newblock In \emph{2018 IEEE International Conference on Data Mining Workshops
  (ICDMW)}, pages 905--912. IEEE, 2018.

\bibitem[Liu et~al.(2019)Liu, Chen, Liu, Qin, Luo, and Wang]{liu2019structured}
Y.~Liu, K.~Chen, C.~Liu, Z.~Qin, Z.~Luo, and J.~Wang.
\newblock Structured knowledge distillation for semantic segmentation.
\newblock In \emph{Proceedings of the IEEE/CVF Conference on Computer Vision
  and Pattern Recognition}, pages 2604--2613, 2019.

\bibitem[Lopez-Paz et~al.(2015)Lopez-Paz, Bottou, Sch{\"o}lkopf, and
  Vapnik]{lopez2015unifying}
D.~Lopez-Paz, L.~Bottou, B.~Sch{\"o}lkopf, and V.~Vapnik.
\newblock Unifying distillation and privileged information.
\newblock \emph{arXiv preprint arXiv:1511.03643}, 2015.

\bibitem[Menon et~al.(2021)Menon, Rawat, Reddi, Kim, and
  Kumar]{menon2021statistical}
A.~K. Menon, A.~S. Rawat, S.~Reddi, S.~Kim, and S.~Kumar.
\newblock A statistical perspective on distillation.
\newblock In \emph{International Conference on Machine Learning}, pages
  7632--7642. PMLR, 2021.

\bibitem[Mobahi et~al.(2020)Mobahi, Farajtabar, and Bartlett]{mobahi2020self}
H.~Mobahi, M.~Farajtabar, and P.~Bartlett.
\newblock Self-distillation amplifies regularization in hilbert space.
\newblock \emph{Advances in Neural Information Processing Systems}, pages
  3351--3361, 2020.

\bibitem[Papernot et~al.(2016)Papernot, McDaniel, Wu, Jha, and
  Swami]{papernot2016distillation}
N.~Papernot, P.~McDaniel, X.~Wu, S.~Jha, and A.~Swami.
\newblock Distillation as a defense to adversarial perturbations against deep
  neural networks.
\newblock In \emph{2016 IEEE symposium on security and privacy (SP)}, pages
  582--597. IEEE, 2016.

\bibitem[Phuong and Lampert(2019)]{phuong2019towards}
M.~Phuong and C.~Lampert.
\newblock Towards understanding knowledge distillation.
\newblock In \emph{International conference on machine learning}, pages
  5142--5151. PMLR, 2019.

\bibitem[Romero et~al.(2014)Romero, Ballas, Kahou, Chassang, Gatta, and
  Bengio]{romero2014fitnets}
A.~Romero, N.~Ballas, S.~E. Kahou, A.~Chassang, C.~Gatta, and Y.~Bengio.
\newblock Fitnets: Hints for thin deep nets; 2014.
\newblock \emph{arXiv preprint arXiv:1412.6550}, 3, 2014.

\bibitem[Rusu et~al.(2016)Rusu, Colmenarejo, Gulcehre, Desjardins, Kirkpatrick,
  Pascanu, Mnih, Kavukcuoglu, and Hadsell]{rusu2015policy}
A.~A. Rusu, S.~G. Colmenarejo, C.~Gulcehre, G.~Desjardins, J.~Kirkpatrick,
  R.~Pascanu, V.~Mnih, K.~Kavukcuoglu, and R.~Hadsell.
\newblock Policy distillation.
\newblock In \emph{International Conference on Learning Representations}, 2016.

\bibitem[Sun et~al.(2019)Sun, Cheng, Gan, and Liu]{sun2019patient}
S.~Sun, Y.~Cheng, Z.~Gan, and J.~Liu.
\newblock Patient knowledge distillation for bert model compression.
\newblock \emph{arXiv preprint arXiv:1908.09355}, 2019.

\bibitem[Urban et~al.(2017)Urban, Geras, Kahou, Aslan, Wang, Mohamed,
  Philipose, Richardson, and Caruana]{urban2017do}
G.~Urban, K.~J. Geras, S.~E. Kahou, O.~Aslan, S.~Wang, A.~Mohamed,
  M.~Philipose, M.~Richardson, and R.~Caruana.
\newblock Do deep convolutional nets really need to be deep and convolutional?
\newblock In \emph{International Conference on Learning Representations}, 2017.

\bibitem[Vito(2016)]{misc_air_quality_360}
S.~Vito.
\newblock {Air Quality}.
\newblock UCI Machine Learning Repository, 2016.
\newblock {DOI}: https://doi.org/10.24432/C59K5F.

\bibitem[Xia et~al.(2022)Xia, Wang, Ding, Wei, and Chen]{xia2022eliminating}
J.~Xia, T.~Wang, J.~Ding, X.~Wei, and M.~Chen.
\newblock Eliminating backdoor triggers for deep neural networks using
  attention relation graph distillation, 2022.

\bibitem[Yang et~al.(2019{\natexlab{a}})Yang, Xie, Qiao, and
  Yuille]{yang2019training}
C.~Yang, L.~Xie, S.~Qiao, and A.~L. Yuille.
\newblock Training deep neural networks in generations: A more tolerant teacher
  educates better students.
\newblock In \emph{Proceedings of the AAAI Conference on Artificial
  Intelligence}, volume~33, pages 5628--5635, 2019{\natexlab{a}}.

\bibitem[Yang et~al.(2019{\natexlab{b}})Yang, Xie, Su, and
  Yuille]{yang2019snapshot}
C.~Yang, L.~Xie, C.~Su, and A.~L. Yuille.
\newblock Snapshot distillation: Teacher-student optimization in one
  generation.
\newblock In \emph{Proceedings of the IEEE/CVF Conference on Computer Vision
  and Pattern Recognition}, pages 2859--2868, 2019{\natexlab{b}}.

\bibitem[Yim et~al.(2017)Yim, Joo, Bae, and Kim]{yim2017gift}
J.~Yim, D.~Joo, J.~Bae, and J.~Kim.
\newblock A gift from knowledge distillation: Fast optimization, network
  minimization and transfer learning.
\newblock In \emph{Proceedings of the IEEE conference on computer vision and
  pattern recognition}, pages 4133--4141, 2017.

\bibitem[Yoshida and Fujino(2020)]{yoshida}
K.~Yoshida and T.~Fujino.
\newblock Countermeasure against backdoor attack on neural networks utilizing
  knowledge distillation.
\newblock \emph{Journal of Signal Processing}, 2020.

\bibitem[Zagoruyko and Komodakis(2017)]{zagoruyko2016paying}
S.~Zagoruyko and N.~Komodakis.
\newblock Paying more attention to attention: Improving the performance of
  convolutional neural networks via attention transfer.
\newblock In \emph{International Conference on Learning Representations}, 2017.

\bibitem[Zhang et~al.(2019)Zhang, Song, Gao, Chen, Bao, and Ma]{zhang2019your}
L.~Zhang, J.~Song, A.~Gao, J.~Chen, C.~Bao, and K.~Ma.
\newblock Be your own teacher: Improve the performance of convolutional neural
  networks via self distillation.
\newblock In \emph{Proceedings of the IEEE/CVF International Conference on
  Computer Vision}, pages 3713--3722, 2019.

\bibitem[Zhang et~al.(2021)Zhang, Bao, and Ma]{zhang2021self}
L.~Zhang, C.~Bao, and K.~Ma.
\newblock Self-distillation: Towards efficient and compact neural networks.
\newblock \emph{IEEE Transactions on Pattern Analysis and Machine
  Intelligence}, 44\penalty0 (8):\penalty0 4388--4403, 2021.

\bibitem[Zhang et~al.(2018)Zhang, Xiang, Hospedales, and Lu]{zhang2018deep}
Y.~Zhang, T.~Xiang, T.~M. Hospedales, and H.~Lu.
\newblock Deep mutual learning.
\newblock In \emph{Proceedings of the IEEE conference on computer vision and
  pattern recognition}, pages 4320--4328, 2018.

\bibitem[Zhang and Sabuncu(2020)]{zhang2020self}
Z.~Zhang and M.~Sabuncu.
\newblock Self-distillation as instance-specific label smoothing.
\newblock \emph{Advances in Neural Information Processing Systems},
  33:\penalty0 2184--2195, 2020.

\bibitem[Zhu et~al.(2024)Zhu, Jordan, and Jiao]{zhu2024iterative}
B.~Zhu, M.~I. Jordan, and J.~Jiao.
\newblock Iterative data smoothing: Mitigating reward overfitting and
  overoptimization in rlhf.
\newblock \emph{arXiv preprint arXiv:2401.16335}, 2024.

\end{thebibliography}
\newpage
\appendix


\section{Notation} \label{sec-app-notation}

In this short section, we collect some notation used throughout the proofs.

\textbf{Decomposition of }$\X$. Let ${\rm rank}(\X)=r$ and the SVD of $\X$ be $\X = \sum_{j=1}^r \s_j \u_j \v_j^T$ where $\s_1 \geq \s_2 \geq \cdots \geq \s_r > 0$, and each $\u_j \in \R^d$ and $\v_j \in \R^n$. Further, let $\{ \u_1, \u_2, \cdots, \u_d\}$ be the full set of left singular vectors of $\X$ (even those corresponding to zero singular values), forming an orthonormal basis of $\R^d$. Let $\Ud \in \R^{d \times d}$ and $\Ur \in \R^{d \times r}$ denote the left singular matrix of $\X$ for the full and truncated set of left singular vectors respectively. Similarly, let $\Vd \in \R^{n \times d}$ and $\Vr \in \R^{n \times r}$ denote the full and truncated right singular matrix. Let $\Sd \in \R^{d \times d}$ and $\Sr \in \R^{r \times r}$ denote the collection of the singular values (with and without zeros respectively). Then, it holds that
\begin{equation} \label{X-SVD}
    \X = \Ur \Sr \Vr^\top = \Ud \Sd \Vd^\top \; \;.
\end{equation}
\textbf{Indices}. Throughout the text, the indices $i$ lie in $[k]$, i.e. they denote subsequent steps of self-distillation. The indices $j$ lie in $[r]$ or $[d]$, ie they denote dimensions of the $d$-dimensional space (aligned with vectors of $\Ur$ or $\Ud$). $v_j, V_{i,j}$ will denote indexing into a vector $v$, matrix $V$. There is one exception to $v_j$ denoting a vector's $j^{th}$ element, which is the below. \\
\textbf{Components of $\stheta$ on $\Ud$}. We will denote $\stheta_j := \langle \stheta, \u_j \rangle^2$, $j \in [d]$ as the components of $\stheta$ onto $\X$'s left singular space. Note that $\sum_{j=1}^d \stheta_j = \| \stheta \|_2^2$. 

\section{Discussion on norm used in excess risk metric} \label{sec-app-normdiscussion}

We point out that \citet{pmlr-v202-das23d} used the $\|.\|_2$ norm instead of the more natural $\|.\|_{\hSigma_n}$ norm to measure their fixed design excess risk (in eq (\ref{eq-ER-fixed})).
Almost all our results have an equivalent version in the $\|.\|_2$ norm setting also, since the only difference is in the relative weighing of the underlying dimensions of variation, i.e. with the $\|.\|_{\hSigma_n}$ norm, $\forall j \in [d]$, direction $\u_j$ is weighed by $\s_j^2/n$ instead of a constant $1$ weight (independent of $j$). 
In particular, we also present the version of \cite[Eq.~(9) from Theorem 3.8]{pmlr-v202-das23d} that will result in a strict dominance like Eq.~\eqref{eq-thm-sdgain} under the $\hSigma_n$ norm, i.e.
\begin{equation} \label{eq-prop-sdgain}
    \min_{\lambda \geq 0, \xi \in \R} \underbrace{\E_{\Noise} \left[ \| \htheta(\lambda, \xi) - \stheta \|_{\hSigma_n}^2  \right]}_{= \ExcessRisk(\htheta(\lambda, \xi))} \;<\; \min_{\lambda \geq 0} \underbrace{\E_{\Noise} \left[ \| \htheta(\lambda) - \stheta \|_{\hSigma_n}^2  \right]}_{=\ExcessRisk(\htheta(\lambda))} \; \;.
\end{equation}
\begin{proposition} \label{prop-condition}
    Let $\slambda := \argmin_{\lambda > 0} \ExcessRisk( \htheta(\lambda) )$. Then Eq.~\eqref{eq-prop-sdgain} holds on a problem instance $(\X, \stheta, \gamma^2)$ when
    \begin{equation} \label{eq-prop-condition}
        \sum_{k=1}^r \sum_{j=1}^{k-1}  \frac{\s_j^4 \s_k^4 \left( \s_j^2 - \s_k^2\right) \left( \langle \stheta, \u_k \rangle^2 - \langle \stheta, \u_j \rangle^2 \right) }{\left( \slambda + \s_j^2 \right)^4\left( \slambda + \s_k^2 \right)^4} < 0 \; \; .
    \end{equation}
\end{proposition}
Note that this differs from \cite[Eq.~(9)]{pmlr-v202-das23d} in just one respect: it has $\s_j^4 \s_k^4$ instead of $\s_j^2 \s_k^2$.

\section{Details on $\xi$ parameters for general $k$-step SD} \label{sec-app-xis}

\subsection{Full $k$-step SD is representationally no larger than Repeated $k$-step SD} \label{sec-app-xis-fullVinc}

We first illustrate the Full $k$-step SD in Figure \ref{fig-sketch}. The repeated version introduces $k$ extra hyperparameters in the form of $\bxik \in \R^k$ parameters, whereas the full version introduces $\nicefrac{k(k+1)}{2}$.  
\begin{figure}[h]
\centering
\begin{tikzpicture}[->, >=stealth', auto, semithick, node distance=1.5cm]
\tikzstyle{state}=[draw, thick, fill=buff, minimum size=8mm]

\node[state] (S0) {$S_0$};
\node[state, fill=burlywood] (S1) [right of=S0] {$S_1$};
\node[state, fill=bronze] (S2) [right of=S1] {$S_2$};
\node[state, draw=none, fill=white] (Sdots) [right of =S2] {$\cdots$};
\node[state, fill=brown] (Sk) [right of=Sdots] {$S_k$};
\node[state, draw=none, fill=white] (lbla) at (10, 0) {\textit{Repeated $k$-step self-distillation}};
\path (S0) edge (S1)
      (S1) edge (S2)
      (S2) edge (Sdots)
      (Sdots) edge (Sk);

\node[state] (S0a) [below of=S0] {$S_0$};
\node[state, fill=burlywood] (S1a) [right of=S0a] {$S_1$};
\node[state, fill=bronze] (S2a) [right of=S1a] {$S_2$};
\node[state, draw=none, fill=white] (Sdotsa) [right of =S2a] {$\cdots$};
\node[state, fill=brown] (Ska) [right of=Sdotsa] {$S_k$};
\node[state, draw=none, fill=white] (lblb) at (10, -1.5) {\textit{Full $k$-step self-distillation}};
\path (S0a) edge (S1a)
      (S1a) edge (S2a)
      (S2a) edge (Sdotsa)
      (Sdotsa) edge (Ska)
      (S0a) edge[bend left] (S2a)
      (S0a) edge[bend left] (Ska)
      (S1a) edge[bend left] (Ska)
      (S2a) edge[bend left] (Ska);
\end{tikzpicture}
\caption{Illustrating two possible generalizations of $1$-step SD to a $k$-step process.}
\label{fig-sketch}
\end{figure}
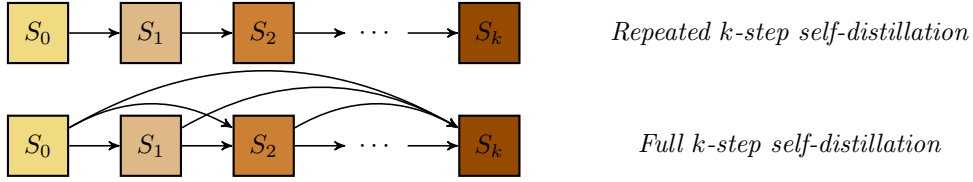

Consider the case of $k=2$, since that is the lowest value of $k$ for which the full version and the repeated version differ. Figure \ref{fig-2step-close} illustrates this difference explicitly. We will show that the freedom of $\txi \in \R^3$ is no more than the freedom of $\bxi \in \R^2$. This shows the equivalence of the full $2$-step and the repeated $2$-step versions, when $\txi \in \R^3$, $\bxi \in \R^2$ are free parameters optimized over the entire respective spaces.
Such equivalence for the general $k$-step case is then easy to see.

Nodes $S_0, \tS_0$ are both solving the ridge regression problem (\ref{est-ridge-2}). Let $\theta_0 = \ttheta_0 =  \Om_{\lambda}^{-1} \X \Y$ denote the estimator for both these nodes.
Similarly $S_1$ and $\tS_1$ are solving the same problem, although with different parameters. We have $\theta_1 (\bxi_1) =  \Om_{\lambda}^{-1} \X \left( \bxi_1 \cdot \X^\top \btheta_0 + (1 - \bxi_1) \cdot \Y \right)$ (and similarly, one can write $\ttheta_1$ with $\txi_1$ instead of $\bxi_1$). 
Now the node $S_2$ is also solving a $1$ parameter supervised SD problem, so $\theta_2 (\bxi_1, \bxi_2) =  \Om_{\lambda}^{-1} \X \left( \bxi_2 \cdot \X^\top \theta_1 (\bxi_1) + (1 - \bxi_2) \cdot \Y \right)$. Expanding this gives
\begin{equation} \label{eq-SD2step-inc}
    \theta_2 (\bxi_1, \bxi_2) =  \left\{ (1 - \bxi_2) \cdot \Id +  (\bxi_2 - \bxi_1 \bxi_2) \cdot \Om_{\lambda}^{-1} \X \Xt + \bxi_1 \bxi_2 \cdot (\Om_{\lambda}^{-1} \X \Xt)^2 \right\} \Om_{\lambda}^{-1} \X \Y \; \;.
\end{equation}
But the optimization problem for $\tS_2$ is a $2$ parameter supervised SD. It evaluates to
\begin{equation*}
    \argmin_{\theta \in \R^d} \left( \frac{\txi_{2a}}{2} \| \Xt \ttheta_0 - \Xt \theta \|^2 + \frac{\txi_{2b}}{2} \| \Xt \ttheta_1 - \Xt \theta \|^2 + \frac{(1 - \txi_{2a} - \txi_{2b})}{2} \| \Y - \Xt \theta \|^2 + \frac{\lambda}{2} \| \theta \|^2 \right) \; \;.
\end{equation*}
Following through a similar calculation, we observe that $\ttheta_2$ for node $\tS_2$ is given by
\begin{equation} \label{eq-SD2step-full}
    \ttheta_2 (\txi_1, \txi_{2a}, \txi_{2b}) =  \left\{ (1 - \txi_{2a} - \txi_{2b}) \cdot \Id +  (\txi_{2a} + \txi_{2b} - \txi_1 \txi_{2b} ) \cdot \Om_{\lambda}^{-1} \X \Xt + \txi_1 \txi_{2b} \cdot (\Om_{\lambda}^{-1} \X \Xt)^2 \right\} \Om_{\lambda}^{-1} \X \Y \; \;.
\end{equation}

\begin{figure}[H]
\centering
\begin{tikzpicture}[->, >=stealth', auto, semithick, node distance=2.5cm]
\tikzstyle{state}=[draw, thick, fill=buff, minimum size=8mm]

\node[state] (S0) {$S_0$};
\node[state, fill=burlywood] (S1) [right of=S0] {$S_1$};
\node[state, fill=bronze] (S2) [right of=S1] {$S_2$};
\node[state, draw=none, fill=white] (lbla) at (10, 0) {\textit{Repeated $2$-step self-distillation}};
\path (S0) edge node {$\bxi_1$} (S1)
      (S1) edge node {$\bxi_2$} (S2);

\node[state] (S0a) [below of=S0] {$\tS_0$};
\node[state,fill=burlywood] (S1a) [right of=S0a] {$\tS_1$};
\node[state,fill=bronze] (S2a) [right of=S1a] {$\tS_2$};
\node[state, draw=none, fill=white] (lblb) at (10, -2.5) {\textit{Full $2$-step self-distillation}};
\path (S0a) edge node {$\txi_1$} (S1a)
      (S1a) edge node {$\txi_{2b}$} (S2a)
      (S0a) edge[bend left] node {$\txi_{2a}$} (S2a);
\end{tikzpicture}    
\caption{Repeated vs Full $2$-step SD. We show that the extra freedom of the parameter $\txi_{2a}$ does not provide any additional freedom, when the other two $\txi_{1}, \txi_{2b}$ are optimized as free parameters.}
\label{fig-2step-close}
\end{figure}
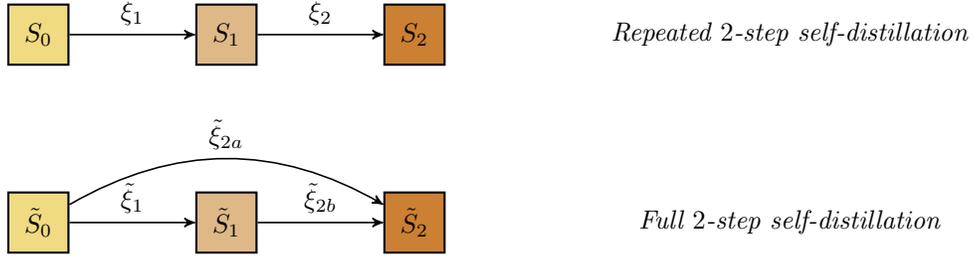

Equations (\ref{eq-SD2step-inc}) vs (\ref{eq-SD2step-full}) show that the full $2$-step offers the same freedom as the repeated $2$-step. However the repeated version has one shortcoming.
To generate the optimal $2$-step SD estimator, $\bxi_1$ needs to be different than the value needed for generating the optimal $1$-step SD estimator.
That is, the repeated $k$-step version, with its $k$ free parameters, allows only to generate the \emph{final} $k^{th}$-step estimator as the optimal one (i.e. if we choose the $(\bxi_1, \bxi_2)$ values so that the $2^{nd}$ estimator is the optimal $2$-step SD estimator, then the $1^{st}$ estimator with the chosen $\bxi_1$ won't be the optimal $1$-step SD estimator).
Whereas the full $k$-step version, with all its $k(k+1)/2$ free parameters, allows us to generate a sequence of \emph{all} $k$ optimal estimators.

\subsection{Proof that $k$-step SD estimator with $\bxik \in \R^k$ will have the form given in eq (\ref{est-sdk})} \label{sec-app-xis-objective}

Consider Figure \ref{fig-sdk}. $\bxik \in \R^k$ is the set of actual imitation parameters used for running $k$-step (repeated) SD. Let $\htheta (\lambda, \bxik) $ denote the $k$-step estimator generated by using $\bxik$ parameters. In what follows, we will prove that $\htheta (\lambda, \bxik) $ will have the form given in eq (\ref{est-sdk}), with $\xik \in \R^k$ as the described reparametrization of $\bxik \in \R^k$. 
Since we're in the repeated version, the $k^{th}$ step objective will be a combination of losses w.r.to ground-truth labels and predictions from the $(k-1)^{th}$ step.
So, the objective for the $k^{th}$ step is 
\begin{equation} \label{eq-obj-Sk}
    \htheta(\lambda, \bxik) := \argmin_{\theta \in \R^d} \left( \frac{\bxik_k}{2} \| \Xt \htheta(\lambda, \bxi^{(k-1)}) - \Xt \theta \|^2 + \frac{(1 - \bxik_k)}{2} \| \Y - \Xt \theta \|^2 + \frac{\lambda}{2} \| \theta \|^2 \right) \; \;.
\end{equation}
Note that $\htheta(\lambda, \bxik)$ recursively depends on predictions from the previous $\htheta(\lambda, \bxi^{(k-1)})$. This objective is of the form in eq (\ref{est-sd1}), so similar to eq (\ref{est-sd1-2}), we have the following expression
\begin{align}
    \htheta(\lambda, \bxik) &= \left( \X \Xt + \lambda \Id \right)^{-1} \X \left\{ \bxik_k \cdot \X^\top \htheta( \lambda, \bxi^{(k-1)}) + (1 - \bxik_k) \cdot \Y \right\} \label{eq-est-Sk} \\
    &= \left\{ (1 - \bxik_k) \cdot \Om_{\lambda}^{-1} \X \Y + \bxik_k \cdot \Om_{\lambda}^{-1} \X \X^\top \htheta( \lambda, \bxi^{(k-1)}) \right\} \; \;. \label{eq-est-Sk-2}
\end{align}
Using this, we can inductively prove that the general form of this estimator is captured by Eq.~\eqref{est-sdk}. \\ 
\textbf{Claim}. With the described reparametrization, i.e.,
$\xik_i := (1 - \bxik_{k-i}) \prod_{l=k-i+1}^{k} \bxik_{l}$ for each $i \in [k]$ (where we let $\bxik_0=0$), Eq.~\eqref{est-sdk} is the solution to the recursive form Eq.~\eqref{eq-est-Sk-2}.

\textbf{Proof}. We will prove this by induction. \\
\textit{Base Case}. From eqs (\ref{est-sdk}) and (\ref{eq-est-Sk-2}), the case for $k=1$ is true (with $\barxi^{(1)} = \bxi^{(1)}$). \\
\textit{Inductive Step}. Assuming Eq.~\eqref{est-sdk} captures the solution of Eq.~\eqref{eq-est-Sk-2} for $k-1$, we get
\begin{align*}
    \htheta(\lambda, \bxik) &= \left\{ \left\{ 1 -  \bxik_k \right\} \Id + \bxik_k \Om_\lambda^{-1} \X \Xt \left(  \left\{ 1 - \sum_{i=1}^{k-1} \barxi^{(k-1)}_i \right\} \Id + \sum_{i=1}^{k-1} \barxi^{(k-1)}_i \left( \Om_\lambda^{-1} \X \Xt \right)^i \right) \right\}  \cdot \Om_\lambda^{-1} \X \Y \; \;.
\end{align*}
This again satisfies the form in equation (\ref{est-sdk}), when the following coefficients match
\begin{align*}
    \bxik_k &= \sum_{i=1}^k \xik_i \; \;, \\
    \bxik_k \cdot ( 1 - \sum_{i=1}^{k-1} \barxi^{(k-1)}_i )  &= \xik_1 \; \;, \\
    \bxik_k \cdot \barxi^{(k-1)}_{i-1} &= \xik_i  \quad \quad \forall i \in \{2, 3, \cdots, k\} \; \;.
\end{align*}
One can then see that the described reparametrization makes the above hold true. 

\textbf{Remark}. 
Since $\htheta(\lambda, \bxi^{(1)})$ is simply the $1$-step SD estimator with the form $\htheta(\lambda, \bxi^{(1)}) = \Pre \cdot  \Om_{\lambda}^{-1} \X \Y $ for some preconditioner $\Pre$, plugging this in the equation eq (\ref{eq-est-Sk-2}), we realize that we can factor out $\Om_{\lambda}^{-1} \X \Y$ (i.e. the ridge solution) from the expression on the right side. That is, $\htheta(\lambda, \bxi^{(2)}) = \Pre' \cdot  \Om_{\lambda}^{-1} \X \Y $ for some different pre-conditioner $\Pre'$. This is why inductively we get that $\htheta(\lambda, \bxik)$, $k \geq 1$ all produce a pre-conditioning on the ridge solution, as shown in eq (\ref{est-sdk}). Further, eq (\ref{eq-est-Sk-2}) also dictates why we get increasing powers of the term $\Om_{\lambda}^{-1} \X \Xt$ in the final expression.


\subsection{Explicit reparametrization for $k= 2, 3$} \label{sec-app-xis-reparam}

We explicitly demonstrate the reparametrization for $k=2, 3$.
As noted in Section \ref{sec-expts-hparams} (and Appendix \ref{sec-app-choosing-xi}), owing to the quadratic form of excess risk in $\xik$, one can find the optimal $\xik$ analytically. It can then be translated back to the original $\bxik$ as follows.

For $k=2$, the form is (dropping the $.^{(2)}$ for ease)
\begin{equation} \label{trans-2stepSD}
    \bxi_1 = 1 - \frac{\barxi_1}{\left( \barxi_1 + \barxi_2 \right)}
    \quad, \quad \quad \quad \quad 
    \bxi_2 = \barxi_1 + \barxi_2 \; \;.
\end{equation}
For $k=3$, the form is (dropping the $.^{(3)}$ for ease)
\begin{equation} \label{trans-3stepSD}
    \bxi_1 = 1 - \frac{\barxi_2}{\barxi_2 + \barxi_3} 
    \quad, \quad \quad 
    \bxi_2 = 1 - \frac{\barxi_1}{\barxi_1 + \barxi_2 + \barxi_3} 
    \quad, \quad \quad 
    \bxi_3 = \barxi_1 + \barxi_2 + \barxi_3 \; \;.
\end{equation}


\section{Algebraic expansion of $\htheta(\lambda, \bxik)$ from eq (\ref{est-sdk})} \label{sec-app-proof-lem-est-sdk-expansion}

\begin{lemma} \label{lem-est-sdk-expansion}
The estimator $\htheta(\lambda, \bxik)$ given in eq (\ref{est-sdk}) can be expanded as
    \begin{align*}
        \htheta(\lambda, \bxik) &= \sum_{j=1}^d \left( 1 - \sum_{i=1}^k \xik_i \left\{ 1 - \left( \frac{\s_j^2}{\lambda + \s_j^2} \right)^i \right\} \right) \cdot \frac{\s_j^2}{\lambda + \s_j^2} \cdot \langle \stheta, \u_j \rangle \cdot \u_j \\ & \hspace{25pt} + \sum_{j=1}^d \left( 1 - \sum_{i=1}^k \xik_i \left\{ 1 - \left( \frac{\s_j^2}{\lambda + \s_j^2} \right)^i \right\} \right) \cdot \frac{\s_j}{\lambda + \s_j^2} \cdot \langle \Noise,  \v_j \rangle \cdot \u_j
    \end{align*}
\end{lemma}

\begin{proof}
    The expansion relies on $\X \Xt = \Ud \Sd^2 \Ud^\top$ and $\X \Xt + \lambda \Id = \Ud (\Sd^2 + \lambda \Id) \Ud^\top$. $\Ud$ is orthonormal ($\Ud \Ud^\top = \Id = \Ud^\top \Ud$), which neatly cancels all occurences of $\Ud$ in the middle, and allows us to directly combine the matrices in the eigenvalue space.
\end{proof}

\section{Detailed version and a proof of Theorem~\ref{thm-FD-separation}} \label{sec-app-proof-thm-FD-separation}

We first write a detailed version of the theorem that exactly characterizes the family of problem instances that admits the separation. We then present a proof. We point the reader to Appendix~\ref{sec-app-notation} for notations used throughout the proof.

\begin{theorem}[Detailed version] \label{thm-FD-separation-App}
Consider the following conditions on $\stheta, \X$ that characterize a family of problem instances $\{ (\X, \stheta, \gamma^2) \}$:
\begin{enumerate}
    \item $\stheta = \| \stheta \| \cdot \u_1$ (implying $\stheta_1 = \| \stheta \|^2$ and $\stheta_j = 0$ for $j \geq 2$) \hfill [Assumption \ref{assump-FD-sj}.\ref{assump-FD-sj2}]
    \item $\| \stheta \|^2 > 99 \cdot \left( \nicefrac{r \gamma^2}{\s_r^2} \right) \cdot \left( \nicefrac{\s_1^2}{\s_r^2} \right) $
    \item Assumption \ref{assump-FD-sj}.\ref{assump-FD-sj1} holds on the singular values of $\X$ \hfill [Assumption \ref{assump-FD-sj}.\ref{assump-FD-sj1}]
    \item $( \nicefrac{\s_1^2}{\s_r^2} )\leq 2$ holds on the singular values of $\X$
\end{enumerate}
\begin{align*}
    \intertext{Under Condition \#1 + Condition \#3, it holds that}
    \exists \lambda > 0, \exists \bxir \in \R^r, \hspace{10pt} &\ExcessRisk\left( \htheta (\lambda, \bxir) \right) \leq \frac{\gamma^2}{n} \tag{Rewriting eq (\ref{eq-thm-FD-sep-SD})} 
    \intertext{Under Condition \#1 + Condition \#2 + Condition \#4, it holds that}
    \forall \lambda > 0, \forall \xi \in \R, \hspace{10pt} &\ExcessRisk\left( \htheta (\lambda, \xi) \right) \geq \frac{0.99}{2^{9}} \cdot \frac{r\gamma^2}{n} \tag{Rewriting eq (\ref{eq-thm-FD-sep-oneSD})} 
    \intertext{Under Condition \#1 + Condition \#2, it holds that}
    \forall \lambda > 0, \hspace{10pt} &\ExcessRisk\left(\htheta(\lambda) \right) \geq  0.98 \cdot \frac{r \gamma^2}{n} \tag{Rewriting eq (\ref{eq-thm-FD-sep-Ridge})}
\end{align*}
\end{theorem}

\begin{proof}
We will analyze all three: $k$-step SD, ridge, and $1$-step SD.

\textbf{$k$-step SD}: Since Assumption \ref{assump-FD-sj}.\ref{assump-FD-sj1} holds, from Theorem \ref{thm-sdkPre-optPre} we directly have
\begin{align}
    \exists \lambda > 0, \exists \bxir \in \R^r, \quad \ExcessRisk\left( \htheta \left( \lambda, \bxir \right) \right) &= \frac{\gamma^2}{n} \sum_{j=1}^r \frac{ \stheta_j }{\left(\stheta_j  + \frac{\gamma^2}{\s_j^2} \right)} \\
    \intertext{Since $\stheta_j = 0, j \geq 2$ from Condition \#1, we have}
    &= \frac{\gamma^2}{n} \cdot \frac{ \stheta_1}{\stheta_1 + \frac{\gamma^2}{\s_1^2}} \label{eq-SD-UB-proof-PREV} \\
    & \leq \frac{\gamma^2}{n} \label{eq-SD-UB-proof}
\end{align}
This completes the proof for (\ref{eq-thm-FD-sep-SD}). 

\textbf{Ridge}: We will now show (\ref{eq-thm-FD-sep-Ridge}), by characterizing the Excess Risk for the ridge estimator in this regime, showing that it can be upto $r$ times worse.
From Lemma \ref{lem-Ridge}, we realize that the Excess Risk expression for the simple ridge estimator $\htheta(\lambda)$ is given by
\begin{align} 
    \forall \lambda > 0, \quad \ExcessRisk\left( \htheta(\lambda) \right) &= \frac{1}{n} \sum_{j=1}^r \frac{\left( \lambda^2 \stheta_j + \gamma^2 \s_j^2 \right) \cdot \s_j^2}{(\lambda + \s_j^2)^2} \label{eq-ER-Ridge-inproof} \\
    &\geq \underbrace{\frac{\gamma^2}{n} \sum_{j=1}^r \frac{ \s_j^4 }{(\lambda + \s_j^2)^2}}_{\text{Just the Variance term}} \label{eq-ERVar-Ridge}
\end{align}
Inequality (\ref{eq-ERVar-Ridge}) above comes from ignoring the bias term (since it is non-negative). Note that the variance term is a \emph{decreasing} function of $\lambda$.
And again from Lemma \ref{lem-Ridge}, we get the following expression for $\slambda > 0$ that minimizes the $\ExcessRisk$.
\begin{align}
    \slambda &= \gamma^2 \cdot \frac{\sum_{j=1}^r \frac{\s_j^4}{\left(\slambda + \s_j^2\right)^3}}{\sum_{j=1}^r \frac{\stheta_j \s_j^4}{\left(\slambda + \s_j^2\right)^3}} \label{eq-slambda-exp-inproof} \\
    &= \frac{\gamma^2}{\| \stheta \|^2} \cdot \sum_{j=1}^r \left( \frac{\s_j}{\s_1} \right)^4 \cdot \left( \frac{\slambda + \s_1^2}{\slambda + \s_j^2} \right)^3 \tag{$\stheta_j = 0$, $j \geq 2$ and $\stheta_1 = \| \stheta \|^2$ from Cond \#1} \\
    & \leq \frac{\gamma^2}{\| \stheta \|^2} \cdot \sum_{j=1}^r \left( \frac{\s_j}{\s_1} \right)^4 \cdot \left( \frac{\s_1^2}{\s_j^2} \right)^3 \tag{Since $\frac{\slambda + \s_1^2}{\slambda + \s_j^2} \leq \frac{\s_1^2}{\s_j^2}$, as $\slambda > 0$} \\
    & \leq \frac{\gamma^2}{\| \stheta \|^2} \cdot \sum_{j=1}^r \left( \frac{\s_1}{\s_j} \right)^2 \label{ineq-slambda-ridge-OG} \\
    &\leq \frac{\gamma^2}{\| \stheta \|^2} \cdot r \left( \frac{\s_1}{\s_r} \right)^2 \label{ineq-slambda-ridge}
\end{align}

Now since the variance term in (\ref{eq-ERVar-Ridge}) is a decreasing function of $\lambda$, so we can use the upper bound of $\slambda$ from (\ref{ineq-slambda-ridge}) to lower bound the $\ExcessRisk$ of optimal ridge as
\begin{align}
    \ExcessRisk\left( \htheta(\slambda) \right) &\geq \frac{\gamma^2}{n} \cdot \sum_{j=1}^r \frac{1}{\left( 1 + \frac{1}{s_j^2} \cdot \frac{\gamma^2}{ \| \stheta \|^2} \cdot r \left(\frac{s_1}{s_r}\right)^2 \right)^2} \tag*{} \\
    & \geq \frac{\gamma^2}{n} \cdot \sum_{j=1}^r \frac{1}{\left( 1 + \frac{s_r^2}{s_j^2} \cdot \frac{1}{99} \right)^2} \tag{Using Condition \#2} \\
    &\geq \frac{\gamma^2}{n} \cdot \sum_{j=1}^r \frac{1}{\left( 1 + \frac{1}{99} \right)^2} \tag{Since $s_r \leq s_j$, $\forall j \leq r$} \\
    &\geq \frac{\gamma^2}{n} \cdot r \cdot  \underbrace{\frac{99^2}{100^2}}_{\geq 0.98}
\end{align}
This completes the proof of eq (\ref{eq-thm-FD-sep-Ridge}).

    \textbf{$1$-step SD}: To evaluate $1$-step SD's $\ExcessRisk$, we make use of Theorem \ref{thm-quadraticRisk}. Observe that since $k=1$, the $\ExcessRisk$ is a simple quadratic in one variable. We will call $\xi^{(1)} \in \R$ as just $\xi$ (similar to eq (\ref{est-sd1-3})). And similarly, we will call $M^{(1)}, m^{(1)}$ as just $M, m$, both real-valued. We then have
    \begin{align}
        \ExcessRisk \left( \htheta( \lambda, \xi) \right) &= M \xi^2 + 2m \xi + c \tag*{} \\
        & \geq c - \frac{m^2}{M} \tag{By simple quadratic min} \\
        & = \frac{Mc - m^2}{M} \label{eq-proof-1SD-Mcm}
    \end{align}
    Note that $M, m, c$ are all functions of $\lambda$.
    Now we evaluate their expressions.
    $c$ is simply $\ExcessRisk$ of ridge, which we will fetch from Lemma \ref{lem-Ridge}.
    Evaluate $M$ from Theorem \ref{thm-quadraticRisk}:
    \begin{align*}
        M &= \frac{1}{n} \sum_{j=1}^r \frac{\left(  \frac{\lambda \stheta_j}{\rho_j} + \gamma^2 \right) }{(1 + \rho_j)^2} \cdot C_j(1) \cdot C_j(1) \\
        &= \frac{1}{n} \sum_{j=1}^r \frac{\left(  \frac{\lambda \stheta_j}{\rho_j} + \gamma^2 \right) }{(1 + \rho_j)^2} \cdot \frac{\rho_j^2}{(1 + \rho_j)^2} \tag{Using $C_j(1) = 1 - \frac{1}{1 + \rho_j}$} \\
        &= \frac{1}{n} \sum_{j=1}^r \frac{\left(  \lambda \stheta_j \rho_j + \gamma^2  \rho_j^2 \right) }{(1 + \rho_j)^4} \\
        &= \frac{1}{n} \left( \sum_{j=1}^r \frac{\left(  \lambda^2 \stheta_j \s_j^6 + \gamma^2  \lambda^2 \s_j^4 \right) }{(\lambda + \s_j^2)^4} \right) \tag{Using $\rho_j = \frac{\lambda}{\s_j^2}$}
    \end{align*}
    Evaluate $m$ from Theorem \ref{thm-quadraticRisk}:
    \begin{align*}
        m &= \frac{1}{n} \sum_{j=1}^r \frac{\left(\lambda \stheta_j - \gamma^2 \right) }{(1 + \rho_j)^2} \cdot C_j(1) \\
        &= \frac{1}{n} \sum_{j=1}^r \frac{\left(\lambda \stheta_j - \gamma^2 \right) }{(1 + \rho_j)^2} \cdot \frac{\rho_j}{1 + \rho_j} \tag{Using $C_j(1) = 1 - \frac{1}{1 + \rho_j}$} \\
        &= \frac{1}{n} \left( \sum_{j=1}^r \frac{\left(\lambda^2 \stheta_j \s_j^4  - \gamma^2 \lambda \s_j^4 \right) }{(\lambda + \s_j^2)^3} \right)  \tag{Using $\rho_j = \frac{\lambda}{\s_j^2}$} 
    \end{align*}

    We now write the expressions together for comparison, before we use them.
    \begin{align*}
        n \cdot M &= \lambda^2 \sum_{j=1}^r \frac{ \stheta_j \s_j^6 }{(\lambda + \s_j^2)^4} + \lambda^2 \gamma^2 \sum_{j=1}^r \frac{\s_j^4 }{(\lambda + \s_j^2)^4} \\
        n \cdot m &= \lambda^2 \sum_{j=1}^r \frac{ \stheta_j \s_j^4 }{(\lambda + \s_j^2)^3} - \lambda \gamma^2 \sum_{j=1}^r \frac{\s_j^4 }{(\lambda + \s_j^2)^3} \\
        n \cdot c &= \lambda^2 \sum_{j=1}^r \frac{ \stheta_j \s_j^2 }{(\lambda + \s_j^2)^2} +  \gamma^2 \sum_{j=1}^r \frac{\s_j^4 }{(\lambda + \s_j^2)^2} \tag{Directly from Lemma \ref{lem-Ridge}}
    \end{align*}
    With all these pieces, for the numerator in eq (\ref{eq-proof-1SD-Mcm}), we have
    \begin{align*}
        n^2 \cdot \left( Mc - m^2 \right) &= T_1 + T_2 + T_3 
    \end{align*}
    where we use $T_1, T_2, T_3$ to capture terms of different forms. $T_1$ will capture the $\stheta_j$ product terms, $T_2$ will capture the $\gamma^2$ product terms, and $T_3$ will capture the cross terms. Namely,
    \begin{align*}
        T_1 &= \lambda^4 \sum_{j=1}^r \sum_{l=1}^r \frac{ \stheta_j \stheta_l \s_j^2 \s_l^2}{ (\lambda + \s_j^2)^2 (\lambda + \s_l^2)^2  } \cdot \left( \frac{s_j^4}{(\lambda + \s_j^2)^2} + \frac{s_l^4}{(\lambda + \s_l^2)^2} -  \frac{ 2 s_j^2 \s_l^2}{(\lambda + \s_j^2) (\lambda + \s_l^2) } \right) \\
        &= \lambda^4 \sum_{j=1}^r \sum_{l=1}^r \frac{ \stheta_j \stheta_l \s_j^2 \s_l^2}{ (\lambda + \s_j^2)^2 (\lambda + \s_l^2)^2  } \cdot \left( \frac{s_j^2}{(\lambda + \s_j^2)} - \frac{s_l^2}{(\lambda + \s_l^2)} \right)^2 \\
        & \geq 0 \tag{\textdagger}\\
        T_2 &= \lambda^2 \gamma^4 \sum_{j=1}^r \sum_{l=1}^r \frac{  \s_j^4 \s_l^4}{ (\lambda + \s_j^2)^2 (\lambda + \s_l^2)^2  } \cdot \left( \frac{1}{(\lambda + \s_j^2)^2} + \frac{1}{(\lambda + \s_l^2)^2} - \frac{2}{(\lambda + \s_j^2) (\lambda + \s_l^2) } \right) \\
        &= \lambda^2 \gamma^4 \sum_{j=1}^r \sum_{l=1}^r \frac{  \s_j^4 \s_l^4}{ (\lambda + \s_j^2)^2 (\lambda + \s_l^2)^2  } \cdot \left( \frac{1}{(\lambda + \s_j^2)} - \frac{1}{(\lambda + \s_l^2)^2} \right)^2 \\
        & \geq 0 \tag{\textdagger \textdagger} \\
        T_3 &=  \lambda^2 \gamma^2 \left( \sum_{j=1}^r \frac{ \stheta_j \s_j^6 }{(\lambda + \s_j^2)^4} \right) \left( \sum_{j=1}^r \frac{\s_j^4 }{(\lambda + \s_j^2)^2} \right) \\
        & \hspace{10pt} + \lambda^4 \gamma^2 \left( \sum_{j=1}^r \frac{ \stheta_j \s_j^2 }{(\lambda + \s_j^2)^2} \right) \left( \sum_{j=1}^r \frac{\s_j^4 }{(\lambda + \s_j^2)^4} \right) \\
        & \hspace{10pt} + 2 \lambda^3 \gamma^2 \left( \sum_{j=1}^r \frac{ \stheta_j \s_j^4 }{(\lambda + \s_j^2)^3} \right) \left( \sum_{j=1}^r \frac{\s_j^4 }{(\lambda + \s_j^2)^3}  \right) \\
        & \geq \lambda^2 \gamma^2 \left( \sum_{j=1}^r \frac{ \stheta_j \s_j^6 }{(\lambda + \s_j^2)^4} \right) \left( \sum_{j=1}^r \frac{\s_j^4 }{(\lambda + \s_j^2)^2} \right) + \lambda^4 \gamma^2 \left( \sum_{j=1}^r \frac{ \stheta_j \s_j^2 }{(\lambda + \s_j^2)^2} \right) \left( \sum_{j=1}^r \frac{\s_j^4 }{(\lambda + \s_j^2)^4} \right) \tag{Ignoring the $3^{rd}$ term}
    \end{align*}
    For the overall lower bound on the numerator of eq (\ref{eq-proof-1SD-Mcm}), we combine the above lower bound on $T_3$ with eqs (\textdagger), (\textdagger \textdagger).
    Under Condition \#1, this simplifies to
    \begin{align}
        n^2 \cdot \left( Mc - m^2 \right) &\geq \lambda^2 \gamma^2 \cdot \frac{ \| \stheta \|^2 \s_1^6 }{(\lambda + \s_1^2)^4}  \left( \sum_{j=1}^r \frac{\s_j^4 }{(\lambda + \s_j^2)^2} \right) + \lambda^4 \gamma^2 \cdot  \frac{ \| \stheta \|^2 \s_1^2 }{(\lambda + \s_1^2)^2}  \left( \sum_{j=1}^r \frac{\s_j^4 }{(\lambda + \s_j^2)^4} \right) \tag*{} \\
        &= \gamma^2 \cdot \| \stheta \|^2 \s_1^2 \left\{  \frac{ \lambda^2 \s_1^4 }{(\lambda + \s_1^2)^4}  \left( \sum_{j=1}^r \frac{\s_j^4 }{(\lambda + \s_j^2)^2} \right) +   \frac{ \lambda^4 }{(\lambda + \s_1^2)^2}  \left( \sum_{j=1}^r  \frac{ \s_j^4 } {(\lambda + \s_j^2)^4} \right) \right\} \label{nume-LB}
    \end{align}
    
    And for the denominator in eq (\ref{eq-proof-1SD-Mcm}), we have
    \begin{align}
        n \cdot M &= \lambda^2 \sum_{j=1}^r \frac{ \stheta_j \s_j^6 }{(\lambda + \s_j^2)^4} + \lambda^2 \gamma^2 \sum_{j=1}^r \frac{\s_j^4 }{(\lambda + \s_j^2)^4} \tag*{} \\
        &= \lambda^2 \left( \sum_{j=1}^r \frac{ \stheta_j \s_j^6 }{(\lambda + \s_j^2)^4} + \gamma^2 \sum_{j=1}^r \frac{\s_j^4 }{(\lambda + \s_j^2)^4} \right)  \label{expr-nm}
    \intertext{Upper Bound $\propto \lambda^2$ from eq (\ref{expr-nm}):}
        &\leq \lambda^2 \left( \sum_{j=1}^r \frac{ \stheta_j \s_j^6 }{\s_j^8} + \gamma^2 \sum_{j=1}^r \frac{\s_j^4 }{\s_j^8} \right) \tag*{}\\
        &\leq \lambda^2 \left( \sum_{j=1}^r \frac{ \stheta_j \s_1^2 }{\s_j^4} + \gamma^2 \sum_{j=1}^r \frac{1}{\s_j^4} \right) \tag*{} \\
        &\leq \frac{ \lambda^2 }{s_r^4} \left( \| \stheta \|^2 \s_1^2 + r \gamma^2 \right) \label{deno-UB-1}
    \intertext{Upper Bound $\propto \nicefrac{1}{\lambda^2}$ from eq (\ref{expr-nm}):}
        & \leq \lambda^2 \left( \sum_{j=1}^r \frac{\stheta_j \s_j^6}{\lambda^4}  + \gamma^2 \sum_{j=1}^r \frac{\s_j^4}{\lambda^4} \right) \tag*{}\\ 
        & = \frac{1}{\lambda^2} \left( \sum_{j=1}^r \stheta_j \s_j^6 + 
        \gamma^2 \sum_{j=1}^r \s_j^4 \right) \tag*{} \\
        & \leq \frac{1}{\lambda^2} \left( \s_1^6 \sum_{j=1}^r \stheta_j  + \gamma^2 \s_1^4 \sum_{j=1}^r 1 \right) \tag*{} \\
        & \leq \frac{\s_1^4}{\lambda^2} \left(  \| \stheta \|^2 \s_1^2  + r \gamma^2  \right) \label{deno-UB-2}
    \end{align}

    Now we put together the numerator lower bound from eq (\ref{nume-LB}), with the denominator upper bound from eq (\ref{deno-UB-1}) for the first term, and from eq (\ref{deno-UB-2}) for the second term. We then get
    \begin{align}
        n \cdot \frac{Mc - m^2}{M} &\geq \gamma^2 \cdot \frac{  \| \stheta \|^2 \s_1^2 }{ \left( \| \stheta \|^2 \s_1^2  + r \gamma^2  \right)  } \left( \frac{ \s_1^4 \s_r^4 }{(\lambda + \s_1^2)^4}   \left( \sum_{j=1}^r \frac{\s_j^4 }{(\lambda + \s_j^2)^2} \right) +  \frac{ \nicefrac{\lambda^6}{\s_1^4} }{ (\lambda + \s_1^2)^2}  \left( \sum_{j=1}^r \frac{\s_j^4}{(\lambda + \s_j^2)^4} \right) \right) \tag*{}
        \intertext{Using $\s_j \geq \s_r$ and $\nicefrac{1}{(\lambda + \s_j^2)} \geq \nicefrac{1}{(\lambda + \s_1^2)}$ for all $j \in [r]$, we get}
        &\geq \gamma^2 \cdot \frac{  \| \stheta \|^2 \s_1^2 }{ \left( \| \stheta \|^2 \s_1^2  + r \gamma^2  \right)  } \left( \frac{ \s_1^4 \s_r^8 }{(\lambda + \s_1^2)^6}   \left( \sum_{j=1}^r 1 \right) +  \frac{ \lambda^6 \cdot ( \nicefrac{\s_r^4}{\s_1^4} ) }{ (\lambda + \s_1^2)^6}  \left( \sum_{j=1}^r 1 \right) \right) \tag*{} \\ 
        &= r \gamma^2 \cdot \underbrace{ \frac{  \| \stheta \|^2 \s_1^2 }{ \left( \| \stheta \|^2 \s_1^2  + r \gamma^2  \right) } }_{Q_1} \underbrace{ \left( \frac{ \s_1^4 \s_r^8 + ( \nicefrac{\s_r^4}{\s_1^4} )  \lambda^6 }{(\lambda + \s_1^2)^6} \right) }_{Q_2} \label{eq-Q1-Q2}
    \end{align}
    Under Condition \#2, it holds that
    \begin{align}
        Q_1 \geq \frac{ 99 \cdot \nicefrac{s_1^4}{s_r^4} }{99 \cdot \nicefrac{s_1^4}{s_r^4} + 1} \geq \frac{99}{100} = 0.99 \quad \quad \quad \text{(since $\s_1 \geq \s_r$)} \label{eq-Q1}
    \end{align}
    And now we analyze $Q_2$ using simple calculus. Note that
    \begin{align}
        Q_2 &=  \frac{ t_1 + t_2 \lambda^6 }{ (\lambda + \s_1^2)^6}  \quad \quad \quad \quad \text{ where } t_1 = \s_1^4 \s_r^8, \quad t_2 = \frac{\s_r^4}{\s_1^4} \tag*{}
        \intertext{Simple calculus shows that this function is minimized at $\bar{\lambda} = \left( \nicefrac{t_1}{(t_2 \s_1^2)} \right)^{0.2} = \left( \s_1^6 \s_r^4 \right)^{0.2} \leq \s_1^2$ (since $\s_r \leq \s_1$). And the min value of the function (evaluated at $\bar{\lambda}$) is }
        Q_2 & \geq \frac{t_1}{\s_1^2 \left( \bar{\lambda} + \s_1^2 \right)^5 } \geq \frac{t_1}{\s_1^2 \cdot 2^5 \s_1^{10}} =  \frac{1}{32} \cdot \left( \frac{\s_r}{\s_1}\right)^8 \label{eq-Q2}
    \end{align}
    Combining eqs (\ref{eq-proof-1SD-Mcm}), (\ref{eq-Q1-Q2}), (\ref{eq-Q1}), (\ref{eq-Q2}), we get
    \begin{equation*}
        \forall \lambda > 0, \forall \xi \in \R \hspace{10pt} \ExcessRisk \left( \htheta (\lambda, \xi) \right) \geq \frac{0.99}{2^5} \cdot \left( \frac{\s_r}{\s_1} \right)^8 \cdot \frac{r \gamma^2}{n}
    \end{equation*}
    Using Condition \#4 in the above gives the desired eq (\ref{eq-thm-FD-sep-oneSD}).
\end{proof}

\subsection{Discussion on Condition \#1 in Theorem \ref{thm-FD-separation-App}} \label{sec-app-peaky-discussion}

One does not strictly need $\stheta$ to be aligned with the leading eigenvector $\u_1$. 
If instead we had the below conditions (call them Conditions \#1' and \#2'), then also the Theorem would hold.
\begin{enumerate}
    \item $\stheta = \| \stheta \| \cdot \u_r$
    \item $\| \stheta \|^2 > 99 \cdot \frac{r \gamma^2}{\s_r^2} \cdot \left( \frac{\s_1^2}{\s_r^2} \right)^2 $ (i.e., the dependence has changed to $\left( \frac{\s_1^2}{\s_r^2} \right)^2$ from $\left( \frac{\s_1^2}{\s_r^2} \right)$)
\end{enumerate}
\begin{proof}
This is because the Upper Bound for the $r$-step SD estimator still holds with the identical calculation (with $\stheta_r, \s_r$ instead of $\stheta_1, \s_1$ in eq (\ref{eq-SD-UB-proof-PREV})). And the Lower Bound for the ridge estimator works similarly (with just the Variance term sufficing, i.e. eq (\ref{eq-ERVar-Ridge})), but with a slightly different upper bound on $\slambda$ than eq (\ref{ineq-slambda-ridge}). Namely, from eq (\ref{eq-slambda-exp-inproof}), we will get
\begin{align}
    \slambda &= \frac{\gamma^2}{\| \stheta \|^2} \cdot \sum_{j=1}^r \left( \frac{\s_j}{\s_r} \right)^4 \cdot \left( \frac{\slambda + \s_r^2}{\slambda + \s_j^2} \right)^3 \tag{$\stheta_j = 0$, $j \leq r-1$ and $\stheta_r = \| \stheta \|^2$ from Cond \#1'} \\
    & \leq \frac{\gamma^2}{\| \stheta \|^2} \cdot \sum_{j=1}^r \left( \frac{\s_j}{\s_r} \right)^4 \tag{Since $\frac{\slambda + \s_r^2}{\slambda + \s_j^2} \leq 1$, as $\slambda > 0$, $\s_r \leq \s_j$} \\
    &\leq \frac{\gamma^2}{\| \stheta \|^2} \cdot r \left( \frac{\s_1}{\s_r} \right)^4 \label{ineq-slambda-ridge-rpeak}
\end{align}
Using (\ref{ineq-slambda-ridge-rpeak}) instead of (\ref{ineq-slambda-ridge}) completes the argument. 
\end{proof}



\section{Proof of Theorem \ref{thm-FD-nonsep-assump}} \label{sec-app-proof-thm-FD-nonsep-assump}
We point the reader to Appendix~\ref{sec-app-notation} for notations used throughout the proof.
\begin{proof}
Under the condition of $\forall j \in [r], \s_j = 1$, we need to analyze the $\ExcessRisk$ for both $k$-step SD and ridge. \textbf{Denote $Q := \sum_{j=1}^r \stheta_j$ for simplicity}.

\textbf{$k$-step SD}: Since assumption \ref{assump-FD-sj}.\ref{assump-FD-sj1} is violated, we cannot use Theorem \ref{thm-sdkPre-optPre} for the $k$-step SD. Instead, we will work with the quadratic expansion of $\ExcessRisk$ from Theorem \ref{thm-quadraticRisk}. We will rewrite the expressions from Appendix \ref{sec-app-proof-thm-quadraticRisk} for the quadratic coefs. Note that $\rho_j = \lambda / \s_j^2 = \lambda$ for all $j \in [r]$ becomes independent of $j$ in this case. And so $C_j(i) = 1 - \frac{1}{(1 + \lambda)^i}$ for all $j \in [r], i \in [k]$ also becomes independent of $j$. Let $C(i) := 1 - \frac{1}{(1 + \lambda)^i}$ denote the coefs $C_j(i)$, since they're now independent of $j$.
Let $\omega := [ C(1), C(2), \cdots, C(k) ] \in \R^k$. Then, we get
\begin{align*}
    \forall (i_1, i_2) \in [k] \times [k] \quad \quad \Mk_{i_1, i_2} &= \frac{1}{n} \sum_{j=1}^r \frac{\stheta_j + \gamma^2}{(1 + \lambda)^2} \cdot C(i_1) \cdot C(i_2) \\
    &= \frac{1}{n} \cdot \frac{Q + r \gamma^2}{(1 + \lambda)^2} \cdot C(i_1) \cdot C(i_2) \\
    \implies \Mk &= \frac{1}{n} \cdot \frac{Q + r \gamma^2}{(1 + \lambda)^2} \cdot \omega \omega^\top \quad \quad \in \R^{k \times k} \text{ becomes a rank-1 matrix}
    \intertext{Similarly, for $\mk$ we get}
    \forall i_1 \in [k] \quad \quad \mk_{i_1} &= \frac{1}{n} \sum_{j=1}^r \frac{\lambda \stheta_j - \gamma^2}{(1 + \lambda)^2} \cdot C(i_1) \\
    \implies \mk &= \frac{1}{n} \cdot \frac{\lambda Q - r\gamma^2}{(1 + \lambda)^2} \cdot \omega \quad \quad  \in \R^k
    \intertext{And similarly, for $\ck$ we get}
    \ck &= \frac{1}{n} \cdot \frac{\lambda^2 Q + r\gamma^2}{(1 + \lambda)^2} \quad \quad  \in \R
\end{align*}
Using the above expressions, we rewrite the overall $\ExcessRisk$ for any $k$-step SD ($k \geq 1$) as
\begin{equation*}
    \ExcessRisk \left( \htheta (\lambda, \bxik) \right) = \frac{1}{n} \left( \frac{Q + r \gamma^2}{(1 + \lambda)^2} \cdot \langle \omega, \xik \rangle^2 + 2 \cdot \frac{\lambda Q - r\gamma^2}{(1 + \lambda)^2} \cdot \langle \omega, \xik \rangle + \frac{\lambda^2 Q + r\gamma^2}{(1 + \lambda)^2} \right)
\end{equation*}
We're aiming for a lower bound on the above quantity, so we can first minimize with respect to $\bxik$, and then analyze the remaining as a function of $\lambda$. Note that this is a quadratic in the scalar $\langle \omega, \xik \rangle$. Since $q(x) := ax^2 + 2bx + c = c - \frac{b^2}{a} + a \left( x + \frac{b}{a} \right)^2 \geq c - \frac{b^2}{a}$, we have 
\begin{align*}
    \forall k \geq 1, \forall \bxik \in \R^k \quad \quad \ExcessRisk \left( \htheta (\lambda, \bxik) \right) &\geq \frac{1}{n} \cdot \frac{1}{(1 + \lambda)^2} \left(  \lambda^2 Q + r\gamma^2 - \frac{\left( \lambda Q - r \gamma^2 \right)^2}{Q + r\gamma^2} \right) \\
    &= \frac{1}{n} \cdot \frac{1}{(1 + \lambda)^2} \left(  \frac{ \lambda^2 Qr \gamma^2 + Qr \gamma^2 + 2 \lambda Qr \gamma^2 }{Q + r\gamma^2} \right) \\
    &= \frac{1}{n} \cdot \frac{Q r \gamma^2}{ Q + r \gamma^2} \\ 
    &= \frac{r \gamma^2}{n} \cdot \frac{1}{ 1 + \frac{r \gamma^2}{Q}} 
\end{align*}
Note this expression is independent of $\lambda$. Hence the above lower bound holds $\forall k \geq 1, \forall \bxik \in \R^k, \forall \lambda > 0$. This concludes the proof of eq (\ref{eq-thm-FD-nonsep-assump-1}).

\textbf{Ridge}: For ridge, we can simply borrow the expression of $\ck$ from above for its $\ExcessRisk$. That is
\begin{equation*}
    \ExcessRisk\left( \htheta (\lambda) \right) = \frac{1}{n} \cdot \frac{\lambda^2 Q + r\gamma^2}{(1 + \lambda)^2}
\end{equation*}
Let $\slambda$ denote its minimizer over $\lambda > 0$. Simple calculus gives $\slambda = \frac{r \gamma^2}{Q}$. Plugging this in, we get the same expression
\begin{equation*}
    \ExcessRisk\left( \htheta (\slambda) \right) = \frac{r \gamma^2}{n} \cdot \frac{1}{ 1 + \frac{r \gamma^2}{Q}} 
\end{equation*}
This completes the proof of ridge achieving the lower bound in eq (\ref{eq-thm-FD-nonsep-assump-1}).
\end{proof}

\section{Proof of Theorem \ref{thm-FD-nonsep-nonpeaky}} \label{sec-app-proof-thm-FD-nonsep-nonpeaky}
We point the reader to Appendix~\ref{sec-app-notation} for notations used throughout the proof.

\begin{proof}
Under the condition of $\forall j \in [r], \stheta_j = z > 0$, we need to analyze the $\ExcessRisk$ for both $k$-step SD and ridge.

\textbf{$k$-step SD}: Again from Theorem \ref{thm-sdkPre-optPre}, any $k$-step SD estimator's $\ExcessRisk$ is
\begin{align*}
    \forall k \geq 1, \forall \lambda > 0, \forall \bxik \in \R^k, \quad \ExcessRisk \left( \htheta(\lambda, \bxik) \right) &\geq \frac{\gamma^2}{n} \sum_{j=1}^r \frac{ \stheta_j }{\left(\stheta_j  + \frac{\gamma^2}{\s_j^2} \right)} \\
    &\geq \frac{\gamma^2}{n} \sum_{j=1}^r \frac{ 1 }{\left(1  + \frac{ \gamma^2}{z} \cdot \frac{1}{\s_j^2} \right)} \tag{Since $\forall j \in [r], \stheta_j = z$}
\end{align*}
This completes the proof of eq (\ref{eq-thm-FD-nonsep-nonpeaky-1}).

\textbf{Ridge}: Now for the ridge estimator, we will use Lemma \ref{lem-Ridge}.
From eq (\ref{eq-slambda-exp}), we get an exact expression for $\slambda > 0$ that minimizes the $\ExcessRisk$. Namely
\begin{align*}
    \slambda = \frac{\gamma^2}{z}
\end{align*}
Substituting this in eq (\ref{eq-ER-Ridge}), we get
\begin{align*}
    \ExcessRisk \left( \htheta(\slambda) \right) &= \frac{1}{n} \sum_{j=1}^r \frac{\left( (\slambda)^2 \stheta_j + \gamma^2 \s_j^2 \right) \cdot \s_j^2}{(\slambda + \s_j^2)^2} \\
    &= \frac{\gamma^2}{n} \sum_{j=1}^r \frac{\left( \slambda + \s_j^2 \right) \cdot \s_j^2}{(\slambda + \s_j^2)^2} \tag{Since $\slambda \stheta_j = \frac{\gamma^2}{z} z = \gamma^2$} \\
    &= \frac{\gamma^2}{n} \sum_{j=1}^r \frac{1}{(1 + \frac{\slambda}{\s_j^2} )} \\
    &= \frac{\gamma^2}{n} \sum_{j=1}^r \frac{1}{(1 + \frac{\gamma^2}{z} \cdot \frac{1}{\s_j^2} )} \tag{Substituting $\slambda = \frac{\gamma^2}{z}$}
\end{align*}
This completes the proof of ridge achieving the lower bound in eq (\ref{eq-thm-FD-nonsep-nonpeaky-1}).
\end{proof}


\section{Proof of Theorem \ref{thm-sdkPre-optPre}} \label{sec-app-proof-thm-sdkPre-optPre}
We point the reader to Appendix~\ref{sec-app-notation} for notations used throughout the proof.

\begin{proof}
The proof of (\ref{eq-thm-sdk-opt-ER-LowerBound}) is a simple instantiation of Lemma \ref{lem-optimalPre}. Since the SD estimator is a particular instance of the general $\htheta(\Pre)$, i.e.
\[
    \htheta \left( \lambda, \bxik \right) = \htheta \left( \Pre \leftarrow \Pre \left(\lambda, \bxik \right) \cdot \Om_{\lambda}^{-1} \right)
\]
Eq (\ref{eq-thm-sdk-opt-ER-LowerBound}) follows from the lower bound in Lemma \ref{lem-optimalPre}.

For proving the equality being achieved, we will work with the general $k$-step SD estimator, and show that: Under assumption \ref{assump-FD-sj}.\ref{assump-FD-sj1}, $k=r$ steps are sufficient to provide enough freedom to $\bxik$ so that the $k$-step SD estimator achieves the lowest possible $\ExcessRisk$.

Using lemma \ref{lem-optimalPre}, note that the condition $\Pre \left(\lambda, \bxik \right) \cdot \Om_{\lambda}^{-1} = \Pre^\star$ is sufficient to ensure that $\htheta \left( \lambda, \bxik \right) = \htheta(\Pre^\star)$, which would mean that the $k$-step SD estimator admits the lowest possible $\ExcessRisk$.
Since the eigenspaces for both sides of the equation are the same ($\Ud$), we only need to ensure that the eigenvalues match on both sides.
That is, we need the following condition (for indices $j \geq r+1$, there's no condition since lemma \ref{lem-optimalPre} tells us that any real value of $\ts_j$ suffices).
\begin{align} \label{eq-thm-sdk-opt-condition-APP}
    \forall j \in [r] \hspace{25pt} \left( 1 - \sum_{i=1}^k \xik_i \left\{ 1 - \left( \frac{\s_j^2}{\lambda + \s_j^2} \right)^i \right\} \right) \cdot \frac{1}{\lambda + \s_j^2} &= \ts_j^\star \\
    \iff \forall j \in [r] \hspace{25pt} \left( 1 - \sum_{i=1}^k \xik_i + \sum_{i=1}^k \xik_i \left( \frac{\s_j^2}{\lambda + \s_j^2} \right)^i  \right) \cdot \frac{\s_j^2}{\lambda + \s_j^2} &= \frac{\stheta_j}{\stheta_j + \frac{\gamma^2}{\s_j^2}}
\end{align}

Let $a_j(\lambda) := \frac{\s_j^2}{\lambda + \s_j^2}$. Since $\lambda > 0$, $a_j \in (0, 1), \forall j \in [r]$. The above condition can then be written succinctly in matrix form as
\begin{equation} \label{eq-thm-xik-consistency}
    \underbrace{\Ak}_{\in \R^{r \times k}} \cdot \underbrace{\xik}_{\in \R^k} = \underbrace{\alpha(\lambda)}_{\in \R^r}
\end{equation}
With the following describing the elements of $\Ak, \alpha(\lambda)$ as
\begin{align*}
    \left[ \alpha(\lambda) \right]_j &:= 1 - \left( \lambda + \s_j^2 \right) \ts_j^\star \quad \quad \quad j \in [r] \\
    \left[ \Ak \right]_{j, i} &:= 1 - \left( a_j(\lambda) \right)^i \quad \quad \quad \quad j \in [r], i \in [k]
\end{align*}
The notation $\Ak, \alpha(\lambda)$ explicitly denotes the dependence on $\lambda$.

Now, $\Ak$ being invertible would ensure that $\exists \xik \in \R^k$ that makes equation (\ref{eq-thm-xik-consistency}) hold true (i.e. the system of equations admits a solution). For that, we need (1) $k=r$ and (2) $\Ar$ being full-rank. The first condition is stated in the lemma.
In what follows, we will prove that $\exists \lambda > 0$ such that the second condition is satisfied, i.e. $\Ar$ being full-rank. Decompose $\Ar$ as
\[
    \Ar =   \underbrace{\begin{bmatrix}
                1 & 1 & \cdots & 1 \\
                1 & 1 & \cdots & 1 \\
                \vdots & \vdots & \ddots & \vdots \\
                1 & 1 & \cdots & 1
                \end{bmatrix}_{r \times r}}_{W}
                -
                \underbrace{\begin{bmatrix}
                a_1(\lambda) & a_1(\lambda)^2 & \cdots & a_1(\lambda)^r \\
                a_2(\lambda) & a_2(\lambda)^2 & \cdots & a_2(\lambda)^r \\
                \vdots & \vdots & \ddots & \vdots \\
                a_r(\lambda) & a_r(\lambda)^2 & \cdots & a_r(\lambda)^r
                \end{bmatrix}_{r \times r}}_{V(\lambda)}
\]
Where $W = \bone \bone^\top$ is the matrix of all ones, and $V(\lambda)$ is akin to the square Vandermonde matrix (only difference being that the standard definition of Vandermonde also has a column of ones).

Using the Matrix Determinant Lemma, we can write
\[
    \det ( \Ar ) = \left( 1 - \bone^\top V(\lambda)^{-1} \bone \right) \cdot \det ( -V(\lambda) )
\]
First note that, with the determinant expansion rule based on the first row
\[ 
    \det V(\lambda) = \det \begin{bmatrix}
                1 & \mathbf{0}^\top  \\
                \bone & V(\lambda)
                \end{bmatrix}_{(r+1) \times (r+1)}
\]
The matrix on the right is exactly the Vandermonde matrix with $a_0 = 0, a_1(\lambda), a_2(\lambda), \cdots a_r(\lambda)$. 
Using the formula for the det of a standard (with a row of ones) Vandermonde matrix, we get
\[
    \det V(\lambda) = \prod_{0 \leq i < j \leq r} \left( a_j (\lambda) - a_i(\lambda) \right) = \prod_{1 \leq i \leq r} a_i (\lambda) \cdot \prod_{1 \leq i < j \leq r} (a_j (\lambda)- a_i(\lambda))
\]
Since $a_j(\lambda) \in (0, 1)$ for all $j \in [r]$, and $a_j(\lambda)$'s are all distinct, we conclude that $\det V(\lambda) \neq 0$, i.e. $V(\lambda)$ is full-rank.
What remains to show is that the scalar $\left( 1 - \bone^\top V(\lambda)^{-1} \bone \right)$ is non-zero \emph{for some} $\lambda > 0$.
With the SVD of $V(\lambda) = C D E^{-1}$ where $C, E$ are orthonormal and $D$ is diagonal with positive entries, one can expand this term as:
\[
    \bone^\top V(\lambda)^{-1} \bone = (E^{-1} \bone)^\top D^{-1} (C^{-1} \bone) 
    \implies | \bone^\top V^{-1} \bone | \geq \frac{1}{d_{max}} \cdot | \langle C^{-1} \bone, E^{-1} \bone \rangle |
\]
Now observe that as $\lambda \rightarrow \infty$, $a_j(\lambda) \rightarrow 0$ for all $j \in [r]$, which means that $V(\lambda) \rightarrow \mathbf{0}_{r \times r}$, which means that (1) $d_{max} \rightarrow 0^{+}$, and (2) $C \leftrightarrow E$ (since $V(\lambda)$ becomes closer to a symmetric matrix, allowing its left and right singular matrices to approach equality to each other).
That is, $\frac{1}{d_{max}} \rightarrow \infty$ and $\langle C^{-1} \bone, E^{-1} \bone \rangle \rightarrow \| \bone \|^2 = r$.
This would ensure that $| \bone^\top V(\lambda)^{-1} \bone | > 1$, meaning that the scalar $(1 - \bone^\top V(\lambda)^{-1} \bone)$ would be non-zero. 
Hence $\exists \lambda > 0$, such that $\Ar$ is full-rank.
\end{proof}


\subsection{Proof of Lemma \ref{lem-optimalPre}} \label{sec-app-proof-lem-optimalPre}

\begin{proof}
The estimator $\htheta(\Pre)$ expands as
\begin{equation*}
    \htheta(\Pre) = \Ud \tS \Sd^2 \Ud^\top \stheta + \Ud \tS \Sd \Vd^\top \Noise \; \;.
\end{equation*}
Expanding the Excess Risk shows
\begin{align*}
    \ExcessRisk \left( \htheta(\Pre) \right) &= \underbrace{\sum_{j=1}^d \left( \left( \ts_j \cdot \s_j^2 - 1 \right)^2 \cdot \stheta_j \cdot w_{j} \right) }_{Bias} \text{ } + \text{ } \underbrace{\gamma^2 \sum_{j=1}^d \left( \ts_j^2 \cdot \s_j^2  \cdot   w_{j}  \right) }_{Variance} \; \;,
\end{align*}
where $w_{j} = \frac{\s_j^2}{n}$ for all $j \in [d]$. This is a simple quadratic expression in $\ts$. Completing the squares gives the desired optimal values of $\ts_j^\star, j \in [d]$. 
\end{proof}
\section{Quadratic $\ExcessRisk$: detailed version of Theorem \ref{thm-quadraticRisk} and a proof} \label{sec-app-proof-thm-quadraticRisk}
We point the reader to Appendix~\ref{sec-app-notation} for notations used throughout the proof.

\begin{theorem}[Formal version of Theorem~\ref{thm-quadraticRisk}] \label{thm-quadraticRisk-formal}
    The Excess Risk is quadratic in $\xik \in \R^k$. Namely
    \begin{equation} \label{quadratic-ER-App}
        \ExcessRisk \left( \htheta(\lambda, \bxik)  \right) = \left(\xik\right)^\top \underbrace{\Mk}_{\in \R^{k \times k}} \left(\xik \right) + 2  \left( \xik \right)^\top \underbrace{\mk}_{\in \R^k} + \text{ } \ck
    \end{equation}
    where the below holds, for indices $(i_1, i_2)$ in $[k] \times [k]$,
    \begin{align*}
        \Mk_{i_1, i_2} &= \underbrace{\frac{\lambda}{n} \sum_{j=1}^r \frac{\stheta_j}{\rho_j (1 + \rho_j)^2} \cdot C_j(i_1) \cdot C_j(i_2)}_{\Bk_{i_1, i_2}} + \underbrace{\frac{\gamma^2}{n} \sum_{j=1}^r \frac{1}{(1 + \rho_j)^2} \cdot C_j(i_1) \cdot C_j(i_2)}_{\Vk_{i_1, i_2}} \\
        &= \frac{1}{n} \sum_{j=1}^r \frac{\left(  \frac{\lambda \stheta_j}{\rho_j} + \gamma^2 \right) }{(1 + \rho_j)^2} \cdot C_j(i_1) \cdot C_j(i_2) \; \;, \\
        \mk_{i_1} &= \underbrace{\frac{\lambda}{n} \sum_{j=1}^r \frac{\stheta_j}{(1 + \rho_j)^2} \cdot C_j(i_1)}_{\bk_{i_1}} + \underbrace{\frac{\gamma^2}{n} \sum_{j=1}^r \frac{\left(-1 \right)}{(1 + \rho_j)^2} \cdot C_j(i_1)}_{\vk_{i_1}} \\
        &= \frac{1}{n} \sum_{j=1}^r \frac{\left(\lambda \stheta_j - \gamma^2 \right) }{(1 + \rho_j)^2} \cdot C_j(i_1) \; \;, \\
        \ck &=  \underbrace{ \frac{\lambda}{n}  \sum_{j=1}^r \frac{\stheta_j \rho_j}{(1 + \rho_j)^2} }_{\text{From Bias}} + \underbrace{ \frac{\gamma^2}{n}  \sum_{j=1}^r \frac{1}{(1 + \rho_j)^2} }_{\text{From Variance}} \\
        &= \frac{1}{n}  \sum_{j=1}^r \frac{ \left( \lambda \stheta_j \rho_j + \gamma^2 \right) }{(1 + \rho_j)^2} \; \;.
    \end{align*}
    The $B, b$ and $V, v$ notation is used to indicate the (squared) Bias and Variance terms respectively.
    Here $\rho_j, j \in [r]$ is used to simplify the notation, and is defined as $\rho_j :=  \frac{\lambda}{\s_j^2}, j \in [r]$. \\
    And the coefs $C_j$, $j \in [r]$ with $i \in [k]$ have the form 
    \begin{equation} \label{quadratic-eq-Cj}
        C_j(i) :=  1 - \frac{1}{(1 + \rho_j)^i} \hspace{25pt} j \in [r], i \in [k] \; \;.
    \end{equation}
\end{theorem}

\begin{proof}
We first use the expansion from Lemma \ref{lem-est-sdk-expansion}.
Secondly, expand $\stheta = \sum_{j=1}^d \langle \stheta, \u_j \rangle \u_j$.
Using these, we can expand Excess Risk as
\begin{align*}
    \ExcessRisk\left( \htheta(\lambda, \bxik) \right) &= \E_{\Noise} \left[ \left\| \htheta(\lambda, \bxik) - \stheta \right\|_{\hSigma_n}^2 \right] \\
    &= \sum_{j=1}^d \underbrace{\frac{\s_j^2}{n}}_{\hSigma_n \text{ weighing}} \cdot \stheta_j \cdot \left( \frac{\s_j^2}{\lambda + \s_j^2} \right)^2 \cdot \left( \frac{\lambda}{\s_j^2} + \sum_{i=1}^k \xik_i \cdot \left\{ 1 - \left( \frac{\s_j^2}{\lambda + \s_j^2} \right)^i \right\}  \right)^2 \\
    & \hspace{10pt} + \sum_{j=1}^d \underbrace{\frac{\s_j^2}{n}}_{\hSigma_n \text{ weighing}} \cdot \gamma^2 \cdot \frac{\s_j^2}{(\lambda + \s_j^2)^2} \cdot \left( -1 + \sum_{i=1}^k \xik_i \cdot \left\{ 1 - \left( \frac{\s_j^2}{\lambda + \s_j^2} \right)^i \right\}  \right)^2
\end{align*}
Writing down the above expansion and collecting the corresponding quadratic, linear and constant terms' coefficients in the $\xik$, give the desired expressions. 
\end{proof}

\subsection{$\ExcessRisk$ for ridge} \label{sec-app-lem-Ridge}

Since we will compare the ridge estimator's $\ExcessRisk$ to the $k$-step SD estimator, we state the $\ExcessRisk$ expression for the ridge estimator (eq (\ref{est-ridge-2})) formally here.
\begin{lemma} \label{lem-Ridge}
    The ridge estimator $\htheta(\lambda)$ satisfies
    \begin{equation}
    \forall \lambda > 0, \quad \ExcessRisk\left( \htheta(\lambda) \right) = \frac{1}{n} \sum_{j=1}^r \frac{\left( \lambda^2 \stheta_j + \gamma^2 \s_j^2 \right) \cdot \s_j^2}{(\lambda + \s_j^2)^2} \; \;. \label{eq-ER-Ridge}
    \end{equation}
    And the optimal penalty $\slambda > 0$ that minimizes this $\ExcessRisk$ satisfies
    \begin{equation}
        \slambda = \gamma^2 \cdot \frac{\sum_{j=1}^r \frac{\s_j^4}{\left(\slambda + \s_j^2\right)^3}}{\sum_{j=1}^r \frac{\stheta_j \s_j^4}{\left(\slambda + \s_j^2\right)^3}} \; \;. \label{eq-slambda-exp} 
    \end{equation}
\end{lemma}

\begin{proof}
    Eq (\ref{eq-ER-Ridge}) is a simple instantiation of Theorem \ref{thm-quadraticRisk} in the vacuous case of $k=0$. Specifically, the quantity $\ck$ captures exactly what we need. Borrowing its expression from the detailed theorem statement (Appendix \ref{sec-app-proof-thm-quadraticRisk}) gives us eq (\ref{eq-ER-Ridge}).
    Now, taking the derivative of the expression in eq (\ref{eq-ER-Ridge}) and setting it to zero, we get the stated expression for $\slambda > 0$. 
\end{proof}


\section{Discussion on synthetic experiments} \label{sec-app-synth-details}

This section follows up Section~\ref{sec-expts-synth} with more details and examples about the synthetic problem. Figure~\ref{fig-synth-appendix} shows four more settings, with $\gamma \in \{0.125, 0.25 \}$ and $\stheta \in \{ \u_1, \nicefrac{1}{\sqrt{2}}(\u_1 + \u_2)\}$.

Figure~\ref{fig-synth_1_1} validates that repeated steps of SD do provide a reduction in the excess risk, since the \emph{lowest point} of the curve for each $k$ is reducing as $k$ increases.
Observe that at $k=r=4$ steps of SD, the curves in Figure~\ref{fig-synth-appendix} become flat. This is because we stated Theorem~\ref{thm-sdkPre-optPre} with a "$\exists \lambda > 0$" such that $r$-step SD can achieve the lower bound, but in practice it can happen "$\forall \lambda > 0$".

\begin{figure}[h]
\begin{subfigure}{.5\textwidth}
  \includegraphics[width=\textwidth]{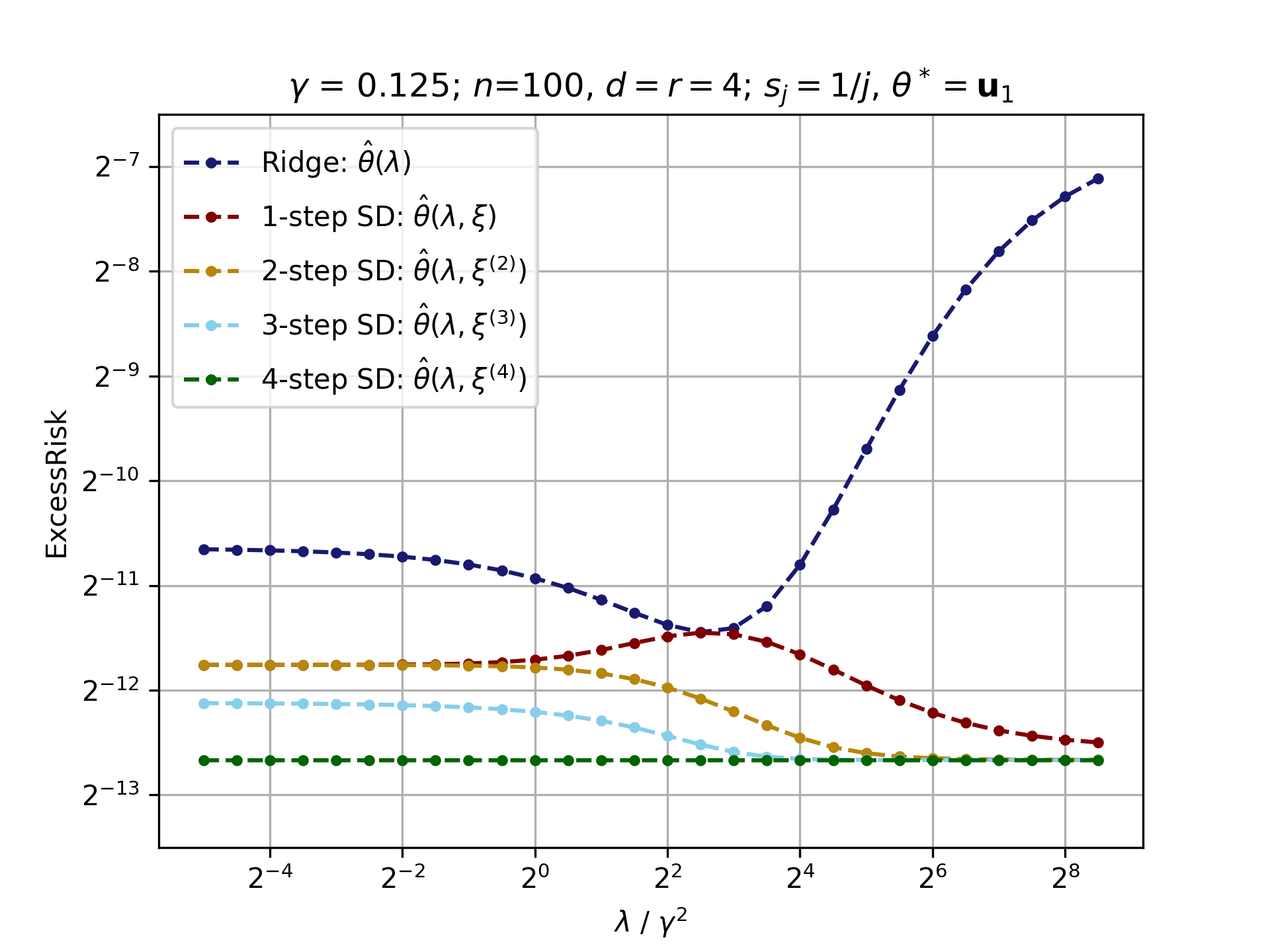}
  \caption{$\gamma=0.125, \stheta=\u_1$}
  \label{fig-synth_1_1}
\end{subfigure}
\begin{subfigure}{.5\textwidth}
  \includegraphics[width=\linewidth]{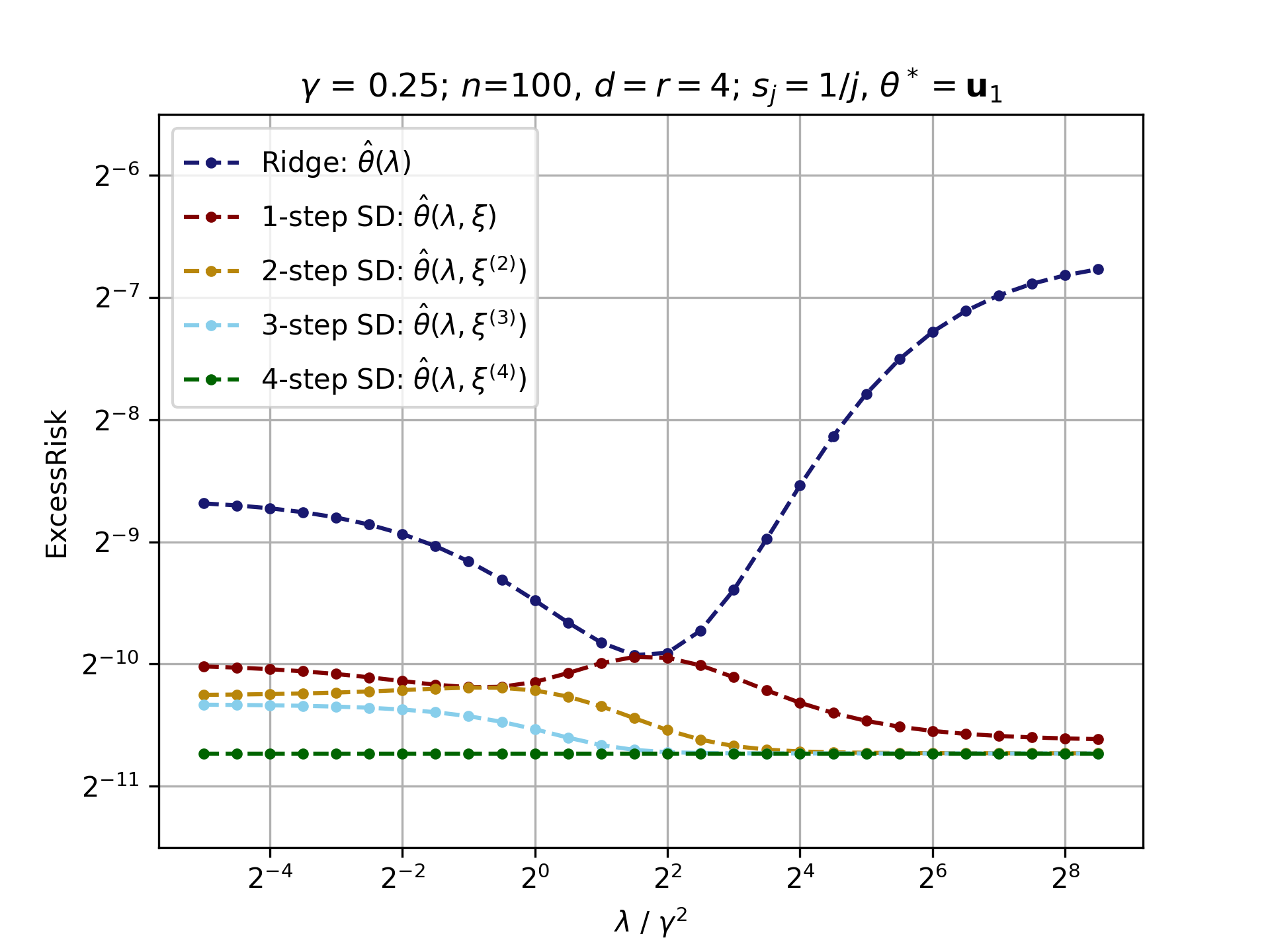}
  \caption{$\gamma=0.25, \stheta=\u_1$}
  \label{fig-synth_1_2}
\end{subfigure}
\medskip
\begin{subfigure}{.5\textwidth}
  \includegraphics[width=\linewidth]{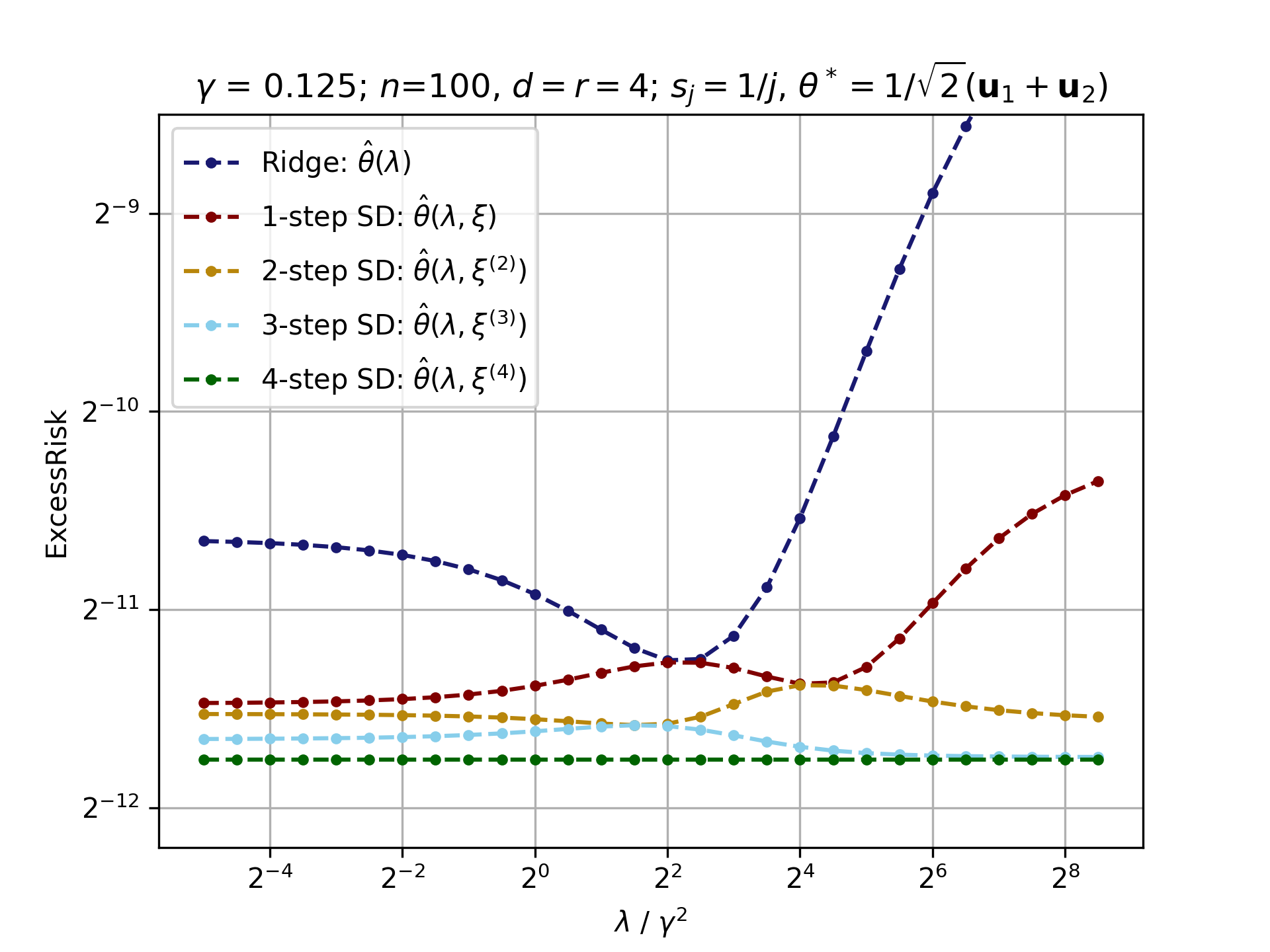}
  \caption{$\gamma=0.125, \stheta=\nicefrac{1}{\sqrt{2}}(\u_1 + \u_2)$}
  \label{fig-synth_2_1}
\end{subfigure}
\begin{subfigure}{.5\textwidth}
  \includegraphics[width=\linewidth]{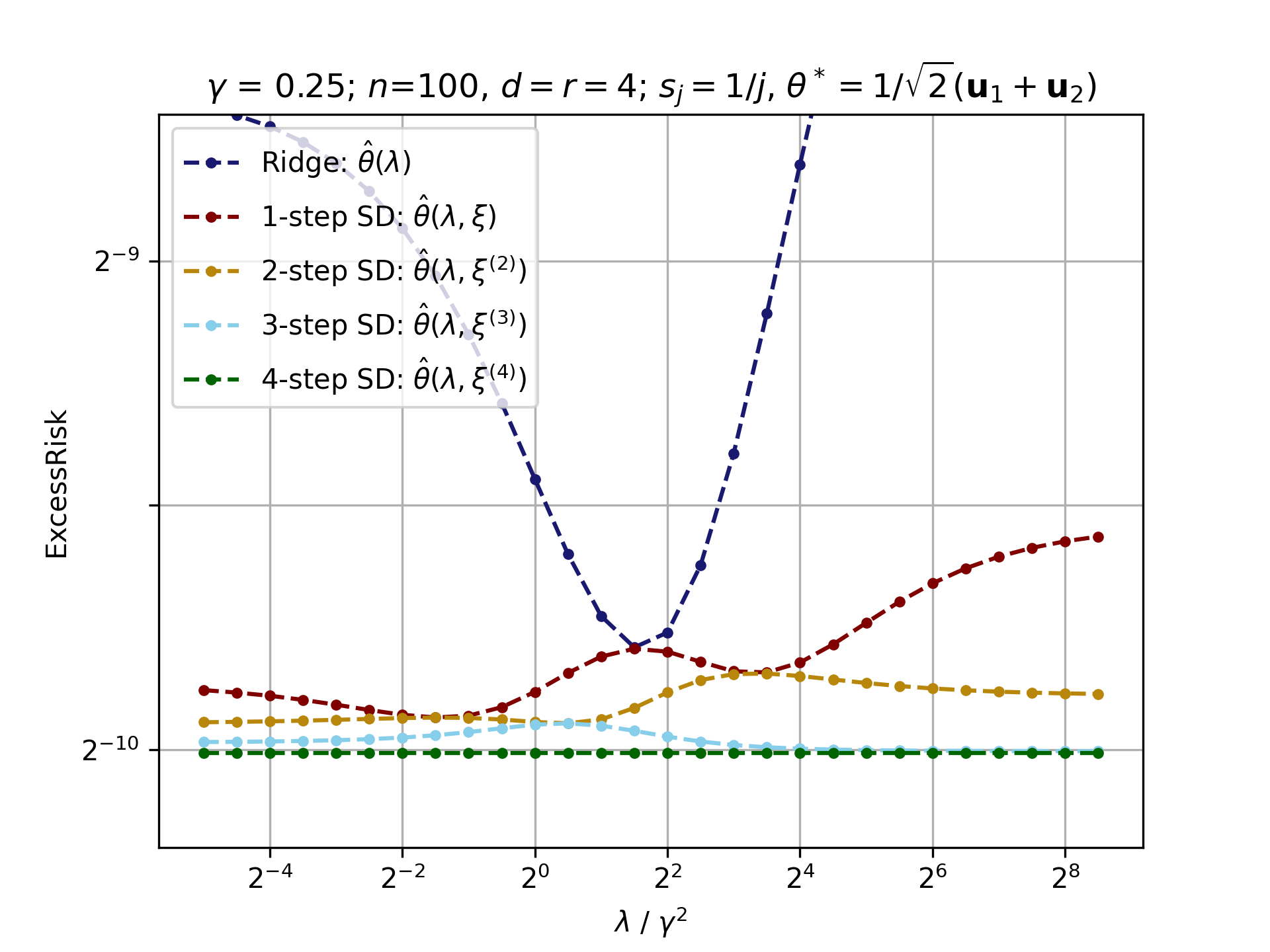}
  \caption{$\gamma=0.25, \stheta=\nicefrac{1}{\sqrt{2}}(\u_1 + \u_2)$}
  \label{fig-synth_2_2}
\end{subfigure}
\caption{Plot of excess risk of $k$-step SD with optimal $\bxik$ for each $\lambda$, for $k \in \{0, 1, 2, 3, 4 \}$ on a synthetic problem (Section~\ref{sec-expts-synth}).}
\label{fig-synth-appendix}
\end{figure}

\section{Discussion on choosing $\xi$ parameters} \label{sec-app-choosing-xi}

For the $k$-step SD estimator (eq (\ref{est-sdk})), for any chosen $\lambda$, Theorem \ref{thm-quadraticRisk} tells us that the $\ExcessRisk$ is quadratic in $\xik \in \R^k$.
To estimate the coefficients of this quadratic from certain chosen evaluations of $\bxik$, we need a total of $\nicefrac{k(k+3)}{2} + 1$ evaluations. This is because the number of unknown coefficients are (i) $k$ for square terms, (ii) $\nicefrac{k(k-1)}{2}$ for cross-square terms, (iii) $k$ for linear terms, and (iv) $1$ for the constant term. We now illustrate this in detail for $k=2$.

\subsection{Illustration for $k=2$} \label{sec-app-choosing-xi-ill2}

\begin{figure}[h]
\centering
\begin{tikzpicture}[->, >=stealth', auto, semithick, node distance=2cm]
\tikzstyle{state}=[draw, thick, fill=buff, minimum size=8mm]

\node[state] (S0) {$S_0$};
\node[state, fill=burlywood] (S1) [right of=S0] {$S_1$};
\node[state, fill=bronze] (S2) [right of=S1] {$S_2$};
\path (S0) edge node {$\bxi_1$} (S1)
      (S1) edge node {$\bxi_2$} (S2);
\end{tikzpicture}
\caption{Illustrating $2$-step SD.}
\label{fig-2step-choosing-xi}
\end{figure}
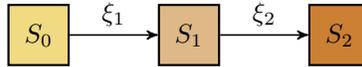

As explained above, the $\ExcessRisk$ for a chosen $\lambda$ has the following form
\begin{align*}
    \ExcessRisk \bigl( \htheta(\lambda, \underbrace{\xi}_{\in \R^2}) \bigr) &= A \barxi_1^2 + B \barxi_2^2 + 2C \barxi_1 \barxi_2 + 2D \barxi_1 + 2E \barxi_2 + F 
\end{align*}
where the reparametrization is $\barxi_1 \leftarrow \bxi_2 (1 - \bxi_1)$, and $\barxi_2 \leftarrow \bxi_2 \bxi_1$ (refer to Appendix \ref{sec-app-xis-reparam} for details on this reparametrization).
Let $\textsc{Eval}(\bxi_2, \bxi_1)$ denote the result of measuring this estimator's Risk on the validation dataset. We get the below system of equations:

\hspace{5pt} 

\begin{align*}
    F &= \textsc{Eval} (\bxi_2 = 0, \bxi_1 = 0) \\ 
    A + 2D + F &= \textsc{Eval} (\bxi_2 = 1, \bxi_1 = 0) \\
    A - 2D + F &= \textsc{Eval} (\bxi_2 = -1, \bxi_1 = 0) \\
    B + 2E + F &= \textsc{Eval} (\bxi_2 = 1, \bxi_1 = 1) \\
    B - 2E + F &= \textsc{Eval} (\bxi_2 = -1, \bxi_1 = 1) \\
    \frac{A + B}{4} + \frac{C}{2} + D + E + F &= \textsc{Eval} (\bxi_2 = 1, \bxi_1 = 0.5)
\end{align*}
The above $6$ $\textsc{Eval}$ operations help us identify all the $6$ unknown coefficients. Given $AB - C^2 > 0$ (\cite{jankovicQUAD}), we will have the below $\barxi_1^\star, \barxi_2^\star$ minimizing the test loss (i.e. giving the optimal $2$-step SD coefs for the chosen $\lambda$). And using them, we calculate the actual $\bxi_1^\star, \bxi_2^\star$ by doing the inverse mapping of the reparametrization.
\begin{align*}
    \barxi_1^\star = \frac{CE - DB}{AB - C^2}, \barxi_2^\star = \frac{CD - AE}{AB - C^2} \\ 
    \bxi_1^\star = \frac{\barxi_2^\star}{(\barxi_1^\star + \barxi_2^\star)}, \bxi_2^\star = \barxi_1^\star + \barxi_2^\star
\end{align*}

\section{Details on regression experiments} \label{sec-app-regr}

We note that all experiments run on a single CPU within 60 seconds (wall-clock time). We utilize sklearn's implementation of the $\textsc{Ridge}$.

{\bf Methodology}. We now explain our methodology in detail, which is followed for all datasets:
\begin{enumerate}
    \item First, split the original dataset into three parts for a Train-Validation-Test split. We divide all datasets in a $30-30-40$ split. Note that 30\% of the data suffices for training since we work with small datasets which have $d$ on the order of ten (and total number of samples $n$ on the order of thousands). For datasets that have a temporal  notion (e.g. date/timestamps), we do the three-way split sequentially.
    \item Perform two standard transformations on all three splits of the data: (i) Remove records with missing values, and (ii) coordinate-wise whitening for all $\X$ features and the $\Y$ target, i.e., subtracting the mean (computed on the train set) and dividing by the standard deviation (also computed on the train set).
    \item Select a grid of $\lambda$ values (and ensure that it is large enough so that the optimal $\lambda$ lies in it). The grid has a factor of $\sqrt{10}$ difference between consecutive values (e.g., $\{1, \sqrt{10}, 10, \cdots, 10^4\}$). 
    Then, for all $\lambda$ in the grid, compute:
    \begin{itemize}
        \item The ridge estimator $\htheta(\lambda)$ using the train set.
        \item The $1$-step SD estimator $\htheta(\lambda, \xi^\star)$ using the train set, where $\bxi^\star$ for each $\lambda$ is found using the validation set by the strategy described in Section \ref{sec-expts-hparams}.
        \item The $2$-step SD estimator $\htheta(\lambda, (\bxi^{(2)})^\star)$ using the train set, where $(\bxi^{(2)})^\star$ for each $\lambda$ is found using the validation set by the strategy described in Appendix \ref{sec-app-choosing-xi-ill2}.
    \end{itemize}
    Let $\slambda_0$ denote the optimal penalty for ridge that minimizes the MSE on the validation set. Similarly, let $(\slambda_1, \bxi^\star)$ and $\bigl(\slambda_2, (\bxi_1^\star, \bxi_2^\star)\bigr)$ denote the optimal parameters for $1$-step, $2$-step SD, again chosen via the validation set.
    \item Evaluate the MSE on the test set (unseen as yet) for all three computed estimators: $\htheta (\slambda_0)$, $\htheta (\slambda_1, \xi^\star)$, $\htheta \bigl(\slambda_2, (\bxi_1^\star, \bxi_2^\star)\bigr)$.
\end{enumerate}

\subsection{Description of the datasets used} \label{sec-app-datasets}
{\bf UCI Air Quality}.
The UCI Air Quality dataset \cite{misc_air_quality_360} contains records of hourly averaged readings from 5 metal oxide chemical sensors (which serve as covariates), along with ground-truth hourly averaged concentrations of the pollutants from a co-located reference certified analyzer. The recordings are from an Italian city over a one year period in 2004-2005.
After removing records with missing entries, this dataset has $n=6941$ rows and we consider $d=8$ relevant covariates (5 values of the metal oxide chemical sensor readings + 3 values of Temperature, Relative Humidity, and Absolute Humidity).
We explicitly mention the field names (all real-numbered non-negative).
\begin{itemize}
    \item $\x$ covariates' names in the dataset: \texttt{[PT08.S1(CO), PT08.S2(NMHC), PT08.S3(NOx), PT08.S4(NO2), PT08.S5(O3), T, RH, AH]}
    \item $\y$ target's name in the dataset: \texttt{NO2(GT)}
\end{itemize}


{\bf UCI Airfoil Self-Noise}.
The UCI Airfoil Self-Noise dataset \cite{misc_airfoil_self-noise_291} contains data obtained from NASA's aerodynamic and acoustic tests of airfoil blade sections.
It contains $5$ real-valued covariates, which are relevant physical quantities.
The target is also real-valued, which is the sound pressure created (in dB).
There are no missing entries. This dataset has $n=1503$ rows and $d=5$ covariates.
We explicitly mention the field names:
\begin{itemize}
    \item $\x$ covariates' names in the dataset: \texttt{[frequency, attack-angle, chord-length, free-stream-velocity, suction-side-displacement-thickness]}
    \item $\y$ target's name in the dataset: \texttt{scaled-sound-pressure}
\end{itemize}

{\bf UCI Appliances Energy Prediction}.
The UCI AEP dataset \cite{misc_appliances_energy_prediction_374} contains energy appliances' data collected at $10$ min frequency for about $4.5$ months. 
This dataset has no missing entries, and a total of $19735$ instances with $28$ covariates.
We downsample the dataset to hourly frequency, giving $n=3290$ rows with $d=24$ relevant covariates (removing degenerate covariates: `date', `lights', `windspeed', `visiblity').
The full list of covariates is $24$ long, and so we do not list it out here (we already mentioned the $4$ we removed from the total of $28$).

\subsection{Observed quadratic nature of MSE w.r.t. $\xi$ parameters} \label{sec-app-observed-quadratic}

Theorem~\ref{thm-quadraticRisk} showed that the excess risk is quadratic w.r.t. $\xik$ variables (the reparameterized version). We used this theoretical insight to devise a hyperparameter tuning strategy in section~\ref{sec-expts-hparams}.
The below curves ratify this empirically.
Inspecting the minima of each of the below curves, we observe that the values of $\xi^\star$ for $1$-step SD in Table~\ref{table} (found through the strategy described in section~\ref{sec-expts-hparams}) indeed coincide with the minima producing values in the below curves.
\begin{figure}[h]
\begin{subfigure}{.33\textwidth}
  \includegraphics[width=\textwidth]{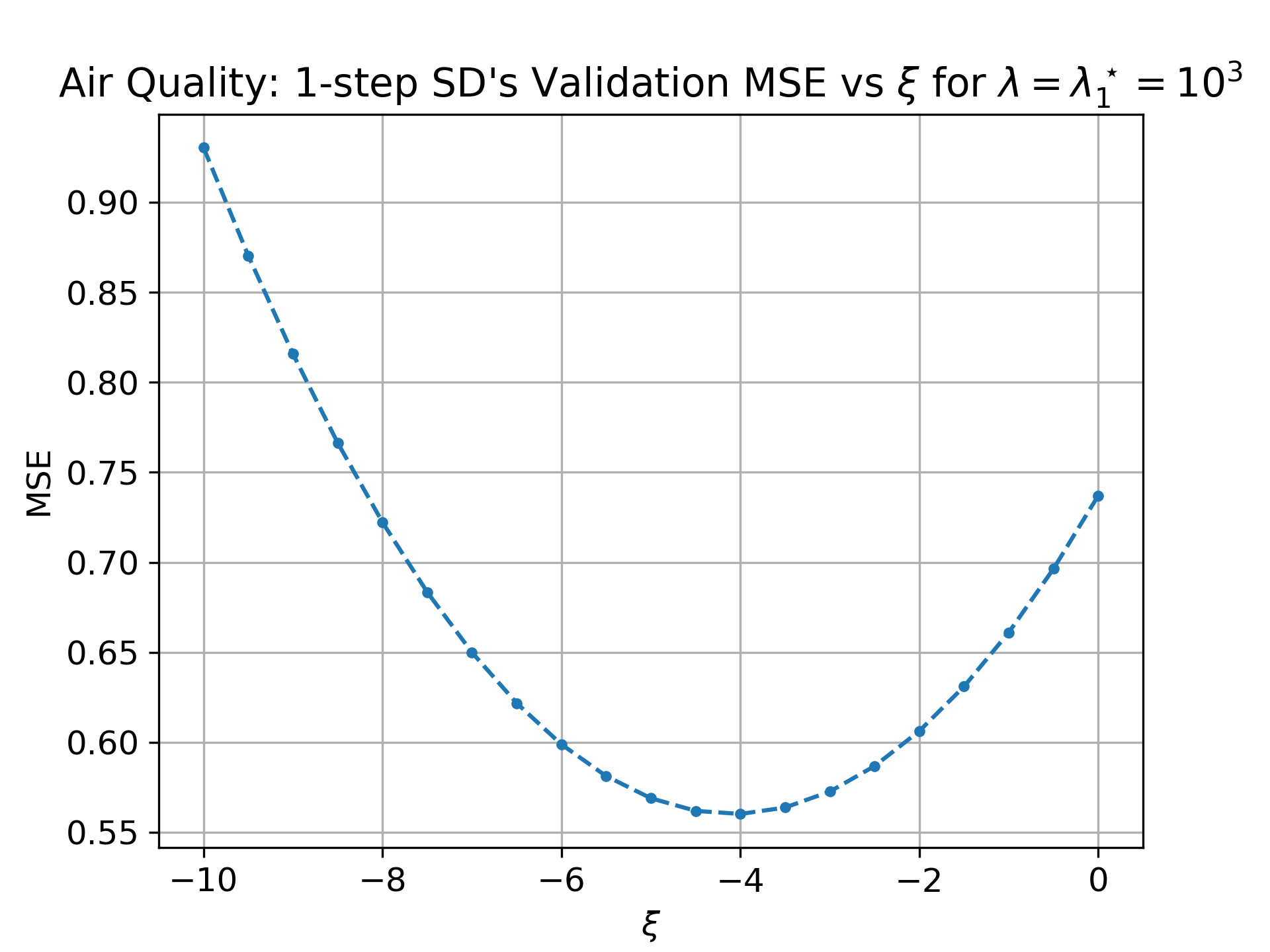}
  \caption{Air Quality dataset}
  \label{fig-regr-quad-AQ}
\end{subfigure}
\begin{subfigure}{.33\textwidth}
  \includegraphics[width=\linewidth]{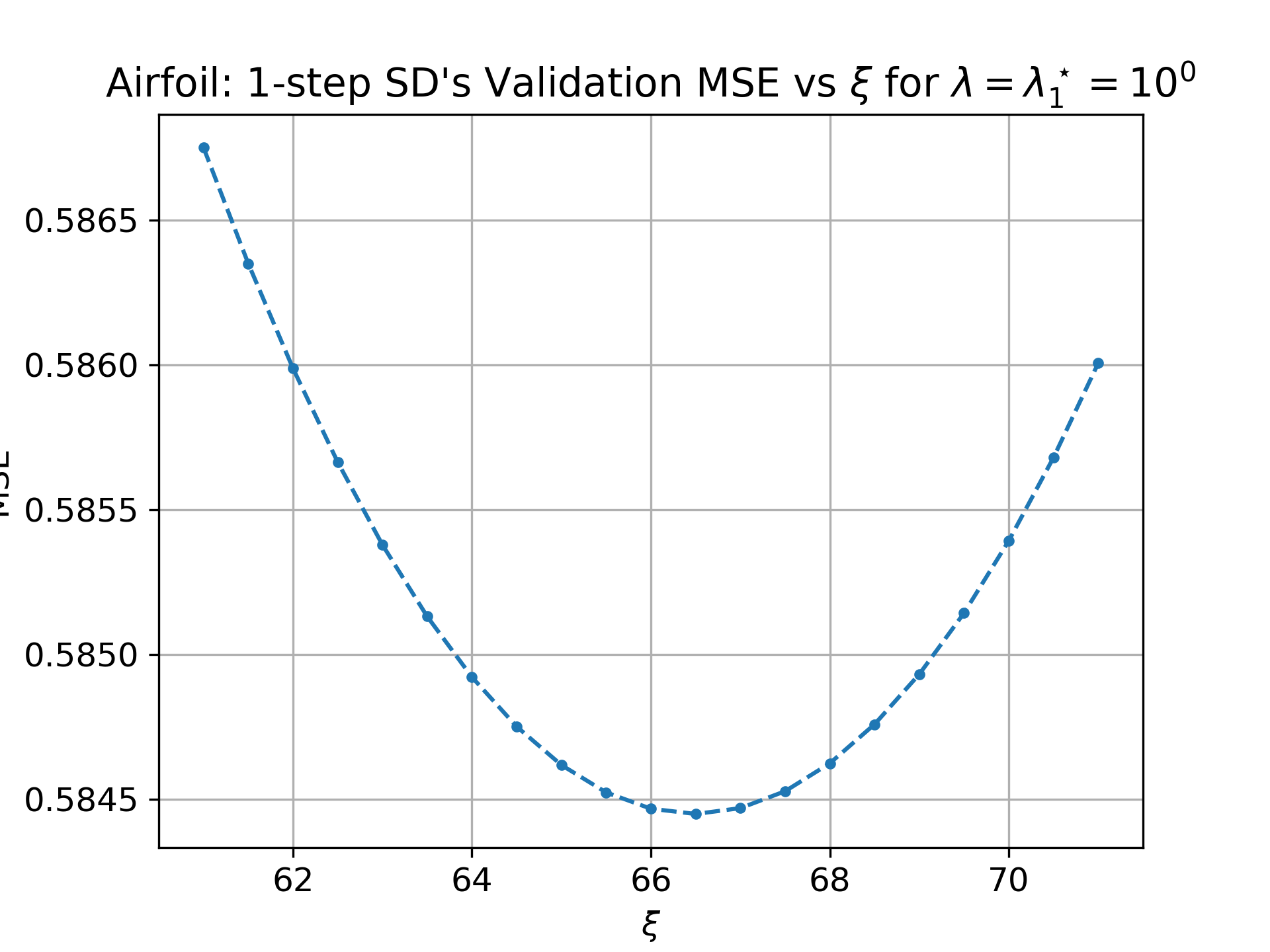}
  \caption{Airfoil dataset}
  \label{fig-regr-quad-AF}
\end{subfigure}
\begin{subfigure}{.33\textwidth}
  \includegraphics[width=\linewidth]{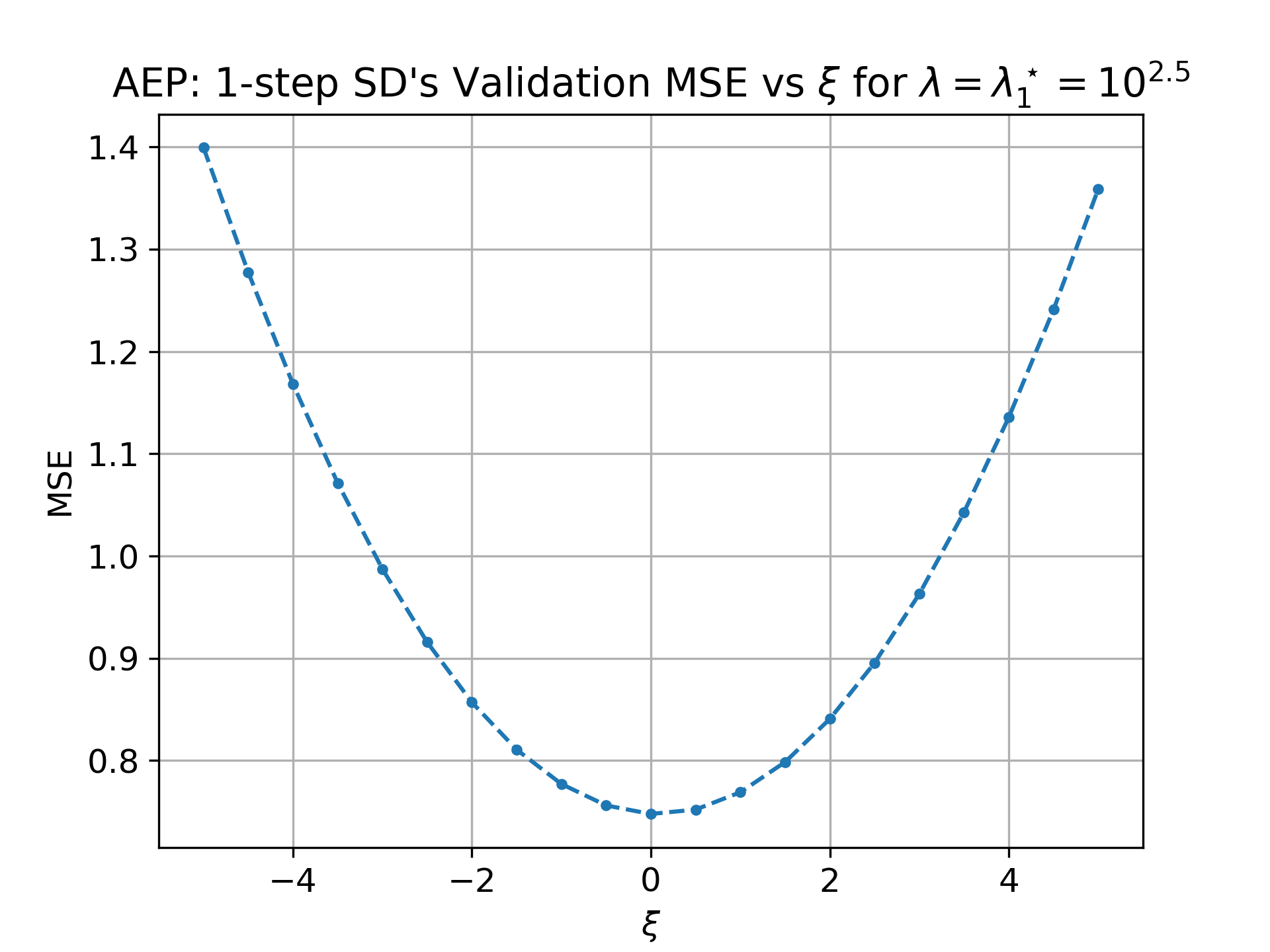}
  \caption{AEP dataset}
  \label{fig-regr-quad-AEP}
\end{subfigure}
\caption{Observed quadratic nature of MSE (on validation set) of $1$-step SD vs $\xi$ for $\lambda = \slambda_1$. This agrees with Theorem~\ref{thm-quadraticRisk} and validates the hyperparameter tuning strategy outlined in section~\ref{sec-expts-hparams}.
}
\label{fig-regr-appendix}
\end{figure}

\end{document}